\documentclass[twoside]{article}

\ifdefined\pdfsuppressptexinfo\pdfsuppressptexinfo=-1\fi
\ifdefined\pdfinfoomitdate\pdfinfoomitdate=1\fi

\usepackage[preprint]{aistats2027}

\usepackage{amssymb}
\usepackage{amsthm}
\usepackage{mathtools}
\usepackage{algorithm}
\usepackage{algpseudocode}
\usepackage[round]{natbib}
\usepackage{xcolor}
\usepackage[hidelinks]{hyperref}

\makeatletter
\providecommand{\theHALG@line}{}
\renewcommand{\theHALG@line}{\thealgorithm.\arabic{ALG@line}}
\makeatother

\usepackage{cleveref}
\usepackage{enumitem}
\usepackage{booktabs}
\usepackage{array}
\usepackage{graphicx}

\newcommand{\R}{\mathbb{R}}
\newcommand{\E}{\mathbb{E}}
\newcommand{\norm}[1]{\left\lVert #1 \right\rVert}
\newcommand{\ip}[2]{\left\langle #1,#2 \right\rangle}
\newcommand{\grad}{\nabla}

\DeclareMathOperator*{\argmin}{arg\,min}

\newtheorem{assumption}{Assumption}

\newtheorem{lemma}{Lemma}
\newtheorem{proposition}{Proposition}

\newtheorem{theorem}{Theorem}
\newtheorem{corollary}{Corollary}
\begin{document}
\runningauthor{Tovmasyan, Hanzely, Shulgin, Richt\'arik, Stich}



\twocolumn[

\aistatstitle{SCAFFOLD under Hessian Similarity: Sharp Convergence Rates}

\aistatsauthor{ Zhirayr Tovmasyan\textsuperscript{1} \And Slavom\'ir Hanzely\textsuperscript{3} \And Egor Shulgin\textsuperscript{1,2} \AND Peter Richt\'arik\textsuperscript{1} \And Sebastian U. Stich\textsuperscript{3} }

\aistatsaddress{ \rule{0pt}{2em}\textsuperscript{1}King Abdullah University of Science and Technology (KAUST) \\  \textsuperscript{2}Mohamed bin Zayed University of Artificial Intelligence (MBZUAI) \\
\textsuperscript{3}CISPA Helmholtz Center for Information Security } ]

\begin{abstract}
One of the main challenges of federated learning is heterogeneous data split among the clients, leading to biased local updates called client drift.
The SCAFFOLD algorithm \citep{karimireddy2020scaffold} corrects client drift with control
variates, but its relation to the similarity of client data is not yet fully understood:
the original similarity-based guarantees cover only quadratics with a stronger
second-order similarity assumption, and a follow-up weakly convex analysis
\citep{luo2025revisiting} obtains a \(\sqrt{L\delta}\) dependence for \(L\)-smooth losses with averaged Hessian similarity \(\delta\). We analyze
SCAFFOLD with minibatch-based control variates and establish linear dependence on
Hessian similarity beyond quadratics, for both convex and nonconvex objectives.
In particular, for convex objectives with exact gradients, \(R\) SCAFFOLD rounds with \(K\) local
steps lead to an objective gap of order \((L/K+\delta)D_0^2/R\), where \(D_0\)
is the initial distance to a minimizer.
When \(K=\Theta(L/\delta)\),
it improves the minibatch SGD rate by a factor \(\Theta(L/\delta)\).
For nonconvex objectives, we prove an analogous stationarity guarantee under uniform Hessian similarity for a client-selection variant.
In both settings, we extend these guarantees to stochastic
gradients and prove matching lower bounds for the analyzed methods with constant local stepsizes. We also
show that, under partial participation, even identical clients need not obtain
the full communication benefit of additional local steps.
\end{abstract}

\section{Introduction}

Federated learning enables multiple clients to collaboratively train a shared model while keeping their data local. However, communication between clients and the central server is often a major bottleneck. Popular optimization algorithms, such as FedAvg \citep{mcmahan2017fedavg}, mitigate this cost by performing multiple local updates between communication rounds, trading additional local computation for reduced communication overhead. In the presence of heterogeneous client data, however, increasing the number of local updates can exacerbate the tendency of local iterates to deviate from the global optimization trajectory, a phenomenon known as \emph{client drift}.

The federated optimization algorithm SCAFFOLD \citep{karimireddy2020scaffold} addresses client drift using control variates that correct local update directions. The effectiveness of this correction depends, in part, on the degree of heterogeneity across clients. Quantifying this heterogeneity is therefore important for understanding the convergence behavior of federated optimization algorithms. One way to characterize client similarity is through Hessian similarity, which measures differences in curvature between the global and client-specific loss functions.

Despite this connection, the theoretical benefits of SCAFFOLD under Hessian similarity remain incompletely understood. The original SCAFFOLD paper \citep{karimireddy2020scaffold} establishes the benefit of control variates for quadratic objectives, but its analysis considers a simplified variant that uses exact control variates at the current server iterate. Subsequently, \citet{luo2025revisiting} extend the analysis beyond quadratic objectives, but obtain guarantees with weaker dependence on client similarity or under additional structural assumptions. For convex objectives, linear dependence of the communication complexity on the similarity measure has been established for methods incorporating regularization, including DANE+, FedRed, and S-DANE \citep{jiang2024fedred,jiang2024sdane}. However, whether SCAFFOLD's corrected local steps achieve the same dependence under the averaged Hessian similarity assumption (\Cref{ass:similarity}) remains an open question.

To address this question, we analyze SCAFFOLD with minibatch Option-I controls (Algorithm~\ref{alg:scaffold-option-i}) under the \(\delta\)-averaged Hessian similarity assumption (\Cref{ass:similarity})\footnote{\citet{karimireddy2020scaffold} propose two options for updating the control variates. Option-I computes the full local gradient at the server model at the beginning of each round, whereas Option-II reuses gradients computed during local updates. We replace the full gradient in Option-I by a stochastic minibatch estimate.}. For convex, \(L_{\mathrm{loc}}\)-smooth objectives (\Cref{ass:convex-smooth}) and deterministic local gradients (\(\sigma=0\) in \Cref{ass:stochastic-oracle}), our convergence bound \eqref{eq:deterministic-rate} depends on the problem geometry through the factor \(L_{\mathrm{loc}}/K+\delta\), where \(K\) denotes the number of local steps per communication round. Thus, increasing the number of local steps reduces the \(L_{\mathrm{loc}}/K\) term, but the resulting improvement is limited by the parameter \(\delta\). More precisely, when \(0<\delta\leq L_{\mathrm{loc}}\), choosing \(K=\Theta(L_{\mathrm{loc}}/\delta)\) gives an improvement over minibatch SGD by a factor of \(\Theta(L_{\mathrm{loc}}/\delta)\) (\Cref{tab:convex-stochastic-comparison}).

Beyond the deterministic setting, our analysis accommodates stochastic gradients for both local model updates and control variate estimation, while preserving linear dependence on the Hessian similarity parameter. With an appropriate choice of the number of local steps, our method achieves a communication-round complexity matching that of FedRed-SGD \citep{jiang2024fedred}, without requiring exact gradients or regularization. Its stochastic gradient complexity also matches that of FedRed-SGD up to a logarithmic factor in the noise-dependent term.  Detailed comparisons with existing methods are provided in Section~\ref{sec:related-work}.
% TODO reference

We complement our convergence guarantees for both convex and nonconvex objectives with matching lower bounds up to a constant factor, establishing the tightness of our rates and characterizing the communication and local computation complexities of constant-stepsize SCAFFOLD variants.

\subsection{Summary of contributions}

Our main contributions can be summarized as follows:

\begin{itemize}
    \item \textbf{Convex convergence and optimality.}
    Under the averaged Hessian similarity assumption (\Cref{ass:similarity}), we establish objective-gap guarantees for the averaged server iterate of SCAFFOLD Option-I (\Cref{alg:scaffold-option-i}) with constant local stepsizes and control minibatch size equal to the number of local steps per communication round, \(B=K\) (\Cref{thm:scaffold-convex}). We establish matching lower bounds, showing that these rates are optimal for the considered constant-stepsize SCAFFOLD variants (\Cref{thm:constant-step-optimality}).

    \item \textbf{Nonconvex convergence and optimality.}
    Under uniform Hessian similarity (\Cref{ass:uniform-similarity}), we establish a stationarity guarantee for the SCAFFOLD client-selection variant with randomly selected local output (\Cref{alg:scaffold-nonconvex}, \Cref{thm:scaffold-nonconvex}) with constant local stepsizes. Under local smoothness, we establish a matching lower bound within the parameter regime specified in Theorem~\ref{thm:nc-lower}.

    \item \textbf{Partial participation.}
    For the scenario in which not all clients participate in each round but exact gradients are available, we establish convergence guarantees under uniform client sampling (\Cref{thm:pp-main}) and show that refreshing control variates only for participating clients can prevent the full \(1/K\) improvement, even when all clients have identical convex objectives. The lower bound holds for every constant local stepsize \(0<\eta_\ell\leq 1/L_{\mathrm{loc}}\) (\Cref{thm:pp-lower}).
    
\end{itemize}

\subsection{Related work}
\label{sec:related-work}

\paragraph{SCAFFOLD and client drift.}
Beyond the quadratic analysis of \citet{karimireddy2020scaffold},
\citet{gorbunov2021localsgd} study shifted gradient estimators and \citet{mitra2021fedlin} investigate corrected updates with unequal local
workloads.
Under \(\delta\)-averaged Hessian similarity and $\rho$-weak convexity,
\citet{luo2025revisiting} obtain the stationarity coefficient
$L_{\mathrm{loc}}/K+\sqrt{L_{\mathrm{loc}}\delta}+\rho$, and 
under additional uniform similarity, $\bar\delta\ge\delta$, and a quadratic convex combination of
local objectives they were able to replace $\sqrt{L_{\mathrm{loc}}\delta}$ by
$\sqrt{\bar\delta\delta}$. Nevertheless, neither result shows linear dependence on averaged
similarity $\delta$ in the convex case. In the nonconvex case, the
constant $\rho$ may exceed uniform similarity $\bar\delta$.

\paragraph{Communication under similarity.}
\citet{shamir2014dane} propose DANE, which exploits similarity through corrected, regularized local
subproblems. For general convex objectives
under averaged similarity, \citet{jiang2024fedred} present DANE+, which achieves
$\mathcal O(\delta D_0^2/\varepsilon)$ communication rounds, and
FedRed, which reaches the same order in expectation through randomized
synchronization.
\citet{jiang2024sdane} propose S-DANE, which supports accelerated local solvers, and Acc-S-DANE,
which improves the communication order to
$\mathcal O(\sqrt{\delta D_0^2/\varepsilon})$.
DANE+, FedRed, and S-DANE use exact gradients for their
corrections, including in the analyzed implementations with
stochastic local computation.
Under strong convexity, SVRP \citep{khaled2023svrp} combines proximal updates with
client sampling and variance reduction,
while ProxSkip \citep{mishchenko2022proxskip} accelerates communication through randomized
synchronization without a similarity assumption.
We summarize comparison of communication
and computation costs in Table~\ref{tab:convex-stochastic-comparison}.

\paragraph{Nonconvex methods.}
BVR-L-SGD \citep{murata2021bvr} and CE-LSGD
\citep{patel2022communication} exploit second-order similarity
using variance-reduced estimators that evaluate the same sample
at two points and require additional smoothness of stochastic gradients. Both permit one selected client to perform local
steps and return a sampled local iterate. The deterministic variant CE-LGD
achieves the optimization coefficient $L_{\mathrm{loc}}/K+\bar\delta$.
DANE+ \citep{jiang2024fedred} permits a predetermined client to solve the local
subproblem. MimeMVR admits nonconvex convergence
guarantees with a single client performing local updates, under
sample-wise Hessian similarity \citep{karimireddy2021mime}.

AuxMOM \citep{chayti2024auxiliary} separates global smoothness from similarity, obtaining
the optimization term $(L_g/K+\bar\delta)F_0/R$ without sample-wise
smoothness. Its bound contains a local
gradient-noise term proportional to $1/K$, which does not decrease
with more communication rounds.
With full local gradients, SPAM \citep{karagulyan2024spam} proves proximal convergence guarantees without global
smoothness and
\citet{jiang2025costaware} analyze client selection with explicit
communication and computation costs.
We summarize rates and oracle requirements in Table~\ref{tab:nonconvex-comparison}.

\paragraph{Lower bounds and stochastic iterate bias.}
Prior work establishes minimax lower bounds for classes of distributed
optimization algorithms
\citep{woodworth2020minibatch,patel2022communication}.
Our bounds concern the specified constant-stepsize methods and
outputs. For Local SGD, \citet{woodworth2020localsgd,glasgow2022}
establish lower bounds through stochastic iterate bias, even
for identical clients. \citet{mangold2025scaffold} characterize
the bias and variance of stochastic SCAFFOLD under strong
convexity. Our convex constructions build on the Local SGD
examples. Our nonconvex constructions adapt quadratic SGD
examples \citep{drori2020} and distributed zero chains
\citep{patel2022communication,carmon2020}.
Sections~\ref{sec:constant-step-optimality}
and~\ref{sec:nc-optimality} explain the adaptations.

\paragraph{Partial participation.}
The original SCAFFOLD analysis includes client sampling
\citep{karimireddy2020scaffold}. SCAFFOLD-M \citep{cheng2024momentum} improves
partial-participation guarantees through momentum, and Amplified SCAFFOLD \citep{crawshaw2024amplified} addresses
periodic participation.
Our partial-participation analysis in
Appendix~\ref{app:partial-participation} concerns the rule in which
only participating clients refresh their controls.

\section{Preliminaries}
We consider a federated optimization problem with \(N\) clients having
differentiable local objectives \(f_1,\ldots,f_N\). The global objective is their average,
\begin{equation}
    \min_{x\in\R^d} f(x),
    \qquad
    f(x) \coloneqq \frac1N\sum_{i=1}^N f_i(x).
    \label{eq:objective}
\end{equation}
Each client \(i\) has a local data
distribution \(\mathcal D_i\), and its objective is defined as the expected loss
\begin{equation}
    f_i(x)
    \coloneqq
    \E_{\xi\sim\mathcal D_i}\!\left[f_i(x;\xi)\right],
    \label{eq:local-objective}
\end{equation}
where \(f_i(x;\xi)\) denotes the loss evaluated at model \(x\) on sample \(\xi\).

Our main results consider full participation, meaning that all \(N\) clients estimate their control variates in every communication round. We denote \([N]\coloneqq \{1, \dots, N\}\). In Algorithm~\ref{alg:scaffold-option-i}, all clients also perform local optimization, whereas in Algorithm~\ref{alg:scaffold-nonconvex}, only one selected client performs local optimization per round. We defer the analysis of a partial-participation variant to Appendix~\ref{app:partial-participation}.

In the convex setting, we assume that \(f\) attains its minimum at \(x^\star\in\argmin_{x\in\mathbb{R}^d}f(x)\), and write \(f^\star\coloneqq f(x^\star)\). In the nonconvex setting, we require only that \(f\) be bounded below. We invoke the assumptions for these two settings separately: Section~\ref{sec:convex-results} considers locally convex and locally smooth objectives under averaged Hessian similarity, whereas Section~\ref{sec:nonconvex} considers globally smooth objectives under uniform Hessian similarity. We formalize and discuss them below.

\begin{assumption}[Local convexity and smoothness]
\label{ass:convex-smooth}
Every local objective \(f_i\) is differentiable and convex: for every
\(i\in[N]\) and \(x,y\in\R^d\),
\begin{equation}
    f_i(y)
    \ge f_i(x)+\ip{\grad f_i(x)}{y-x}.
    \label{eq:local-convexity}
\end{equation}
It is also \(L_{\mathrm{loc}}\)-smooth for some \(L_{\mathrm{loc}}>0\):
for every \(i\in[N]\) and \(x,y\in\R^d\),
\begin{equation}
    \norm{\grad f_i(x)-\grad f_i(y)}
    \le L_{\mathrm{loc}}\norm{x-y}.
    \label{eq:local-smoothness}
\end{equation}
\end{assumption}

These conditions are standard in SCAFFOLD analyses
\citep{karimireddy2020scaffold,luo2025revisiting,takezawa2025fedmuon} and in the analysis of
gradient-based local solvers for DANE+ and FedRed \citep{jiang2024fedred}.

\begin{assumption}[Global smoothness]
\label{ass:global-smoothness}
The global objective \(f\) is differentiable and \(L_g\)-smooth for some
\(L_g>0\): for all \(x,y\in\R^d\) holds
\begin{equation}
    \norm{\grad f(x)-\grad f(y)}\le L_g\norm{x-y}.
    \label{eq:global-smoothness}
\end{equation}
\end{assumption}

Local smoothness implies global smoothness with \(L_g\le L_{\mathrm{loc}}\).

For each client \(i\), define the difference between its local objective and the global objective as
\(h_i\coloneqq f_i-f\).

\begin{assumption}[Averaged Hessian similarity]
\label{ass:similarity}
For some \(\delta\ge0\) and all \(x,y\in\R^d\) holds
\begin{equation}
    \frac1N\sum_{i=1}^N
    \norm{\grad h_i(x)-\grad h_i(y)}^2
    \le
    \delta^2\norm{x-y}^2.
    \label{eq:similarity}
\end{equation}
\end{assumption}

This condition is used by \citet{luo2025revisiting} and is referred
to as averaged Hessian dissimilarity in the DANE+/FedRed analysis
\citep{jiang2024fedred}. Under full participation, it coincides with the
second-order dissimilarity condition used in S-DANE
\citep{jiang2024sdane}.

\begin{assumption}[Uniform Hessian similarity]
\label{ass:uniform-similarity}
For every client \(i\in[N]\), there is a constant \(\delta_i\ge0\) such
that, for all \(x,y\in\R^d\) holds
\begin{equation}
    \norm{\grad h_i(x)-\grad h_i(y)}
    \le\delta_i\norm{x-y}.
    \label{eq:uniform-similarity}
\end{equation}
We write \(\bar\delta=\max_i\delta_i\) for the uniform bound. 
\end{assumption}

Uniform Hessian similarity implies
\eqref{eq:similarity} with
$\delta^2=N^{-1}\sum_{i=1}^N\delta_i^2$, and together with global
smoothness implies that each $f_i$ is $(L_g+\delta_i)$-smooth.
This assumption is also used in DANE+ \citep{jiang2024fedred},
SPAM \citep{karagulyan2024spam}, and the SCAFFOLD analysis of \citet{luo2025revisiting}. It is implied by the uniform Hessian-difference bound used in the original
SCAFFOLD analysis \citep{karimireddy2020scaffold} and CE-LGD \citep{patel2022communication}, but itself
does not require second derivatives.

At local step $k$ of round $r$, client $i$ draws a sample
$\xi_{i,k}^r\sim\mathcal D_i$ and evaluates the stochastic gradient
$\grad f_i(y_{i,k}^r;\xi_{i,k}^r)$. Let $\mathcal F_{r,k}$ denote all iterates and randomness observed up to the
beginning of this local step (i.e., $y_{i,k}^r$ is
$\mathcal F_{r,k}$-measurable). We assume that the local stochastic gradients are unbiased and have
bounded variance $\sigma^2$ -- value
$\sigma=0$ corresponds to the deterministic setting of exact gradients.

\begin{assumption}[Local stochastic gradients]
\label{ass:stochastic-oracle}
For every $r\ge1$, \(i\in[N]\), $k\in\{0,\ldots,K-1\}$,
the stochastic gradient is conditionally unbiased almost surely,
\begin{equation}
    \E\!\left[
        \grad f_i(y_{i,k}^r;\xi_{i,k}^r)
        \mid\mathcal F_{r,k}
    \right]
    =\grad f_i(y_{i,k}^r), \label{eq:stochastic-unbiasedness}
\end{equation}
with its conditional variance almost surely bounded by
\begin{equation}
    \E\!\left[
        \norm{
            \grad f_i(y_{i,k}^r;\xi_{i,k}^r)
            -\grad f_i(y_{i,k}^r)
        }^2
        \mid\mathcal F_{r,k}
    \right]
    \le \sigma^2. \label{eq:stochastic-oracle}
\end{equation}
The same conditional moment bounds apply to gradient evaluations used
to estimate control variates, conditioned on the information available
before each sample is drawn. All oracle samples, including those used
for control-variate estimates, are drawn independently across clients
and oracle calls.
\end{assumption}
% TODO

This assumption is also standard \citep{karimireddy2020scaffold,luo2025revisiting}.

With the assumptions in place, we now present the studied algorithm and establish its convergence rates.

\subsection{SCAFFOLD with minibatch Option-I controls}

Algorithm~\ref{alg:scaffold-option-i} formalizes the full-participation method analyzed in our convex setting.  Starting from an initial model \(x^0\), each client performs \(K\ge1\)
local steps with stepsize \(\eta_\ell>0\), after which the server averages
the client endpoints at the server (server stepsize is \(\eta_g=1\)).
The control variates are estimated using minibatches of size \(B\ge1\),
including at initialization.

\begin{algorithm}[t]
\caption{SCAFFOLD \citep{karimireddy2020scaffold}, full participation with minibatch Option-I controls}
\label{alg:scaffold-option-i}
\begin{algorithmic}[1]
\Require Model \(x^0\), rounds \(R\), and local steps \(K\)
\Require Stepsize \(\eta_\ell\) and control batch size \(B\)
\State Each client draws \(B\) samples \(\xi_{i,b}^{0,c}\sim\mathcal D_i\)
\State Each client sets \(c_i^0\gets B^{-1}\sum_{b=1}^B\grad f_i(x^0;\xi_{i,b}^{0,c})\)
\State Server sets \(c^0\gets N^{-1}\sum_{i=1}^N c_i^0\)
\For{rounds \(r=1,\ldots,R\)}
    \State Server broadcasts \((x^{r-1},c^{r-1})\)
    \For{each client \(i\in[N]\) in parallel}
        \State \(y_{i,0}^r\gets x^{r-1}\)
        \For{\(k=0,\ldots,K-1\)}
            \State Draw \(\xi_{i,k}^r\sim\mathcal D_i\)
            \State \(d_{i,k}^r\gets
            \grad f_i(y_{i,k}^r;\xi_{i,k}^r)-c_i^{r-1}+c^{r-1}\)
            \State \(y_{i,k+1}^r\gets y_{i,k}^r-\eta_\ell d_{i,k}^r\)
        \EndFor
        \State Draw \(B\) new samples \(\xi_{i,b}^{r,c}\sim\mathcal D_i\)
        \State Set \(c_i^r\gets B^{-1}\sum_{b=1}^B\grad f_i(x^{r-1};\xi_{i,b}^{r,c})\)
        \State Send \(\bigl(y_{i,K}^r,c_i^r\bigr)\) to the server
    \EndFor
    \State \(x^r\gets N^{-1}\sum_{i=1}^N y_{i,K}^r\),\quad
    \(c^r\gets N^{-1}\sum_{i=1}^N c_i^r\)
\EndFor
\State \Return \(\bar x^R\gets R^{-1}\sum_{r=1}^R x^r\)
\end{algorithmic}
\end{algorithm}

Throughout our main results, we set $B=K$. Then initialization requires $K$ stochastic gradient evaluations
per client and one control aggregation. In subsequent rounds, every
client performs $2K$ stochastic gradient evaluations: $K$ for the local
updates and $K$ to estimate its next control variate. The client sends
its endpoint and the next control variate to the server in a single message.
The control update after the last round may be omitted.

Adopting the convention $x^{-1}=x^0$, the control variates used in
round $r$ are estimated at $x^{r-2}$, one server iterate before the
round's starting model $x^{r-1}$,
\begin{equation}
\begin{aligned}
    c_i^{r-1}&=\frac1B\sum_{b=1}^B
        \grad f_i(x^{r-2};\xi_{i,b}^{r-1,c}), \\
    c^{r-1}&=\frac1N\sum_{i=1}^N c_i^{r-1}.
\end{aligned}
    \label{eq:stale-control-identity}
\end{equation}
Thus the first round uses controls at \(x^0\), and every subsequent round
uses controls that are one server iterate stale. This staleness is what
allows the endpoint and the next control to travel in the same message.

We are now ready to present the convergence rates.

\section{Theoretical results}
\subsection{Convex objectives}
\label{sec:convex-results}

We measure suboptimality using the expected objective gap of the averaged
server iterate,
\begin{equation}
    \bar x^R\coloneqq\frac1R\sum_{r=1}^R x^r,
    \qquad
    \Delta_R\coloneqq\E[f(\bar x^R)]-f^\star.
    \label{eq:averaged-output-gap}
\end{equation}

For convenience, define the effective stepsize of a round $\alpha\coloneqq K\eta_\ell$ and initial distance
$D_0\coloneqq\norm{x^0-x^\star}$. When $D_0,\sigma>0$, we also define a logarithmic term that will appear with the local noise,
\begin{equation}
    \kappa\coloneqq
    \log\left(e+\frac{L_{\mathrm{loc}}D_0\sqrt K}
                       {\sigma\sqrt R}\right).
    \label{eq:effective-log}
\end{equation}

Our first result bounds the expected objective gap for any admissible
constant stepsize.

\begin{proposition}[Fixed stepsize]
\label{prop:fixed-stepsize}
Suppose Assumptions~\ref{ass:convex-smooth}, \ref{ass:similarity},
and~\ref{ass:stochastic-oracle} hold. Run Algorithm~\ref{alg:scaffold-option-i}
with \(B\ge1\) and a constant stepsize satisfying
\(\eta_\ell L_{\mathrm{loc}}\le1\) and \(\alpha\delta\le1/(3\sqrt2)\). Then
\begin{equation}
\begin{aligned}
    \Delta_R\le{}&40\biggl[\frac{D_0^2}{\alpha R}
    +\frac{\alpha\sigma^2}{N}\Bigl(\frac1K+\frac1B\Bigr)\\
    &\quad+\alpha\sigma^2\Bigl(\frac1K+\frac1B\Bigr)\log(1+\alpha L_{\mathrm{loc}})\biggr].
\end{aligned}
    \label{eq:fixed-stepsize-rate}
\end{equation}
\end{proposition}

When \(\sigma=0\), only the first term remains. It decreases as
\(\alpha\) grows, and the stepsize conditions allow \(\alpha\) up to order
\(\min\{K/L_{\mathrm{loc}},1/\delta\}\), so local steps reduce the
deterministic term until similarity becomes the limit. Setting \(B=K\)
and tuning \(\alpha\) gives the following theorem. The analysis of other
control minibatch sizes is in Appendix~\ref{app:convex-batch-size}.

\begin{theorem}
\label{thm:scaffold-convex}
Suppose Assumptions~\ref{ass:convex-smooth}, \ref{ass:similarity},
and~\ref{ass:stochastic-oracle} hold, and \(D_0,\sigma>0\).
The output of Algorithm~\ref{alg:scaffold-option-i} with \(B=K\) for \(R\ge1\)
rounds and the constant stepsize specified in
\eqref{eq:app-complete-tuning} of Appendix~\ref{app:convex-tuning} satisfies
\[
    \Delta_R\le200\min\{L_{\mathrm{loc}}D_0^2,\mathcal T\},
\]
where
\begin{equation}
\begin{aligned}
    \mathcal T\coloneqq{}&
    \max\left\{\frac{L_{\mathrm{loc}}}{K},\delta\right\}
        \frac{D_0^2}{R}
    +\frac{\sigma D_0}{\sqrt{NKR}}\\
    &+\min\left\{
        \left(
            \frac{L_{\mathrm{loc}}\sigma^2D_0^4}{KR^2}
        \right)^{1/3},
        \frac{\sigma D_0\sqrt\kappa}{\sqrt{KR}}
    \right\}.
\end{aligned}
\label{eq:tuned-rate}
\end{equation}
\end{theorem}

The first term in the minimum caps the guarantee at the initial scale
\(200L_{\mathrm{loc}}D_0^2\), and \(\mathcal T\) decreases with more
computation. Within \(\mathcal T\), the first part is the only term that
remains with exact gradients (\(\sigma=0\)), and there local steps reduce
the smoothness part until \(\delta\) becomes the limiting factor. The
second part is the statistical error, reduced by averaging over the \(N\)
clients. The third part is the error from the local stochastic
trajectories, which client averaging does not reduce. Its two branches
come from two complementary bounds on this error, and the smaller one
applies.

\paragraph{Deterministic case.}
For \(\sigma=0\), choose the largest \(\alpha\) admitted by
Proposition~\ref{prop:fixed-stepsize}, namely
\(\alpha^{-1}=\max\{L_{\mathrm{loc}}/K,3\sqrt2\delta\}\).
Proposition~\ref{prop:fixed-stepsize} then gives
\begin{equation}
    \Delta_R
    =\mathcal{O}\!\left(
        \left(\frac{L_{\mathrm{loc}}}{K}+\delta\right)
        \frac{D_0^2}{R}
    \right).
    \label{eq:deterministic-rate}
\end{equation}
For \(0<\delta\le L_{\mathrm{loc}}\) and
\(0<\varepsilon\le\delta D_0^2\), choosing
\(K=\Theta(L_{\mathrm{loc}}/\delta)\) gives
\(\mathcal O(\delta D_0^2/\varepsilon)\) communication rounds and
\(\mathcal O(L_{\mathrm{loc}}D_0^2/\varepsilon)\) exact-gradient
evaluations per client, including those used to compute control variates. These rates match
FedRed-GD's expected communication and per-client gradient complexities
\citep{jiang2024fedred}.

\begin{table*}[t]
\centering
\caption{
    Comparison of communication and computation complexities of stochastic federated learning methods
    under full participation for achieving
    \(\E[f(\widehat x)]-f^\star\le\varepsilon\), where \(0<\varepsilon\le\delta D_0^2\).
    We assume that the local objectives are convex
    and \(L\)-Lipschitz smooth (\Cref{ass:convex-smooth}), and satisfy \(\delta\)-averaged Hessian similarity (\Cref{ass:similarity}), \(0<\delta \leq L\).
    We report upper bounds for the non-accelerated
    implementations with appropriately chosen number of local steps and minibatch sizes (see Appendix~\ref{app:convex-comparison} for details).
    Computation is measured by the number of stochastic gradient
    evaluations per client across all communication rounds, including
    initialization. Constants and logarithmic factors are suppressed.
    Our analysis improves SCAFFOLD's communication complexity from
    order $LD_0^2/\varepsilon$ to
    $\delta D_0^2/\varepsilon$, matching that of regularized methods
    without requiring exact-gradient evaluations.
}
\label{tab:convex-stochastic-comparison}
\smallskip
\small
\setlength{\tabcolsep}{5pt}
\renewcommand{\arraystretch}{1.3}
\begin{tabular*}{\textwidth}{@{\extracolsep{\fill}}lccc@{}}
\toprule
Method & \shortstack{Communication\\rounds}
 & \shortstack{Stochastic gradient\\evaluations\textsuperscript{a}}
 & \shortstack{Exact gradient\\evaluations\textsuperscript{b}} \\
\midrule
\shortstack[l]{Minibatch SGD\\\citep{woodworth2020minibatch}}
 & \(\dfrac{LD_0^2}{\varepsilon}\)
 & \(\dfrac{LD_0^2}{\varepsilon}+\dfrac{\sigma^2D_0^2}{N\varepsilon^2}\)
 & \(0\) \\
\addlinespace[2pt]
\shortstack[l]{SCAFFOLD, Option II\\\citep[Theorem~III]{karimireddy2020scaffold}}
 & \(\dfrac{LD_0^2}{\varepsilon}\)
 & \(\dfrac{LD_0^2}{\varepsilon}+\dfrac{\sigma^2D_0^2}{N\varepsilon^2}\)
 & \(0\) \\
\addlinespace[2pt]
\shortstack[l]{FedRed-SGD\\\citep[Theorem~9]{jiang2024fedred}}
 & \(\dfrac{\delta D_0^2}{\varepsilon}\)
 & \(\dfrac{LD_0^2}{\varepsilon}+\dfrac{\sigma^2D_0^2}{\varepsilon^2}\)
 & \(\dfrac{\delta D_0^2}{\varepsilon}\) \\
\addlinespace[2pt]
\shortstack[l]{S-DANE with SGD\\\citep[Corollary~21]{jiang2024sdane}}
 & \(\dfrac{\delta D_0^2}{\varepsilon}\)
 & \(\dfrac{LD_0^2}{\varepsilon}+\dfrac{L\sigma^2D_0^2}{\delta\varepsilon^2}\)
 & \(\dfrac{\delta D_0^2}{\varepsilon}\) \\
\addlinespace[2pt]
\textbf{This work (\(B=K\))}
 & \(\dfrac{\delta D_0^2}{\varepsilon}\)
 & \(\dfrac{LD_0^2}{\varepsilon}+\dfrac{\sigma^2D_0^2}{\varepsilon^2}\)
 & \(0\) \\
\bottomrule
\end{tabular*}
\par\smallskip
\begin{minipage}{\textwidth}
\footnotesize
\textsuperscript{a} S-DANE's stochastic-gradient bound hides a factor
\(\log(e+L/\delta)\) and our bound hides it only in the noise term. \\
\textsuperscript{b} Exact-gradient evaluations are counted separately from stochastic
samples, with no assumption on their relative cost. FedRed's
communication and exact-gradient counts are in expectation.
\end{minipage}
\end{table*}

\paragraph{Stochastic communication and computation.}
Table~\ref{tab:convex-stochastic-comparison} compares the communication
and computation complexities required to achieve a prescribed accuracy. For \(0<\delta\le L_{\mathrm{loc}}\)
and \(0<\varepsilon\le\delta D_0^2\), take
\(\alpha=c/\delta\) with a sufficiently small universal constant \(c>0\).
Proposition~\ref{prop:fixed-stepsize} shows that
\begin{equation}
\begin{aligned}
 R&=\Theta\!\left(\frac{\delta D_0^2}{\varepsilon}\right),\\
 B=K&=\Theta\!\left(\frac{L_{\mathrm{loc}}}{\delta}
 +\frac{\sigma^2}{\delta\varepsilon}
       \log\!\left(e+\frac{L_{\mathrm{loc}}}{\delta}\right)\right)
\end{aligned}
\label{eq:convex-stochastic-complexity-tuning}
\end{equation}
suffice to reach \(\Delta_R\le\varepsilon\) when these constants are sufficiently large.
Including control initialization and omitting the control update after the last round,
the total number of stochastic gradients per client is \(2KR\), i.e.
\begin{equation}
 \mathcal O\!\left(\frac{L_{\mathrm{loc}}D_0^2}{\varepsilon}
 +\frac{\sigma^2D_0^2}{\varepsilon^2}
       \log\!\left(e+\frac{L_{\mathrm{loc}}}{\delta}\right)\right).
\label{eq:convex-stochastic-computation}
\end{equation}

Thus, stochastic corrected steps match FedRed-SGD's communication
complexity and, up to a logarithmic factor, its stochastic-gradient
complexity, without requiring exact-gradient evaluations. The
communication complexity is independent of the noise level, at the
cost of increasing the number of local steps as the noise increases.
Compared with minibatch SGD, our communication bound improves by a
factor of order \(L_{\mathrm{loc}}/\delta\), while minibatch SGD keeps a
\(1/N\) advantage in the statistical computation term.

\paragraph{Proof overview.}
The two stepsize conditions of Proposition~\ref{prop:fixed-stepsize}
carry the main idea. Local smoothness constrains the individual stepsize
through \(\eta_\ell L_{\mathrm{loc}}\le1\), while similarity constrains
the effective stepsize of a round through \(\alpha\delta\le c\). Together
they produce the coefficient \(L_{\mathrm{loc}}/K+\delta\). The analysis
of \citet{karimireddy2020scaffold} instead bounds how far clients move
from the starting server point, which forces
\(\alpha=\mathcal O(1/L_{\mathrm{loc}})\). Clients can move far from this
point while staying close to one another, and we treat this movement as
optimization progress.

With exact gradients, each client's updates are gradient descent on the
convex corrected objective
\(\widetilde f_i(y)=f_i(y)-
\ip{\grad f_i(x^{r-2})-\grad f(x^{r-2})}{y}\), so the standard bound at
the endpoint scales as \(1/\alpha\) under
\(\eta_\ell L_{\mathrm{loc}}\le1\) alone. Following DANE+
\citep{jiang2024fedred}, but without proximal regularization, we compare
the clients' corrected objective values with the global objective at the
averaged endpoint. Similarity bounds the comparison error by the
disagreement between client endpoints, which the negative
squared-distance terms of the local bounds absorb. What remains is a
term in the server movement over two rounds, since the corrections are
evaluated at \(x^{r-2}\), and \(\alpha\delta\le c\) controls it. A
separate first-round estimate avoids an error term of order
\(L_{\mathrm{loc}}D_0^2/R\).

For the stochastic terms, a suffix-average argument \citep{shamir2013}
gives the factor \(\log(1+\alpha L_{\mathrm{loc}})\) in the local noise
term. When \(\alpha L_{\mathrm{loc}}\le1\), averaging the centered
control errors before comparing endpoints gives the sharper order
\(L_{\mathrm{loc}}\alpha^2\sigma^2/B\). Bounding how one client's
control batch changes the averaged update preserves the \(1/N\)
reduction in the shared-control term. Together, these estimates allow us to set
\(B=K\).

\subsubsection{Optimality of Theorem \ref{thm:scaffold-convex}}
\label{sec:constant-step-optimality}
\begin{theorem}[Matching lower bound]
\label{thm:constant-step-optimality}
Let \(\mathcal P\) be the class of finite-dimensional instances with \(N\ge2\) satisfying
Assumptions~\ref{ass:convex-smooth}, \ref{ass:similarity},
and~\ref{ass:stochastic-oracle}, with parameter upper bounds \(L_{\mathrm{loc}}>0, \sigma>0, 0\le\delta\le L_{\mathrm{loc}}/2\), and initial point \(x^0\) with distance \(D_0>0\).
Denote the best worst-case expected gap over constant stepsizes by
\begin{equation}
    \mathcal E_*\coloneqq
    \inf_{\eta_\ell>0}\sup_{\mathcal I\in\mathcal P}
    \bigl(\E_{\mathcal I}[f(\bar x^R)]-f^\star\bigr),
    \label{eq:lb-minimax}
\end{equation}
where $\bar x^R$ is the output of Algorithm~\ref{alg:scaffold-option-i} with \(B=K\ge2\), \(R\ge1\), and constant stepsize chosen before training. Then
\begin{equation}
    \begin{aligned}
        c\min\{L_{\mathrm{loc}}D_0^2,\mathcal T\}
        \le\mathcal E_*\le200\min\{L_{\mathrm{loc}}D_0^2,\mathcal T\}
    \end{aligned}
    \label{eq:lb-optimality}
\end{equation}
for a universal constant \(c>0\) and \(\mathcal T\) is defined in \eqref{eq:tuned-rate}.
\end{theorem}

The restriction \(K\ge2\) excludes the degenerate case \(K=1\), in which
the control corrections cancel after server averaging and the method
reduces to minibatch SGD.

\paragraph{Scope of optimality.}
The logarithmic lower-bound construction uses differentiable sampled
losses whose expectations are convex and $L_{\mathrm{loc}}$-smooth.
The sampled losses themselves need not be convex or
$L_{\mathrm{loc}}$-smooth.
Notably, it is consistent with Assumption~\ref{ass:stochastic-oracle}.

\paragraph{Lower-bound constructions.}
We combine three families of instances, each fixed before sampling.
Complete proofs are in Appendix~\ref{app:lower-bound}.

Flat and steep quadratic coordinates give the \(L_{\mathrm{loc}}D_0^2/(KR)\) term,
following constructions in
\citet[Section~5]{woodworth2020localsgd} and
\citet[Appendix~B.1]{glasgow2022}.
To obtain the similarity term, we construct heterogeneous quadratics
whose previous-round controls produce a second-order recurrence.
We choose the parameters so that averaging the server iterates
does not cancel the error.

Identical noisy quadratics give a statistical lower bound of order
\(\sigma D_0/\sqrt{NKR}\), also present in
\citet[Theorem~3.1]{glasgow2022}.
Because the updates are linear, the
minibatch-control errors cancel under client averaging, allowing us to
directly compute the noise contribution to the averaged server output.

For the local stochastic term, we adapt the asymmetric
piecewise-quadratic construction of \citet{woodworth2020localsgd},
following \citet{glasgow2022}.
Unequal contraction rates on opposite sides of the minimizer transform
centered noise into a mean displacement that persists under client averaging.
Using different coordinate curvatures, we construct fixed
sampled losses whose gradients generate perturbations at different local
steps, which produces the logarithmic term. The
small-\(\alpha L_{\mathrm{loc}}\) regime gives the cube-root term.
Concavity of the expected local endpoint shows that centered
minibatch-control errors do not reduce this bias. We establish
the lower bound for the averaged server output using fixed sampled
losses for both local updates and control variates.

\subsection{Nonconvex objectives}
\label{sec:nonconvex}

We next analyze nonconvex objectives that are bounded below, \(f(x)\ge f_{\inf}>-\infty\) with
\(F_0\coloneqq f(x^0)-f_{\inf}>0\). We assume global smoothness (\Cref{ass:global-smoothness}) and
uniform Hessian similarity (\Cref{ass:uniform-similarity}), which keeps the roles of the constants \(L_g\)
and \(\delta_i\) separate. We study SCAFFOLD with client selection
(Algorithm~\ref{alg:scaffold-nonconvex})
and measure stationarity at a uniformly sampled local iterate.

\begin{algorithm}[t]
\caption{SCAFFOLD with client selection}
\label{alg:scaffold-nonconvex}
\begin{algorithmic}[1]
\Require Model \(x^0\), rounds \(R\), local steps \(K\), and nonempty set \(S\subseteq[N]\)
\Require Stepsize \(\eta_\ell\) and control batch size \(B\)
\State Each client draws \(B\) samples \(\xi_{i,b}^{0,c}\sim\mathcal D_i\)
\State Each client sets \(c_i^0\gets B^{-1}\sum_{b=1}^B\grad f_i(x^0;\xi_{i,b}^{0,c})\)
\State Server sets \(c^0\gets N^{-1}\sum_{i=1}^N c_i^0\)
\For{rounds \(r=1,\ldots,R\)}
    \State Select \(i_r\in S\) independently of the batches forming \((c_i^{r-1})_{i=1}^N\), conditional on the preceding history
    \State Server broadcasts \(x^{r-1}\) and \(i_r\), and sends \(c^{r-1}\) to client \(i_r\)
    \For{each client \(i\in[N]\) in parallel}
        \If{\(i=i_r\)}
            \State \(y_{i,0}^r\gets x^{r-1}\)
            \For{\(k=0,\ldots,K-1\)}
                \State Draw \(\xi_{i,k}^r\sim\mathcal D_i\)
                \State \(d_{i,k}^r\gets\grad f_i(y_{i,k}^r;\xi_{i,k}^r)-c_i^{r-1}+c^{r-1}\)
                \State \(y_{i,k+1}^r\gets y_{i,k}^r-\eta_\ell d_{i,k}^r\)
            \EndFor
            \State Independently sample \(\widetilde y^r\) uniformly from \(\{y_{i,k}^r:0\le k<K\}\)
        \EndIf
        \State Draw \(B\) new samples \(\xi_{i,b}^{r,c}\sim\mathcal D_i\)
        \State Set \(c_i^r\gets B^{-1}\sum_{b=1}^B\grad f_i(x^{r-1};\xi_{i,b}^{r,c})\)
        \State Send \(c_i^r\) to the server, together with \((y_{i,K}^r,\widetilde y^r)\) if \(i=i_r\)
    \EndFor
    \State \(x^r\gets y_{i_r,K}^r\),\quad \(c^r\gets N^{-1}\sum_{i=1}^N c_i^r\)
\EndFor
\State \Return \(\widehat x\), sampled independently and uniformly from \(\{\widetilde y^1,\ldots,\widetilde y^R\}\)
\end{algorithmic}
\end{algorithm}

\paragraph{SCAFFOLD with client selection.}
Algorithm~\ref{alg:scaffold-nonconvex} fixes a nonempty set
\(S\subseteq[N]\) of clients eligible to perform local optimization and
lets \(\bar\delta_S\coloneqq\max_{i\in S}\delta_i\). A singleton \(S\)
uses the same client throughout, while \(S=[N]\) allows any client to be
selected. In each round one client \(i_r\in S\) performs the \(K\)
corrected stochastic steps from \(x^{r-1}\), and the server takes its
endpoint. The choice may depend on preceding iterates and previously
used controls, but, conditional on that history, must be independent of
the batches used to form the controls \((c_i^{r-1})_{i=1}^N\), including
at initialization. This ensures that the selected client's
control-estimation error remains conditionally centered. Although only the selected client performs local optimization, all
clients compute and send their next control variates in each round.
Thus, with minibatch size $B=K$, each round uses $(N+1)K$
stochastic gradient evaluations.

Algorithm~\ref{alg:scaffold-nonconvex} outputs a point $\widehat x$
sampled uniformly from the $RK$ local iterates
$\{y_{i_r,k}^r:1\le r\le R,\ 0\le k<K\}$. Its expected squared gradient
norm is therefore
\begin{equation}
    \E\norm{\grad f(\widehat x)}^2
    =\frac1{RK}\sum_{r=1}^R\sum_{k=0}^{K-1}
          \E\norm{\grad f(y_{i_r,k}^r)}^2,
    \label{eq:nc-output}
\end{equation}
which the following results bound.

\begin{proposition}[Fixed stepsize, nonconvex]
\label{prop:nc-fixed-stepsize}
Suppose Assumptions~\ref{ass:global-smoothness},
\ref{ass:uniform-similarity}, and~\ref{ass:stochastic-oracle} hold. The output of
Algorithm~\ref{alg:scaffold-nonconvex} with \(B=K\), \(R,K\ge1\), and a
constant stepsize satisfying \(\eta_\ell L_g\le1/4\) and
\(\eta_\ell K\bar\delta_S\le1/8\) satisfies
\begin{equation}
    \E\norm{\grad f(\widehat x)}^2
    \le\frac{2F_0}{\eta_\ell RK}
       +32\bigl(L_g\eta_\ell+\bar\delta_S^2K\eta_\ell^2\bigr)\sigma^2.
    \label{eq:nc-fixed-stepsize-rate}
\end{equation}
\end{proposition}

As in the convex case, smoothness constrains the individual stepsize and
similarity constrains the effective stepsize of a round, so at the
largest admissible stepsize the first term is of order
\((L_g/K+\bar\delta_S)F_0/R\).
Setting \(B=K\) and tuning the stepsize gives the following theorem.
Communication complexity and the analysis of other control minibatch
sizes are in Appendix~\ref{app:nonconvex-tuning}.

\begin{theorem}[Nonconvex convergence]
\label{thm:scaffold-nonconvex}
Suppose Assumptions~\ref{ass:global-smoothness},
\ref{ass:uniform-similarity}, and~\ref{ass:stochastic-oracle} hold.
The output of Algorithm~\ref{alg:scaffold-nonconvex} with \(B=K\),
\(R,K\ge1\), and the constant stepsize specified in
\eqref{eq:nc-refined-tuned-step} of Appendix~\ref{app:nonconvex-tuning} satisfies
\begin{equation}
\begin{aligned}
    \E\norm{\grad f(\widehat x)}^2
    =\mathcal O\!\Biggl(&\left(\frac{L_g}{K}+\bar\delta_S\right)\frac{F_0}{R}+\sigma\sqrt{\frac{L_gF_0}{KR}}\\
       &+\left(\frac{\bar\delta_S^2\sigma^2F_0^2}{KR^2}\right)^{1/3}\Biggr).
\end{aligned}
\label{eq:nc-tuned-rate}
\end{equation}
\end{theorem}

The first term is the only one that remains with exact gradients
(\(\sigma=0\)). It again separates global smoothness from uniform Hessian
similarity: increasing \(K\) reduces the contribution \(L_g/K\) until
\(\bar\delta_S\) becomes the limiting factor.
The remaining terms capture the noise of stochastic gradients and its interaction with client dissimilarity.

\paragraph{Proof idea.}
A control estimate computed from a minibatch of size \(B\) has error
variance at most \(\sigma^2/B\). Since the same estimate is used in all
\(K\) local steps, its error can accumulate throughout the round.

The control-estimation error has zero conditional mean given the
preceding iterates and the selected client. Its effect on the expected
objective decrease therefore comes from the displacement of local
iterates from their starting point. When \(L_g\eta_\ell K\le1\), global
smoothness bounds the resulting gradient change, allowing us to replace
the direct error bound of order \(\sigma^2/B\) with a contribution of
order
\[
    L_g\eta_\ell K\,\frac{\sigma^2}{B}.
\]
When \(B=K\), this has the same order as the usual local-gradient noise
term \(L_g\eta_\ell\sigma^2\). When \(L_g\eta_\ell K>1\), the direct
error bound is already no larger than this contribution.
Finally, we account for uniform Hessian similarity and the use of controls evaluated at an
earlier server iterate to obtain the stated convergence rate.
The complete proof is in Appendix~\ref{app:nonconvex}.

\subsubsection{Local smoothness and optimality}
\label{sec:nc-optimality}

We now specialize to a class with a common local smoothness bound
\(L_{\mathrm{loc}}\). Taking \(L_g=L_{\mathrm{loc}}\), we may choose a
common uniform-similarity bound \(\bar\delta\le2L_{\mathrm{loc}}\) with
\(\delta_i\le\bar\delta\) for every client, which is always possible
because local smoothness gives \(\delta_i\le2L_{\mathrm{loc}}\).
With the stepsize choice in Appendix~\ref{app:nc-local-tuning},
we obtain
\begin{equation}
\begin{aligned}
 &\E\norm{\grad f(\widehat x)}^2
 =\mathcal O\!\Biggl(\min\Biggl\{L_{\mathrm{loc}}F_0,\\
 &\hspace{2.1cm}\left(\frac{L_{\mathrm{loc}}}{K}+\bar\delta\right)\frac{F_0}{R}
 +\sigma\sqrt{\frac{L_{\mathrm{loc}}F_0}{KR}}\Biggr\}\Biggr).
\end{aligned}
\label{eq:nc-local-rate}
\end{equation}
The cube-root term of \eqref{eq:nc-tuned-rate} does not appear here.
Indeed, with \(A=\bar\delta F_0/R\) and
\(B=\sigma\sqrt{L_{\mathrm{loc}}F_0/(KR)}\), it can be rewritten as
\(\bigl((\bar\delta/L_{\mathrm{loc}})AB^2\bigr)^{1/3}\), and \(\bar\delta\le2L_{\mathrm{loc}}\) with the AM--GM
inequality show that it is dominated by the other two terms up to a constant and can therefore be absorbed into the \(\mathcal O(\cdot)\) bound. 

In the regime $0<\bar\delta\le L_{\mathrm{loc}}$ considered in the
nonconvex analysis of FedRed-SGD, choosing
$K=\Theta(L_{\mathrm{loc}}/\bar\delta)$ recovers its convergence
rate under uniform Hessian similarity, without requiring exact-gradient
evaluations. 
Here only one client performs local optimization in each round, whereas
minibatch SGD and the method of
\citet{luo2025revisiting} compute on all clients and retain a $1/\sqrt N$ factor
in the stochastic term. We summarize the assumptions and rates of the related works in the nonconvex setting in Table~\ref{tab:nonconvex-comparison} and move towards the corresponding lower bound.

\begin{theorem}[Nonconvex lower bound]
\label{thm:nc-lower}
Fix \(N\ge2\), a nonempty \(S\subseteq[N]\), integers \(R,K\ge1\),
and \(L_{\mathrm{loc}},F_0>0\), \(\sigma\ge0\),
\(0\le\bar\delta\le2L_{\mathrm{loc}}\). Suppose
\(RK\ge2\).
For the client-selection variant with \(B=K\), any prescribed positive
constant stepsize, and any client-selection rule allowed above, there
is a finite-dimensional instance with
\(L_{\mathrm{loc}}\)-smooth local objectives (see \eqref{eq:local-smoothness}), uniform similarity at
most \(\bar\delta\) (\Cref{ass:uniform-similarity}), initial objective gap at most \(F_0\), and gradient-noise
variance at most \(\sigma^2\), such that
\begin{equation}
\begin{aligned}
 &\E\norm{\grad f(\widehat x)}^2
 \ge c\min\Biggl\{L_{\mathrm{loc}}F_0,\\
 &\hspace{2.5cm}\left(\frac{L_{\mathrm{loc}}}{K}+\bar\delta\right)\frac{F_0}{R}
 +\sigma\sqrt{\frac{L_{\mathrm{loc}}F_0}{KR}}\Biggr\}
\end{aligned}
\label{eq:nc-lower}
\end{equation}
for a universal constant \(c>0\).
\end{theorem}

Therefore, \eqref{eq:nc-local-rate} is optimal up to universal constants
for the prescribed constant-stepsize method and output under local
smoothness for \(RK\ge2\).

\paragraph{Lower-bound constructions.}
We adapt the quadratic SGD examples of
\citet[Propositions~1--2 and Supplement~A.4--A.5]{drori2020}.
Small stepsizes leave a substantial portion of the initial error
uncorrected, whereas larger stepsizes sustain the effect of gradient
noise. The centered control errors add a nonnegative term to
the conditional second moment, preserving the noisy-quadratic lower
bound.

To establish the similarity-dependent term, we adapt the
alternating-client chain of
\citet[Appendix~B, Theorem~B.1]{patel2022communication}, which builds
on \citet{carmon2020}. Alternating links limit the progress achievable
per communication round. The oracle is noiseless, and our analysis
accounts for initialization and the use of previous-round controls at
every selected local iterate. Appendix~\ref{app:nc-lower} provides
the complete proofs and the resulting matching communication bounds.

\begin{figure}[!tb]
    \centering
    \includegraphics[width=\columnwidth]{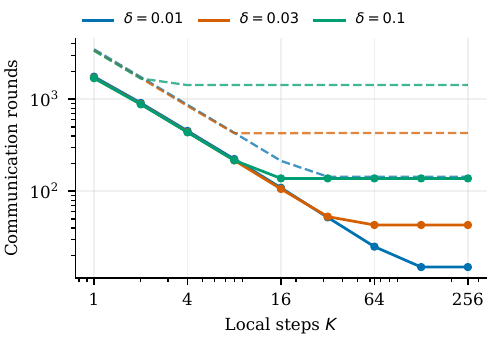}
    \caption{
    Number of communication rounds deterministic SCAFFOLD needs to reach \(f(\bar x^R)-f^\star\le10^{-5}\). Dashed curves use theoretical
    stepsizes, and solid curves use tuned stepsizes.
    Our theory predicts that increasing the number of local steps reduces communication until client
    similarity limits further gains.}
    \label{fig:convex-communication}
\end{figure}

\begin{figure*}[t]
    \centering
    \includegraphics[width=\textwidth]{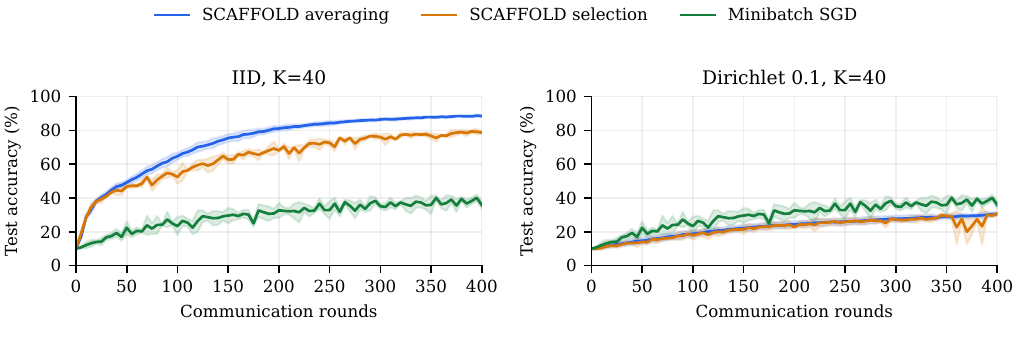}
    \caption{
        Server test accuracies on CIFAR-10 with \(K=40\) for the IID (left) and Dirichlet (right) partitions, reporting means $\pm$ standard deviations over
    three runs.
    Both SCAFFOLD variants outperform minibatch SGD on the IID
    partition, but this advantage disappears on the heterogeneous partition.}
    \label{fig:nn-comparison}
\end{figure*}

\section{Experiments}
\label{sec:experiments}

\paragraph{Communication versus local computation.}
We use a quadratic family adapted from our deterministic lower-bound
construction (Lemma~\ref{lem:stale-quadratic-recurrence}), with independently
controlled local smoothness and Hessian similarity. For
\(L_{\mathrm{loc}}=1\), we vary \(\delta\in\{0.01,0.03,0.1\}\) and
\(K\in\{1,2,4,\ldots,256\}\). With exact gradients, we measure the
communication rounds needed to reach the target objective gap
\(f(\bar x^R)-f^\star\le10^{-5}\), using theoretical and tuned stepsizes
(Figure~\ref{fig:convex-communication}). The experiment illustrates how
increasing local computation reduces communication until client
similarity becomes the limiting factor, consistent with the
\(L_{\mathrm{loc}}/K+\delta\) dependence in our theory. The transition
occurs when \(K\) is of order \(L_{\mathrm{loc}}/\delta\), beyond which
the similarity term limits further communication gains. Smaller
\(\delta\) therefore allows the improvement to continue over a wider
range of \(K\), and tuning reduces round counts while preserving this
pattern. Appendix~\ref{app:experiments} gives the objectives, protocol,
and stepsizes.

\paragraph{Neural-network training.}
We examine whether similarity between client datasets improves the
practical benefit of local computation in neural-network training.
We compare SCAFFOLD with averaging, SCAFFOLD with client selection, and
minibatch SGD on CIFAR-10 \citep{krizhevsky2009cifar}, adapting the
ResNet-18 \citep{he2016resnet} implementation and data preprocessing
from S-DANE \citep{jiang2024sdane}, with BatchNorm replaced by GroupNorm
\citep{wu2018groupnorm} to avoid batch-dependent normalization
statistics, whose drift SCAFFOLD does not correct
\citep{quintana2024bnscaffold}. We distribute the same training set
among ten clients in two ways. An IID split gives each client a random
subset of the pooled data, with similar class proportions. A Dirichlet
split draws the allocation proportions of each class across clients.
With concentration \(0.1\), this produces strongly unequal class
proportions and client sizes. Both partitions preserve the pooled
training set and the example-weighted global objective, so they provide
a statistical comparison motivated by the similarity assumptions without
prescribing a Hessian-similarity constant.

We train for \(400\) communication rounds with \(K\in\{10,40\}\),
using validation-tuned constant stepsizes and three random seeds.
Minibatch SGD uses the same per-round gradient-sample budget as
SCAFFOLD with averaging, including control estimation. It samples
from the pooled training set, and we use the same baseline runs for
both partitions. Appendix~\ref{app:nn-experiments} gives the full
protocol and results for both values of \(K\).

Figure~\ref{fig:nn-comparison} shows server test accuracy for \(K=40\).
Both SCAFFOLD variants outperform minibatch SGD on the IID partition
but fall below the same baseline on the Dirichlet partition.
Increasing \(K\) from \(10\) to \(40\) improves both SCAFFOLD variants
on IID data at the same communication-round budget, but gives no
consistent improvement on the heterogeneous partition. The
validation-selected local stepsizes also differ sharply across
partitions. For \(K=40\), validation selects \(0.2\) and \(0.01\) for
IID averaging and selection, respectively, compared with
\(3\times10^{-5}\) for both on Dirichlet data. The smaller
validation-selected stepsizes suggest that more conservative local
updates perform better on the Dirichlet partition, consistent with the
observed loss spikes (Appendix~\ref{app:nn-experiments}). Thus, in this experiment,
similarity between client datasets is associated with a substantial
advantage of both SCAFFOLD variants over minibatch SGD.

Our nonconvex guarantee uses client selection, as do several related
similarity-based analyses \citep{murata2021bvr,patel2022communication}.
On the IID partition, however, averaging reaches higher accuracy at the
same round budget for both values of \(K\). This empirical advantage
motivates sharper nonconvex guarantees for the averaging variant.
Selection uses only \(55\%\) of averaging's gradient-sample evaluations
per round. Over the shared range of gradient-sample budgets, the curves
are much closer, with selection ahead during part of training
(Figure~\ref{fig:nn-computation} in Appendix~\ref{app:nn-experiments}).

Comparing both partitions at the same models along the common
minibatch-SGD trajectory, gradient-difference measurements consistently
show larger client disagreement on the Dirichlet partition.
These are finite-sample measurements on
sampled objectives, not estimates of the similarity constants.
Appendix~\ref{app:nn-diagnostics} reports them and their sensitivity to
the perturbation radius and number of sampled examples.

\section{Conclusion}
We showed that, with minibatch Option-I controls, the convex convergence
guarantee for SCAFFOLD depends on $L_{\mathrm{loc}}/K+\delta$:
increasing the number of local steps reduces communication until client
similarity becomes the limiting factor. A client-selection variant
achieves an analogous linear dependence on the similarity parameter for
nonconvex objectives. Both guarantees
extend to stochastic gradients. Matching lower bounds
establish their optimality for the analyzed constant-stepsize methods
and outputs, under local smoothness in the nonconvex case.
Under partial participation, refreshing controls only on participating
clients can prevent the full communication benefit of local steps,
even for identical clients.

In our neural-network experiments, SCAFFOLD substantially outperformed
minibatch SGD on the IID partition but showed no such advantage on the
strongly heterogeneous partition studied. On IID data, averaging also outperformed
client selection at the same number of communication rounds, which
motivates extending the nonconvex analysis to the averaging variant.

\subsubsection*{Acknowledgments}
This work was done during the first author's research internship at
CISPA Helmholtz Center for Information Security.
SH and SS gratefully acknowledge funding from the European Research
Council (ERC) under the Horizon Europe Framework Programme (HORIZON) for
proposal number 101170430 CollectiveMinds, which also partially funded
ZT's visit to CISPA.
ZT and PR were supported by funding from King Abdullah University of
Science and Technology (KAUST): i) KAUST Baseline Research Scheme,
ii) Center of Excellence for Generative AI (award no.~5940),
iii) Competitive Research Grant (CRG) Program (award no.~6460).

\bibliography{references}
\clearpage
\appendix
\onecolumn
\aistatstitle{SCAFFOLD under Hessian Similarity: Sharp Convergence Rates \\
Appendix}

\section{Proofs of Proposition~\ref{prop:fixed-stepsize} and Theorem~\ref{thm:scaffold-convex}}
\label{app:scaffold-convex-proof}

We first allow any control batch size \(B\ge1\), with
\(\eta_\ell L_{\mathrm{loc}}\le1\) and
\(\theta=2\alpha^2\delta^2\le1/9\), which is the condition
\(\alpha\delta\le1/(3\sqrt2)\) of Proposition~\ref{prop:fixed-stepsize}.
This fixed-stepsize analysis proves the proposition and also allows
\(\sigma=0\). We then combine it with
a small-stepsize bound to prove the theorem using \(B=K\) and
\eqref{eq:app-complete-tuning}.

For each round \(r\), let \(\mathcal H_r\) contain the previous local
trajectories and previously used controls, but not the batch used to form
\((c_i^{r-1})_{i=1}^N\). In particular, \(x^{r-1}\) and \(x^{r-2}\) are
\(\mathcal H_r\)-measurable. This conditioning is valid because that batch
did not generate \(x^{r-1}\). Although \(c_i^0\) and \(c_i^1\) are both
evaluated at \(x^0\), the algorithm draws a new batch for \(c_i^1\).

Write the control errors in the current round as
\begin{equation}
 e_i=c_i^{r-1}-\grad f_i(x^{r-2}),\qquad
 \bar e=\frac1N\sum_i e_i,\qquad q_i=\bar e-e_i.
 \label{eq:app-control-errors}
\end{equation}
Conditionally on \(\mathcal H_r\), the \(e_i\)'s are independent and centered,
and \(\E[\norm{e_i}^2\mid\mathcal H_r]\le\sigma^2/B\). Consequently,
\begin{equation}
 \sum_iq_i=0,\qquad \E[q_i\mid\mathcal H_r]=0,\qquad
 \E[\norm{q_i}^2\mid\mathcal H_r]
 \le\left(1-\frac1N\right)\frac{\sigma^2}{B}.
 \label{eq:app-control-moments}
\end{equation}
The corrected errors \(q_i\) need not be independent across clients.

\subsection{Server-level reduction}

The server analysis is driven by a conditional comparison between the
objective gap accumulated in one round and the corresponding decrease in
squared distance.  To state this comparison, write
\[
    \E_{r-1}[\,\cdot\,]
    \coloneqq
    \E[\,\cdot\mid\mathcal H_r],
\]
and, for \(r\ge1\), define
\[
     G_r\coloneqq f(x^r)-f^\star,
    \qquad
    D_r\coloneqq\norm{x^r-x^\star},
    \qquad
    b_r\coloneqq\norm{x^r-x^{r-1}},
    \qquad
    s_r\coloneqq\norm{x^r-x^{r-2}}.
\]
The required conditional one-round comparison, used with
\(u=x^\star\) and \(u=x^{r-1}\), is
\begin{equation}
\begin{aligned}
    2\alpha\E_{r-1}[f(x^r)-f(u)]
    \le
    \norm{x^{r-1}-u}^2
    -\E_{r-1}\norm{x^r-u}^2 +\theta\E_{r-1}\norm{x^r-x^{r-2}}^2+Q,
\end{aligned}
\label{eq:app-joint-endpoint-template}
\end{equation}
where \(Q\ge0\) is deterministic.  The first choice of \(u\) gives a
telescoping sum of squared distances, while the second controls the server displacement.  The
next lemma shows that these two consequences are already sufficient for the
desired averaged-output guarantee.  Thus, after proving the lemma, all that
will remain is to verify \eqref{eq:app-joint-endpoint-template} with a sharp
choice of \(Q\).

\begin{lemma}[Server reduction under a one-round comparison]
\label{lem:app-server-reduction}
Suppose \(x^{-1}=x^0\), \(\alpha>0\), \(0\le\theta\le1/9\), and
\eqref{eq:app-joint-endpoint-template} holds in every round, with the same
deterministic \(Q\ge0\), for \(u=x^\star\) and \(u=x^{r-1}\).  Then
\begin{equation}
    \E[f(\bar x^R)]-f^\star
    \le
    \frac{113}{35}\frac{D_0^2}{\alpha R}
    +\frac{81}{35}\frac{Q}{\alpha}.
    \label{eq:app-generic-server-rate}
\end{equation}
\end{lemma}

\begin{proof}
All expectations in this proof are unconditional.  Substituting
\(u=x^\star\) and \(u=x^{r-1}\), respectively, into
\eqref{eq:app-joint-endpoint-template} gives
\begin{align}
    2\alpha\E G_r
    &\le
    \E D_{r-1}^2-\E D_r^2+\theta\E s_r^2+Q,
    \label{eq:app-round-gap}\\
    2\alpha(\E G_{r-1}-\E G_r)
    &\ge
    \E b_r^2-\theta\E s_r^2-Q.
    \label{eq:app-round-descent}
\end{align}
The gap inequality \eqref{eq:app-round-gap} has the desired telescoping
squared-distance terms, except for the stale displacement \(s_r\).  The descent inequality
\eqref{eq:app-round-descent} is useful precisely because it relates this
error to the consecutive displacement \(b_r\).  Since
\[
    s_r
    =\norm{(x^r-x^{r-1})+(x^{r-1}-x^{r-2})},
    \qquad
    s_r^2\le2b_r^2+2b_{r-1}^2,
\]
the stale error involves two neighboring rounds.  We first control the
initial round and then derive a Lyapunov recursion for the subsequent rounds.

\par\medskip\noindent\textbf{Controlling the initial round.}
Because \(x^{-1}=x^0\), we have \(s_1=b_1\).  Since
\(G_1\ge0\), \eqref{eq:app-round-gap} implies
\[
    \E D_1^2\le D_0^2+\theta\E b_1^2+Q.
\]
The inequality \(b_1^2\le2D_0^2+2D_1^2\) therefore yields
\[
    (1-2\theta)\E D_1^2
    \le(1+2\theta)D_0^2+Q.
\]
Using \(\theta\le1/9\), we obtain
\begin{equation}
    \E D_1^2\le\frac{11}{7}D_0^2+\frac97Q,
    \qquad
    \E b_1^2\le\frac{36}{7}D_0^2+\frac{18}{7}Q.
    \label{eq:app-first-distance}
\end{equation}
Dropping the nonpositive term \(-\E D_1^2\) from
\eqref{eq:app-round-gap} and using \eqref{eq:app-first-distance} gives
\begin{equation}
    2\alpha\E G_1
    \le\frac{11}{7}D_0^2+\frac97Q.
    \label{eq:app-first-gap}
\end{equation}
Thus the boundary terms are controlled by \(D_0\) and \(Q\) alone,
with initialization noise included in \(Q\).

\par\medskip\noindent\textbf{A Lyapunov recursion.}
For \(r\ge2\), substitute
\(s_r^2\le2b_r^2+2b_{r-1}^2\) into
\eqref{eq:app-round-descent} and rearrange:
\begin{equation}
    2\alpha\E G_r
    \le
    2\alpha\E G_{r-1}
    -(1-2\theta)\E b_r^2
    +2\theta\E b_{r-1}^2+Q.
    \label{eq:app-before-potential}
\end{equation}
The positive term \(2\theta\E b_{r-1}^2\) prevents the objective gap from
being a decreasing quantity.  It also identifies the correction to use:
attach \(2\theta\E b_r^2\) to the current gap, so that the troublesome
preceding-round term becomes part of the same quantity at round \(r-1\).
Accordingly, define the Lyapunov quantity
\begin{equation}
    P_r\coloneqq
    2\alpha\E G_r+2\theta\E b_r^2.
    \label{eq:app-server-potential}
\end{equation}
Equations \eqref{eq:app-first-distance} and
\eqref{eq:app-first-gap} imply
\begin{equation}
    P_1\le\frac{19}{7}D_0^2+\frac{13}{7}Q.
    \label{eq:app-first-potential}
\end{equation}
Adding \(2\theta\E b_r^2\) to
\eqref{eq:app-before-potential} now gives
\[
    P_{r-1}-P_r\ge(1-4\theta)\E b_r^2-Q.
\]
Thus the same construction both neutralizes the delayed term and controls
the total movement of the server.  Because \(P_R\ge0\), summing this
inequality yields
\begin{equation}
    \sum_{r=2}^R\E b_r^2
    \le
    \frac{P_1+(R-1)Q}{1-4\theta}.
    \label{eq:app-displacement-sum}
\end{equation}

It remains to return to the objective gaps.  For \(R\ge2\), sum
\eqref{eq:app-round-gap} over \(r=2,\ldots,R\), use
\[
    \sum_{r=2}^R\E s_r^2
    \le
    2\E b_1^2+4\sum_{r=2}^R\E b_r^2,
\]
and add \(2\alpha\E G_1\).  The squared distances telescope, and
\eqref{eq:app-server-potential} gives
\begin{align*}
    2\alpha\sum_{r=1}^R\E G_r
    &\le
    \E D_1^2+P_1
    +4\theta\sum_{r=2}^R\E b_r^2+(R-1)Q\\
    &\le
    \E D_1^2+
    \frac{P_1}{1-4\theta}
    +\frac{R-1}{1-4\theta}Q\\
    &\le
    \frac{226}{35}D_0^2+
    \frac{63R+99}{35}Q,
\end{align*}
where the last line uses
\((1-4\theta)^{-1}\le9/5\) and the initial estimates above.  For
\(R=1\), the same final estimate follows directly from
\eqref{eq:app-first-gap}.  Finally, convexity of \(f\) gives
\[
    \E[f(\bar x^R)]-f^\star
    \le\frac1R\sum_{r=1}^R\E G_r.
\]
Dividing the preceding bound by \(2\alpha R\) and using
\((63R+99)/(70R)\le81/35\) proves
\eqref{eq:app-generic-server-rate}.
\end{proof}

\subsection{A common local stochastic-gradient calculation}

The next lemma records the standard one-step SGD descent calculation; see
\citet[the proofs of Theorems~1 and~2]{shamir2013} and, for a more general
formulation, the proof of Lemma~4.1 in Appendix~B.1 of \citet{liu2026}.

\begin{lemma}[One stochastic-gradient step]
\label{lem:app-stochastic-one-step}
Let \(h:\R^d\to\R\) be convex and \(L\)-smooth.  Let \(z\) and \(u\) be
measurable with respect to a sigma-field \(\mathcal G\), and let
\[
    z^+=z-\eta(\grad h(z)+\zeta),
\]
where
\[
    \E[\zeta\mid\mathcal G]=0,
    \qquad
    q\coloneqq\E[\norm{\zeta}^2\mid\mathcal G]<\infty.
\]
If \(\eta>0\) and \(\eta L\le1\), then, almost surely,
\begin{align}
    2\eta\E[h(z^+)-h(u)\mid\mathcal G]
    \le{}&
    \norm{z-u}^2
    -\E[\norm{z^+-u}^2\mid\mathcal G]
    \notag\\
    &-\eta^2(1-\eta L)\norm{\grad h(z)}^2
    +\eta^2(1+\eta L)q,
    \label{eq:app-stochastic-one-step}
\end{align}
and
\begin{equation}
    \E[h(z^+)\mid\mathcal G]
    \le
    h(z)+\frac{L\eta^2}{2}q.
    \label{eq:app-one-step-objective-drift}
\end{equation}
\end{lemma}

\begin{proof}
Write \(g=\grad h(z)\). Since \(z\) and \(u\) are
\(\mathcal G\)-measurable, conditional unbiasedness gives
\[
    \E[\ip{g}{\zeta}\mid\mathcal G]=0,
    \qquad
    \E[\ip{z-u}{\zeta}\mid\mathcal G]=0.
\]
In particular, expanding the square yields
\[
    \E[\norm{g+\zeta}^2\mid\mathcal G]
    =\norm{g}^2+2\E[\ip{g}{\zeta}\mid\mathcal G]+q
    =\norm{g}^2+q.
\]
By smoothness and \(z^+-z=-\eta(g+\zeta)\),
\begin{align*}
    h(z^+)
    &\le h(z)+\ip{g}{z^+-z}+\frac L2\norm{z^+-z}^2\\
    &=h(z)-\eta\norm{g}^2-\eta\ip{g}{\zeta}
      +\frac{L\eta^2}{2}\norm{g+\zeta}^2.
\end{align*}
Taking conditional expectations and using the preceding identities gives
\begin{equation}
\begin{aligned}
    \E[h(z^+)\mid\mathcal G]
    \le{}&
    h(z)-\eta\left(1-\frac{\eta L}{2}\right)
        \norm{\grad h(z)}^2
    +\frac{L\eta^2}{2}q.
\end{aligned}
\label{eq:app-one-step-smoothness}
\end{equation}
Because \(\eta L\le1\), dropping the nonpositive gradient term proves
\eqref{eq:app-one-step-objective-drift}. Moreover, expanding the
squared distance before taking conditional expectations gives
\begin{align*}
    \norm{z^+-u}^2
    &=\norm{z-u}^2-2\eta\ip{g+\zeta}{z-u}
      +\eta^2\norm{g+\zeta}^2.
\end{align*}
Thus, after rearranging,
\begin{align*}
    2\eta\ip{g}{z-u}
    =\norm{z-u}^2-\E[\norm{z^+-u}^2\mid\mathcal G]+\eta^2\norm{g}^2+\eta^2q.
\end{align*}
Convexity gives \(h(z)-h(u)\le\ip{g}{z-u}\). Subtracting
\(h(u)\) from \eqref{eq:app-one-step-smoothness}, multiplying by
\(2\eta\), and applying this convexity inequality yields
\begin{align*}
    2\eta\E[h(z^+)-h(u)\mid\mathcal G]
    \le{}2\eta\ip{g}{z-u}-2\eta^2\left(1-\frac{\eta L}{2}\right)\norm{g}^2
      +L\eta^3q.
\end{align*}
Substituting the rearranged distance identity, the coefficient of
\(\norm{g}^2\) becomes
\[
    \eta^2-2\eta^2\left(1-\frac{\eta L}{2}\right)
    =-\eta^2(1-\eta L),
\]
while the noise terms sum to
\(\eta^2q+L\eta^3q=\eta^2(1+\eta L)q\).
This proves \eqref{eq:app-stochastic-one-step}.
\end{proof}

\subsection{Direct control of the random local endpoints}

For a local SGD trajectory \(z_0,\ldots,z_K\) on an objective \(h\),
applying Lemma~\ref{lem:app-stochastic-one-step}, taking expectations,
and summing over the \(K\) updates bounds the trajectory-average gap.
The server uses the final client iterates, so we need to pass from
\[
    \underbrace{\frac1K\sum_{k=1}^K\E[h(z_k)-h(u)]}
        _{\text{trajectory-average gap}}
    \quad\text{to}\quad
    \underbrace{\E[h(z_K)-h(u)]}_{\text{final-iterate gap}},
\]
where \(u\) is a fixed comparison point. In doing so, we must retain
the subtracted final squared distance \(-\E\norm{z_K-u}^2\) on the
right-hand side of the bound. This term remains after the one-step
distance differences telescope:
\[
    \sum_{k=0}^{K-1}
    \bigl(\E\norm{z_k-u}^2-\E\norm{z_{k+1}-u}^2\bigr)
    =\norm{z_0-u}^2-\E\norm{z_K-u}^2,
\]
with deterministic \(z_0\). After averaging across clients, the subtracted
final squared distance supplies both the server-distance term used for
telescoping across rounds and a client-dispersion term that absorbs the
similarity error. The following lemma bounds the final-iterate gap while retaining the subtracted final squared distance, at the cost of a logarithmic factor in the noise bound. Its proof retains the
one-step derivation and uses suffix averages sharpened by smoothness near
the endpoint.

\begin{lemma}[Last-iterate bound for smooth convex SGD]
\label{lem:app-refined-endpoint}
Let \(h:\R^d\to\R\) be convex and \(L\)-smooth.  Starting from a
deterministic \(z_0\), consider
\[
    z_{k+1}=z_k-\eta\bigl(\grad h(z_k)+\zeta_k\bigr),
    \qquad k=0,\ldots,K-1,
\]
where \(z_k\) is measurable with respect to \(\mathcal G_k\), and
\[
    \E[\zeta_k\mid\mathcal G_k]=0,
    \qquad
    \E[\norm{\zeta_k}^2\mid\mathcal G_k]\le\nu^2.
\]
If \(\eta>0\) and \(\eta L\le1\), then, for every deterministic \(u\),
\begin{equation}
\begin{aligned}
    2K\eta\E[h(z_K)-h(u)]
    \le
    \norm{z_0-u}^2-\E\norm{z_K-u}^2+2K\eta^2\nu^2
        \bigl[1+\log(1+\eta LK)\bigr].
\end{aligned}
\label{eq:app-refined-endpoint}
\end{equation}
The same bound holds after conditioning on the initial information,
treating \(h\), \(z_0\), and \(u\) as fixed, provided the stated noise
assumptions hold under this conditioning.
\end{lemma}

\begin{proof}
\par\medskip\noindent\textbf{The trajectory-average bound.}
Apply Lemma~\ref{lem:app-stochastic-one-step} at step \(k\), discard its
nonpositive gradient term, and use the variance bound.  After taking
expectations,
\begin{equation}
\begin{aligned}
    2\eta\E[h(z_{k+1})-h(u)]
    \le
    \E\norm{z_k-u}^2-\E\norm{z_{k+1}-u}^2+\eta^2(1+\eta L)\nu^2.
\end{aligned}
\label{eq:app-local-one-step}
\end{equation}
Summing over \(k=0,\ldots,K-1\) gives
\begin{equation}
\begin{aligned}
    2\eta\sum_{k=1}^K\E[h(z_k)-h(u)]
    \le
    \norm{z_0-u}^2-\E\norm{z_K-u}^2+K\eta^2(1+\eta L)\nu^2.
\end{aligned}
\label{eq:app-local-average}
\end{equation}
This estimate already contains the endpoint distance required by
\eqref{eq:app-joint-endpoint-template}; only its objective term still
averages the whole trajectory.

\par\medskip\noindent\textbf{Bounding the final objective value using suffix averages.}
Fix \(j\in\{1,\ldots,K-1\}\).  Use \(u=z_j\) in
\eqref{eq:app-local-one-step} for steps \(j,\ldots,K-1\), conditioning first
on \(\mathcal G_j\).  This comparison point is valid because \(z_j\) is then
fixed and remains measurable with respect to every later sigma-field.  The
initial distance is zero, and hence
\[
    2\eta\sum_{k=j+1}^K
        \E[h(z_k)-h(z_j)]
    \le
    (K-j)\eta^2(1+\eta L)\nu^2.
\]
Define the numerical suffix averages
\[
    A_j\coloneqq
    \frac{1}{K-j+1}\sum_{k=j}^K\E h(z_k),
    \qquad j=1,\ldots,K.
\]
The preceding comparison gives
\[
    A_{j+1}-\E h(z_j)
    \le\frac{\eta(1+\eta L)\nu^2}{2}.
\]
Using the exact identity
\[
    A_{j+1}-A_j
    =
    \frac{A_{j+1}-\E h(z_j)}{K-j+1}
\]
and \((1+\eta L)/2\le1\), we obtain
\begin{equation}
    A_{j+1}-A_j
    \le
    \frac{\eta\nu^2}{K-j+1}.
    \label{eq:app-harmonic-suffix}
\end{equation}
This estimate is effective for long suffixes but unnecessarily large near
the endpoint.

\par\medskip\noindent\textbf{Using smoothness on short suffixes.}
Equation~\eqref{eq:app-one-step-objective-drift} and the variance bound imply
\[
    \E h(z_{k+1})-\E h(z_k)
    \le\frac{L\eta^2\nu^2}{2}.
\]
Let \(n=K-j\).  Summing this inequality from \(j\) to each later iterate
and then averaging gives
\[
    A_{j+1}-\E h(z_j)
    \le
    \frac{L\eta^2\nu^2(n+1)}{4}.
\]
Combining this estimate with the suffix identity and
\eqref{eq:app-harmonic-suffix} yields
\begin{equation}
    A_{j+1}-A_j
    \le
    \eta\nu^2
    \min\left\{
        \frac{1}{K-j+1},
        \frac{\eta L}{4}
    \right\}.
    \label{eq:app-clipped-suffix}
\end{equation}
Thus the harmonic estimate is used only where it is smaller than the
smoothness estimate.

Set \(c=\eta L/4\) for the remainder of this calculation.  Since
\(x\mapsto\min\{c,1/x\}\) is nonincreasing,
\[
    \sum_{n=1}^{K-1}\min\left\{c,\frac1{n+1}\right\}
    \le
    \int_1^K\min\left\{c,\frac1x\right\}\,dx
    \le
    \log(1+4cK).
\]
For the last inequality, if \(cK\le1\), the integral is at most
\(cK\le\log(1+4cK)\).  If \(cK>1\), it is at most
\(1+\log(cK)\le\log(1+4cK)\).  The sum is empty when \(K=1\).
Summing \eqref{eq:app-clipped-suffix} over \(j\) therefore gives
\begin{equation}
    \E h(z_K)
    \le
    \frac1K\sum_{k=1}^K\E h(z_k)
    +\eta\nu^2\log(1+\eta LK).
    \label{eq:app-average-to-endpoint}
\end{equation}
Multiplying \eqref{eq:app-average-to-endpoint} by \(2K\eta\), substituting
\eqref{eq:app-local-average}, and using \(1+\eta L\le2\) proves
\eqref{eq:app-refined-endpoint}.
\end{proof}

The suffix-average recursion follows \citet[Theorem~2]{shamir2013}, although
they use \(\eta_t=c/\sqrt{t}\); see also the variable-stepsize result of
\citet[Lemma~4.2]{liu2026}.  For the fixed stepsize considered here,
Lemma~\ref{lem:app-refined-endpoint} additionally uses smoothness near the
endpoint to replace \(\log K\) by \(\log(1+\eta LK)\) and retains the
terminal-distance term needed below.

\subsection{From client endpoints to the server endpoint}

The preceding lemma controls each client separately and therefore leaves
different corrected objectives evaluated at different endpoints.  The
following deterministic comparison converts their average into the global
objective at the averaged endpoint.  Its error depends on the dispersion of
the client endpoints; the negative endpoint-distance term retained in
Lemma~\ref{lem:app-refined-endpoint} will absorb this error.

For an anchor \(w\in\R^d\), define
\begin{equation}
    \widetilde f_i(y)
    \coloneqq
    f_i(y)
    -\ip{\grad f_i(w)-\grad f(w)}{y}.
    \label{eq:app-generic-corrected-objective}
\end{equation}
These functions are convex and \(L_{\mathrm{loc}}\)-smooth.  Moreover, at
every common point \(y\),
\[
    \frac1N\sum_{i=1}^N\widetilde f_i(y)=f(y).
\]

\begin{lemma}[Corrected-objective comparison]
\label{lem:app-corrected-comparison}
For arbitrary points \(y_1,\ldots,y_N\), let
\[
    \bar y\coloneqq\frac1N\sum_{i=1}^Ny_i,
    \qquad
    e^2\coloneqq\frac1N\sum_{i=1}^N\norm{y_i-\bar y}^2.
\]
Then, for every common comparison point \(u\),
\begin{equation}
    \frac1N\sum_{i=1}^N
        \bigl(\widetilde f_i(y_i)-\widetilde f_i(u)\bigr)
    \ge
    f(\bar y)-f(u)-\delta\norm{\bar y-w}\,e.
    \label{eq:app-corrected-comparison}
\end{equation}
\end{lemma}

\begin{proof}
Write \(z_i=y_i-\bar y\), so that \(N^{-1}\sum_i z_i=0\).  Convexity of
\(f_i\) at \(\bar y\) gives
\[
    f_i(y_i)
    \ge
    f_i(\bar y)+\ip{\grad f_i(\bar y)}{z_i}.
\]
Substituting this inequality into the corrected-objective differences and
averaging shows that their difference from \(f(\bar y)-f(u)\) is bounded
below by
\[
    \frac1N\sum_{i=1}^N
    \ip{
        \bigl(\grad f_i(\bar y)-\grad f(\bar y)\bigr)
        -
        \bigl(\grad f_i(w)-\grad f(w)\bigr)
    }{z_i}.
\]
The common terms vanish because the \(z_i\)'s average to zero.
Cauchy--Schwarz over the clients and
Assumption~\ref{ass:similarity} bound the absolute value of the last display
by
\[
    \delta\norm{\bar y-w}
    \left(\frac1N\sum_{i=1}^N\norm{z_i}^2\right)^{1/2}
    =
    \delta\norm{\bar y-w}\,e.
\]
This proves \eqref{eq:app-corrected-comparison}.
\end{proof}

In the following endpoint estimates, \(\alpha=K\eta_\ell\). We apply
them to the convex, \(L_{\mathrm{loc}}\)-smooth functions
\(\widetilde f_i\) at \(w=x^{r-2}\), with starting and comparison points
fixed under \(\mathcal H_r\).

\subsection{Conditional-mean endpoint comparison}
The next lemma separates two sources of randomness in a client's local
updates: a random correction \(q\) that is held constant throughout the
trajectory, and fresh local gradient noise \(\zeta_k\) at each step.
In our application, \(h=\widetilde f_i\) and
\(q=q_i=\bar e-e_i\), where
\(e_i=c_i^{r-1}-\grad f_i(x^{r-2})\) is the client's control-estimation
error and \(\bar e=N^{-1}\sum_i e_i\). Thus \(q\) is the error in the
correction \(-c_i^{r-1}+c^{r-1}\), not the full correction itself.
All expectations in this application are conditional on \(\mathcal H_r\).

\begin{lemma}[Conditional-mean endpoint comparison]
\label{lem:app-mean-endpoint}
Let \(h\) be convex and \(L_{\mathrm{loc}}\)-smooth, and let \(q\) be a
random vector satisfying \(\E q=0\) and \(\E\norm q^2\le\nu^2\).
Starting from deterministic \(Y_0=z\), consider the updates
\[
 Y_{k+1}=Y_k-\eta_\ell(\grad h(Y_k)+q+\zeta_k),
 \qquad k=0,\ldots,K-1,
\]
using the same realization of \(q\) at every step.
Conditionally on \(q\) and the previous local samples, \(\zeta_k\) has
mean zero and second moment at most \(\sigma^2\). Write \(m_k=\E Y_k\),
\(v_k=\E\norm{Y_k-m_k}^2\), where expectation includes \(q\).
If \(\eta_\ell>0\) and \(\eta_\ell L_{\mathrm{loc}}\le1\), then
\begin{equation}
 v_k\le2k\eta_\ell^2\sigma^2+2k^2\eta_\ell^2\nu^2.
 \label{eq:app-minibatch-local-variance}
\end{equation}
If also \(\alpha L_{\mathrm{loc}}\le1\), then every deterministic \(u\) satisfies
\begin{equation}
\begin{split}
 2\alpha[h(m_K)-h(u)]\le{}&\norm{z-u}^2-\norm{m_K-u}^2\\
 &+\frac{2L_{\mathrm{loc}}\alpha^3\sigma^2}{K}
    +\frac43 L_{\mathrm{loc}}\alpha^3\nu^2.
\end{split}
\label{eq:app-minibatch-mean-endpoint}
\end{equation}
\end{lemma}
\begin{proof}
The gradient map \(T_q(x)=x-\eta_\ell\grad h(x)-\eta_\ell q\) is nonexpansive:
cocoercivity gives
\[
 \norm{T_q(x)-T_q(y)}^2
 \le\norm{x-y}^2-\eta_\ell(2/L_{\mathrm{loc}}-\eta_\ell)
                         \norm{\grad h(x)-\grad h(y)}^2.
\]
Since \(\eta_\ell L_{\mathrm{loc}}\le1\), the subtracted term is
nonnegative, giving \(\norm{T_q(x)-T_q(y)}\le\norm{x-y}\).

Define the comparison trajectory without local gradient noise by
\[
 Z_0(q)=z,\qquad Z_{k+1}(q)=T_q(Z_k(q)).
\]
This trajectory is deterministic once \(q\) is fixed, but can be random
before conditioning on \(q\). We also define \(Z_k(0)\) by the same
recursion with \(q=0\); this latter trajectory is deterministic.
Since \(Y_{k+1}=T_q(Y_k)-\eta_\ell\zeta_k\),
\[
 Y_{k+1}-Z_{k+1}(q)
 =T_q(Y_k)-T_q(Z_k(q))-\eta_\ell\zeta_k.
\]
Let \(\mathcal F_k\) be the sigma-field generated by \(q\) and the local
samples preceding step \(k\). Then
\(T_q(Y_k)-T_q(Z_k(q))\) is \(\mathcal F_k\)-measurable, so the conditional
mean-zero property of \(\zeta_k\) makes the cross term vanish:
\begin{align*}
 \E[\norm{Y_{k+1}-Z_{k+1}(q)}^2\mid\mathcal F_k]
 &=\norm{T_q(Y_k)-T_q(Z_k(q))}^2-2\eta_\ell
      \ip{T_q(Y_k)-T_q(Z_k(q))}{\E[\zeta_k\mid\mathcal F_k]}+\eta_\ell^2\E[\norm{\zeta_k}^2\mid\mathcal F_k]\\
 &\le\norm{Y_k-Z_k(q)}^2+\eta_\ell^2\sigma^2.
\end{align*}
Taking expectations conditional on \(q\) and iterating from
\(Y_0=Z_0(q)=z\) gives
\[
 \E[\norm{Y_k-Z_k(q)}^2\mid q]\le k\eta_\ell^2\sigma^2.
\]

Next, \(T_q(x)=T_0(x)-\eta_\ell q\) and nonexpansiveness of \(T_0\)
give
\begin{align*}
 \norm{Z_{k+1}(q)-Z_{k+1}(0)}
 &=\norm{T_0(Z_k(q))-T_0(Z_k(0))-\eta_\ell q}\\
 &\le\norm{Z_k(q)-Z_k(0)}+\eta_\ell\norm q.
\end{align*}
Since \(Z_0(q)=Z_0(0)=z\), induction yields
\[
 \norm{Z_k(q)-Z_k(0)}\le k\eta_\ell\norm q.
\]

Finally, for any deterministic point \(a\),
\[
 \E\norm{Y_k-a}^2=v_k+\norm{m_k-a}^2\ge v_k.
\]
Choose \(a=Z_k(0)\), insert \(Z_k(q)\), and use
\(\norm{a+b}^2\le2\norm a^2+2\norm b^2\). The two comparison bounds
above, averaged also over \(q\), imply
\begin{align*}
 v_k
 &\le\E\norm{Y_k-Z_k(0)}^2\\
 &\le2\E\norm{Y_k-Z_k(q)}^2
      +2\E\norm{Z_k(q)-Z_k(0)}^2\\
 &\le2k\eta_\ell^2\sigma^2+2k^2\eta_\ell^2\E\norm q^2\\
 &\le2k\eta_\ell^2\sigma^2+2k^2\eta_\ell^2\nu^2,
\end{align*}
proving \eqref{eq:app-minibatch-local-variance}.

Put \(g_k=\E\grad h(Y_k)\). Taking expectations in the update gives
\[
 m_{k+1}
 =m_k-\eta_\ell\bigl(g_k+\E q+\E\zeta_k\bigr)
 =m_k-\eta_\ell g_k.
\]
Here \(\E q=0\) by assumption, and \(\E\zeta_k=0\) follows by
taking expectations of its conditional mean.

Convexity at \(Y_k\) and smoothness from \(Y_k\) to \(m_{k+1}\)
give, respectively,
\begin{align*}
 h(Y_k)-h(u)
 &\le\ip{\grad h(Y_k)}{Y_k-u},\\
 h(m_{k+1})-h(Y_k)
 &\le\ip{\grad h(Y_k)}{m_{k+1}-Y_k}
      +\frac{L_{\mathrm{loc}}}{2}\norm{m_{k+1}-Y_k}^2.
\end{align*}
Adding cancels \(h(Y_k)\) and combines the inner products into
\(\ip{\grad h(Y_k)}{m_{k+1}-u}\). Moreover, since
\(\E(Y_k-m_k)=0\),
\begin{align*}
 \E\norm{m_{k+1}-Y_k}^2
 &=\E\norm{(m_k-Y_k)+(m_{k+1}-m_k)}^2\\
 &=v_k+\norm{m_{k+1}-m_k}^2
 =v_k+\eta_\ell^2\norm{g_k}^2.
\end{align*}
Taking expectations in the summed inequality therefore yields
\[
 h(m_{k+1})-h(u)
 \le\ip{g_k}{m_{k+1}-u}
       +\frac{L_{\mathrm{loc}}}{2}(v_k+\eta_\ell^2\norm{g_k}^2).
\]
The mean update also implies the distance identity
\[
 2\eta_\ell\ip{g_k}{m_{k+1}-u}
 =\norm{m_k-u}^2-\norm{m_{k+1}-u}^2
   -\eta_\ell^2\norm{g_k}^2.
\]
Multiplying the objective bound by \(2\eta_\ell\) and substituting
this identity gives
\begin{equation}
\begin{split}
 2\eta_\ell[h(m_{k+1})-h(u)]\le{}&\norm{m_k-u}^2-\norm{m_{k+1}-u}^2\\
 &-\eta_\ell^2(1-\eta_\ell L_{\mathrm{loc}})\norm{g_k}^2+\eta_\ell L_{\mathrm{loc}} v_k.
\end{split}
\label{eq:app-minibatch-mean-one-step}
\end{equation}
To control increases in the objective along the mean trajectory,
smoothness at \(m_k\) gives
\[
 h(m_{k+1})-h(m_k)
 \le-\eta_\ell\ip{\grad h(m_k)}{g_k}
      +\frac{L_{\mathrm{loc}}\eta_\ell^2}{2}\norm{g_k}^2.
\]
Using \(\eta_\ell L_{\mathrm{loc}}\le1\), the right-hand side is at most
\begin{align*}
 -\eta_\ell\ip{\grad h(m_k)}{g_k}
     +\frac{\eta_\ell}{2}\norm{g_k}^2
 &=\frac{\eta_\ell}{2}\norm{g_k-\grad h(m_k)}^2
   -\frac{\eta_\ell}{2}\norm{\grad h(m_k)}^2\\
 &\le\frac{\eta_\ell}{2}\norm{g_k-\grad h(m_k)}^2.
\end{align*}
Jensen's inequality and smoothness further give
\begin{align*}
 \norm{g_k-\grad h(m_k)}^2
 &=\norm{\E[\grad h(Y_k)-\grad h(m_k)]}^2\\
 &\le\E\norm{\grad h(Y_k)-\grad h(m_k)}^2\\
 &\le L_{\mathrm{loc}}^2\E\norm{Y_k-m_k}^2
 =L_{\mathrm{loc}}^2v_k.
\end{align*}
Consequently,
\begin{align}
 h(m_{k+1})-h(m_k)
 &\le\frac{\eta_\ell L_{\mathrm{loc}}^2}{2}v_k.
 \label{eq:app-minibatch-mean-drift}
\end{align}
Summing \eqref{eq:app-minibatch-mean-one-step} over
\(k=0,\ldots,K-1\), dropping the nonpositive gradient terms, and using
\(m_0=z\) gives
\[
 2\eta_\ell\sum_{j=1}^K[h(m_j)-h(u)]
 \le\norm{z-u}^2-\norm{m_K-u}^2
       +\eta_\ell L_{\mathrm{loc}}\sum_{k=0}^{K-1}v_k.
\]
For each \(j=1,\ldots,K\), summing
\eqref{eq:app-minibatch-mean-drift} from \(k=j\) to \(K-1\) yields
\[
 h(m_K)-h(m_j)
 \le\frac{\eta_\ell L_{\mathrm{loc}}^2}{2}
       \sum_{k=j}^{K-1}v_k,
\]
where the sum is empty when \(j=K\). Summing over \(j\) and observing
that each \(v_k\) occurs for exactly \(k\) indices \(j=1,\ldots,k\),
\begin{align*}
 K[h(m_K)-h(u)]
 &\le\sum_{j=1}^K[h(m_j)-h(u)]
      +\frac{\eta_\ell L_{\mathrm{loc}}^2}{2}
        \sum_{j=1}^K\sum_{k=j}^{K-1}v_k\\
 &=\sum_{j=1}^K[h(m_j)-h(u)]
      +\frac{\eta_\ell L_{\mathrm{loc}}^2}{2}
        \sum_{k=0}^{K-1}k v_k.
\end{align*}
Multiplying by \(2\eta_\ell\), applying the trajectory-sum bound,
and recalling \(\alpha=K\eta_\ell\), we obtain
\[
 2\alpha[h(m_K)-h(u)]
 \le\norm{z-u}^2-\norm{m_K-u}^2
       +\eta_\ell L_{\mathrm{loc}}\sum_{k=0}^{K-1}v_k
       +\eta_\ell^2L_{\mathrm{loc}}^2\sum_{k=0}^{K-1}k v_k.
\]
Since \(\eta_\ell L_{\mathrm{loc}}k\le\alpha L_{\mathrm{loc}}\le1\),
the last two terms satisfy
\[
 \eta_\ell L_{\mathrm{loc}}\sum_{k=0}^{K-1}
       (1+\eta_\ell L_{\mathrm{loc}}k)v_k
 \le2\eta_\ell L_{\mathrm{loc}}\sum_{k=0}^{K-1}v_k.
\]
Finally, \eqref{eq:app-minibatch-local-variance} and the formulas for
the sums of integers and their squares give
\begin{align*}
 \sum_{k=0}^{K-1}v_k
 &\le2\eta_\ell^2\sigma^2\frac{K(K-1)}2
      +2\eta_\ell^2\nu^2\frac{K(K-1)(2K-1)}6\\
 &\le\eta_\ell^2K^2\sigma^2+\frac23\eta_\ell^2K^3\nu^2.
\end{align*}
Thus the total error is at most
\[
 2\eta_\ell L_{\mathrm{loc}}\sum_{k=0}^{K-1}v_k
 \le\frac{2L_{\mathrm{loc}}\alpha^3\sigma^2}{K}
       +\frac43L_{\mathrm{loc}}\alpha^3\nu^2,
\]
which proves \eqref{eq:app-minibatch-mean-endpoint}.
\end{proof}

\subsection{Variance of the averaged endpoint}
For a random vector \(U\), write \(\operatorname{Var}(U)=\E\norm{U-\E U}^2\).
\begin{lemma}[Variance of the server endpoint]
\label{lem:app-minibatch-server-variance}
In round \(r\), write \(Y_{i,k}=y_{i,k}^r\) and
\(X=N^{-1}\sum_iY_{i,K}\). All moments in this lemma and its proof are
conditional on \(\mathcal H_r\). If \(\eta_\ell L_{\mathrm{loc}}\le1\), then
\begin{equation}
 \operatorname{Var}(X)\le\frac{\alpha^2\sigma^2}{NK}
       +\frac{8\alpha^2\sigma^2}{NB}
       +\frac{2L_{\mathrm{loc}}^2\alpha^4\sigma^2}{K}.
 \label{eq:app-server-endpoint-variance}
\end{equation}
\end{lemma}
\begin{proof}
Recall that \(e_i=c_i^{r-1}-\grad f_i(x^{r-2})\),
\(\bar e=N^{-1}\sum_i e_i\), and \(q_i=\bar e-e_i\).
All expectations below remain conditional on \(\mathcal H_r\).
Writing \(e=(e_1,\ldots,e_N)\), the total-variance identity gives
\[
 \operatorname{Var}(X)
 =\E[\operatorname{Var}(X\mid e)]
   +\operatorname{Var}(\E[X\mid e]).
\]
We first bound the local sampling variance with the controls fixed,
then bound the variation of the conditional mean as the controls vary.

For the corrected objective \(\widetilde f_i\) anchored at \(x^{r-2}\),
define
\[
 T_{i,q}(x)=x-\eta_\ell\bigl(\grad\widetilde f_i(x)+q\bigr).
\]
This map is nonexpansive by the argument in
Lemma~\ref{lem:app-mean-endpoint}. Conditional on \(e\), the local update is
\(Y_{i,k+1}=T_{i,q_i}(Y_{i,k})-\eta_\ell\zeta_{i,k}\), with centered
fresh noise of second moment at most \(\sigma^2\).
For any nonexpansive map \(T\) and random vector \(U\),
\[
 \operatorname{Var}(T(U))
 \le\E\norm{T(U)-T(\E U)}^2
 \le\E\norm{U-\E U}^2=\operatorname{Var}(U).
\]
The same argument applies conditionally on \(e\). Expanding the
conditional variance of the local update, its noise cross term vanishes
by conditioning first on the preceding local samples. Hence
\begin{align*}
 \operatorname{Var}(Y_{i,k+1}\mid e)
 &=\operatorname{Var}(T_{i,q_i}(Y_{i,k})\mid e)
   +\eta_\ell^2\E[\norm{\zeta_{i,k}}^2\mid e]\\
 &\le\operatorname{Var}(Y_{i,k}\mid e)+\eta_\ell^2\sigma^2.
\end{align*}
Since \(Y_{i,0}=x^{r-1}\) is fixed under \(\mathcal H_r\), iteration gives
\(\operatorname{Var}(Y_{i,k}\mid e)\le k\eta_\ell^2\sigma^2\).
Conditional on \(e\), clients use independent local sample streams, so
\[
 \operatorname{Var}(X\mid e)
 =\frac1{N^2}\sum_i\operatorname{Var}(Y_{i,K}\mid e)
 \le\frac{K\eta_\ell^2\sigma^2}{N}
 =\frac{\alpha^2\sigma^2}{NK}.
\]
Taking expectations over \(e\) bounds the first term in the
total-variance identity by the same quantity.

To bound the second term, define the trajectory without local noise by
\[
 Z_{i,0}(q)=x^{r-1},\qquad
 Z_{i,k+1}(q)=T_{i,q}(Z_{i,k}(q)),\qquad Z_i(q)=Z_{i,K}(q).
\]
For fixed \(q\), this trajectory is deterministic under the initial
conditioning. Write \(m_{i,k}(e)=\E[Y_{i,k}\mid e]\) and
\(b_i(e)=m_{i,K}(e)-Z_i(q_i)\). Conditional centering of the local noise
gives
\begin{align*}
 m_{i,k+1}(e)-T_{i,q_i}(m_{i,k}(e))
 =-\eta_\ell\bigl(
 \E[\grad\widetilde f_i(Y_{i,k})\mid e]
 -\grad\widetilde f_i(m_{i,k}(e))\bigr).
\end{align*}
Smoothness and Cauchy--Schwarz therefore imply
\begin{align*}
 \norm{m_{i,k+1}(e)-T_{i,q_i}(m_{i,k}(e))}
 &\le\eta_\ell L_{\mathrm{loc}}
       \E[\norm{Y_{i,k}-m_{i,k}(e)}\mid e]\\
 &\le\eta_\ell L_{\mathrm{loc}}
       \sqrt{\operatorname{Var}(Y_{i,k}\mid e)}.
\end{align*}
Using nonexpansiveness to compare with the deterministic trajectory gives
\begin{align*}
 \norm{m_{i,k+1}(e)-Z_{i,k+1}(q_i)}
 &\le\norm{m_{i,k}(e)-Z_{i,k}(q_i)}\\
 &\quad+\eta_\ell L_{\mathrm{loc}}
       \sqrt{\operatorname{Var}(Y_{i,k}\mid e)}.
\end{align*}
The two trajectories start at the same point. Iterating and applying the
conditional variance bound yields, almost surely in \(e\),
\begin{equation}
 \norm{b_i(e)}\le L_{\mathrm{loc}}\eta_\ell^2\sigma\sum_{k=0}^{K-1}\sqrt{k}
                \le\frac{L_{\mathrm{loc}}\alpha^2\sigma}{\sqrt K}.
 \label{eq:app-minibatch-conditional-bias}
\end{equation}
Here \(\sum_{k=0}^{K-1}\sqrt{k}\le K\sqrt K\) and
\(\alpha=K\eta_\ell\). This uses smoothness of the expected objective only.

Define \(F(e)=N^{-1}\sum_i Z_i(q_i)\). To control its variance, first
compare two fixed corrections \(q,q'\). Nonexpansiveness and
\(T_{i,q}(x)=T_{i,q'}(x)-\eta_\ell(q-q')\) imply
\[
 \norm{Z_{i,k+1}(q)-Z_{i,k+1}(q')}
 \le\norm{Z_{i,k}(q)-Z_{i,k}(q')}+\eta_\ell\norm{q-q'}.
\]
Starting from the common initial point and iterating gives
\(\norm{Z_i(q)-Z_i(q')}\le K\eta_\ell\norm{q-q'}
=\alpha\norm{q-q'}\).

Let \(e_j'\) be an independent copy of \(e_j\), conditional on
\(\mathcal H_r\), independent also of the other control errors. Denote by
\(e^{(j)}\) the vector obtained by replacing only \(e_j\) by \(e_j'\).
Since \(q_i=\bar e-e_i\),
\[
 q_i(e^{(j)})-q_i(e)
 =\begin{cases}
  -(1-1/N)(e_j'-e_j),&i=j,\\
  (e_j'-e_j)/N,&i\ne j.
 \end{cases}
\]
Applying the endpoint sensitivity bound to each client gives
\begin{align*}
 \norm{F(e)-F(e^{(j)})}
 &\le\frac\alpha N\sum_i\norm{q_i(e)-q_i(e^{(j)})}\\
 &=\frac\alpha N\left(1-\frac1N+\frac{N-1}{N}\right)
       \norm{e_j-e_j'}\\
 &=\frac{2\alpha(N-1)}{N^2}\norm{e_j-e_j'}
 \le\frac{2\alpha}{N}\norm{e_j-e_j'}.
\end{align*}
The centered, independent copies satisfy
\[
 \E\norm{e_j-e_j'}^2
 =2\E\norm{e_j}^2\le\frac{2\sigma^2}{B}.
\]
The Efron--Stein inequality \citep{efron1981}, applied conditionally to the
independent control errors, now yields
\begin{equation}
 \begin{aligned}
 \operatorname{Var}(F)
 &\le\frac12\sum_{j=1}^N\E\norm{F(e)-F(e^{(j)})}^2\\
 &\le\frac12\sum_{j=1}^N\frac{4\alpha^2}{N^2}
                \frac{2\sigma^2}{B}
 =\frac{4\alpha^2\sigma^2}{NB}.
 \end{aligned}
 \label{eq:app-minibatch-replacement-variance}
\end{equation}
One justification of the first inequality is the martingale
decomposition obtained by revealing the independent \(e_j\)'s one at a
time. Its squared increments sum to \(\operatorname{Var}(F)\), and conditional Jensen
bounds each increment by the conditional variance in its coordinate;
the latter is one half the mean squared independent-copy difference.
This argument applies to Euclidean vector-valued functions as well.

Write \(\bar b(e)=N^{-1}\sum_i b_i(e)\). Jensen's inequality and
\eqref{eq:app-minibatch-conditional-bias} give
\[
 \operatorname{Var}(\bar b)
 \le\E\norm{\bar b}^2
 \le\frac1N\sum_i\E\norm{b_i(e)}^2
 \le\frac{L_{\mathrm{loc}}^2\alpha^4\sigma^2}{K}.
\]
By definition, \(\E[X\mid e]=F(e)+\bar b(e)\). Applying the squared
triangle inequality to the centered vectors \(F-\E F\) and
\(\bar b-\E\bar b\), without assuming independence, yields
\[
 \operatorname{Var}(\E[X\mid e])=\operatorname{Var}(F+\bar b)
 \le2\operatorname{Var}(F)+2\operatorname{Var}(\bar b).
\]
Substituting the two bounds into the total-variance identity proves
\begin{align*}
 \operatorname{Var}(X)
 &\le\frac{\alpha^2\sigma^2}{NK}
       +2\frac{4\alpha^2\sigma^2}{NB}
       +2\frac{L_{\mathrm{loc}}^2\alpha^4\sigma^2}{K}\\
 &=\frac{\alpha^2\sigma^2}{NK}
       +\frac{8\alpha^2\sigma^2}{NB}
       +\frac{2L_{\mathrm{loc}}^2\alpha^4\sigma^2}{K},
\end{align*}
as claimed in \eqref{eq:app-server-endpoint-variance}.
\end{proof}

\subsection{Combining the endpoint bounds}
In one round, condition on \(\mathcal H_r\), put \(z=x^{r-1}\) and
\(w=x^{r-2}\), and use \(\widetilde f_i\) from
\eqref{eq:app-generic-corrected-objective}. The local update direction is
\(\grad\widetilde f_i(y)+q_i+\zeta_{i,k}\), where \(\zeta_{i,k}\) is the
local gradient noise. Write \(Y_{i,K}=y_{i,K}^r\) and \(X=x^r\).
All expectations below are conditional on \(\mathcal H_r\), so comparison
points \(u=x^\star\) and \(u=z\) are both allowed.

\paragraph{Case \(\alpha L_{\mathrm{loc}}\le1\).}
Let \(m_i=\E Y_{i,K}\) and \(\bar m=N^{-1}\sum_i m_i=\E X\).
Apply Lemma~\ref{lem:app-mean-endpoint} to each \(\widetilde f_i\), with
\(\nu^2\le\sigma^2/B\), then average to obtain
\begin{align*}
 \frac{2\alpha}{N}\sum_i
    [\widetilde f_i(m_i)-\widetilde f_i(u)]
 \le{}&\norm{z-u}^2-\frac1N\sum_i\norm{m_i-u}^2\\
 &+\frac{2L_{\mathrm{loc}}\alpha^3\sigma^2}{K}
   +\frac{4L_{\mathrm{loc}}\alpha^3\sigma^2}{3B}.
\end{align*}
Define the dispersion of the mean endpoints by
\(d_m^2=N^{-1}\sum_i\norm{m_i-\bar m}^2\). The distance identity is
\[
 \frac1N\sum_i\norm{m_i-u}^2
 =\norm{\bar m-u}^2+d_m^2.
\]
Meanwhile, \eqref{eq:app-corrected-comparison} gives
\[
 f(\bar m)-f(u)
 \le\frac1N\sum_i[\widetilde f_i(m_i)-\widetilde f_i(u)]
      +\delta\norm{\bar m-w}d_m.
\]
Combining these statements leaves the dispersion and similarity terms
\begin{align*}
 -d_m^2+2\alpha\delta\norm{\bar m-w}d_m
 &=-\bigl(d_m-\alpha\delta\norm{\bar m-w}\bigr)^2
    +\alpha^2\delta^2\norm{\bar m-w}^2\\
 &\le\alpha^2\delta^2\norm{\bar m-w}^2.
\end{align*}
Thus
\begin{align*}
 2\alpha[f(\bar m)-f(u)]\le{}&\norm{z-u}^2-\norm{\bar m-u}^2
 +\alpha^2\delta^2\norm{\bar m-w}^2\\
 &+\frac{2L_{\mathrm{loc}}\alpha^3\sigma^2}{K}+\frac{4L_{\mathrm{loc}}\alpha^3\sigma^2}{3B}.
\end{align*}
To replace the mean endpoint by the random server endpoint, smoothness
of \(f\) and \(\E(X-\bar m)=0\) give
\begin{align*}
 \E f(X)
 &\le f(\bar m)+\E\ip{\grad f(\bar m)}{X-\bar m}
       +\frac{L_{\mathrm{loc}}}{2}\E\norm{X-\bar m}^2\\
 &=f(\bar m)+\frac{L_{\mathrm{loc}}}{2}\operatorname{Var}(X).
\end{align*}
The distance terms satisfy
\begin{align*}
 -\norm{\bar m-u}^2
 &=-\E\norm{X-u}^2+\operatorname{Var}(X),\\
 \norm{\bar m-w}^2
 &\le\E\norm{X-w}^2.
\end{align*}
Consequently,
\begin{align*}
 2\alpha\E[f(X)-f(u)]
 \le{}&\norm{z-u}^2-\E\norm{X-u}^2
       +\alpha^2\delta^2\E\norm{X-w}^2\\
 &+\frac{2L_{\mathrm{loc}}\alpha^3\sigma^2}{K}
   +\frac{4L_{\mathrm{loc}}\alpha^3\sigma^2}{3B}\\
 &+(1+\alpha L_{\mathrm{loc}})\operatorname{Var}(X).
\end{align*}
Here the objective comparison contributes
\(\alpha L_{\mathrm{loc}}\operatorname{Var}(X)\), and the distance
identity contributes \(\operatorname{Var}(X)\). By
Lemma~\ref{lem:app-minibatch-server-variance} and
\(1+\alpha L_{\mathrm{loc}}\le2\), their sum is bounded by
\[
 (1+\alpha L_{\mathrm{loc}})\operatorname{Var}(X)
 \le\frac{2\alpha^2\sigma^2}{NK}
     +\frac{16\alpha^2\sigma^2}{NB}
     +\frac{4L_{\mathrm{loc}}^2\alpha^4\sigma^2}{K}.
\]
Since \(\alpha L_{\mathrm{loc}}\le1\), the last term is at most
\(4L_{\mathrm{loc}}\alpha^3\sigma^2/K\); adding it to the existing
\(2L_{\mathrm{loc}}\alpha^3\sigma^2/K\) term gives the coefficient 6
below. Enlarging the similarity coefficient from 1 to 2 gives
\begin{equation}
 2\alpha\E[f(X)-f(u)]\le\norm{z-u}^2-\E\norm{X-u}^2
                 +2\alpha^2\delta^2\E\norm{X-w}^2+Q,
 \label{eq:app-minibatch-one-round}
\end{equation}
where the factor 2 in the similarity term is chosen to agree with the
next case, and the error term satisfies
\begin{align}
 Q\le{}&\frac{2\alpha^2\sigma^2}{NK}
       +\frac{16\alpha^2\sigma^2}{NB}
       +\frac{6L_{\mathrm{loc}}\alpha^3\sigma^2}{K}
       +\frac{4L_{\mathrm{loc}}\alpha^3\sigma^2}{3B}\notag\\
 \le{}&\frac{16\alpha^2\sigma^2}{N}\left(\frac1K+\frac1B\right)
       +6L_{\mathrm{loc}}\alpha^3\sigma^2\left(\frac1K+\frac1B\right).
 \label{eq:app-minibatch-small-budget}
\end{align}
The second inequality uses \(2\le16\) and \(4/3\le6\) to group the
terms with denominators \(K\) and \(B\).

\paragraph{Case \(\alpha L_{\mathrm{loc}}\ge1\).}
Condition additionally on all control errors \(e=(e_1,\ldots,e_N)\).
The functions \(\widetilde f_i+\ip{q_i}{\cdot}\) are then fixed,
convex, and \(L_{\mathrm{loc}}\)-smooth. Apply
Lemma~\ref{lem:app-refined-endpoint} to each client and average:
\begin{align*}
 \frac{2\alpha}{N}\sum_i\E\bigl[
    \widetilde f_i(Y_{i,K})-\widetilde f_i(u)
    +\ip{q_i}{Y_{i,K}-u}\mid e\bigr]
 \le{}&\norm{z-u}^2
       -\frac1N\sum_i\E[\norm{Y_{i,K}-u}^2\mid e]\\
 &+\frac{2\alpha^2\sigma^2}{K}
       [1+\log(1+\alpha L_{\mathrm{loc}})].
\end{align*}
Here we used \(K\eta_\ell=\alpha\) and
\(K\eta_\ell^2=\alpha^2/K\). Define
\[
 d^2=\frac1N\sum_i\norm{Y_{i,K}-X}^2,\qquad
 v^2=\frac1N\sum_i\norm{q_i}^2.
\]
Since \(\sum_iq_i=0\), the linear corrections satisfy, for each
realization of the endpoints,
\begin{align*}
 \frac1N\sum_i\ip{q_i}{Y_{i,K}-u}
 &=\frac1N\sum_i\ip{q_i}{Y_{i,K}-X},\\
 \left|\frac1N\sum_i\ip{q_i}{Y_{i,K}-X}\right|
 &\le\left(\frac1N\sum_i\norm{q_i}^2\right)^{1/2}
       \left(\frac1N\sum_i\norm{Y_{i,K}-X}^2\right)^{1/2}
 =vd.
\end{align*}
The corrected-objective comparison and the endpoint distance identity give
\begin{align*}
 \frac1N\sum_i[\widetilde f_i(Y_{i,K})-\widetilde f_i(u)]
 &\ge f(X)-f(u)-\delta\norm{X-w}d,\\
 \frac1N\sum_i\norm{Y_{i,K}-u}^2
 &=\norm{X-u}^2+d^2.
\end{align*}
Substitute these bounds into the averaged endpoint inequality and take
expectations over the control errors. This yields
\begin{align*}
 2\alpha\E[f(X)-f(u)]
 \le{}&\norm{z-u}^2-\E\norm{X-u}^2+\E[-d^2+2\alpha\delta\norm{X-w}d+2\alpha vd]\\
 &+\frac{2\alpha^2\sigma^2}{K}
       [1+\log(1+\alpha L_{\mathrm{loc}})].
\end{align*}
Complete the square and then use \((a+b)^2\le2a^2+2b^2\):
\begin{equation}
\begin{aligned}
 -d^2+2\alpha\delta\norm{X-w}d+2\alpha vd
 &=-\bigl[d-\alpha(\delta\norm{X-w}+v)\bigr]^2
    +\alpha^2(\delta\norm{X-w}+v)^2\\
 &\le\alpha^2(\delta\norm{X-w}+v)^2\\
 &\le2\alpha^2\delta^2\norm{X-w}^2+2\alpha^2v^2.
\end{aligned}
 \label{eq:app-minibatch-square}
\end{equation}
By \eqref{eq:app-control-moments},
\[
 \E v^2=\frac1N\sum_i\E\norm{q_i}^2
 \le\left(1-\frac1N\right)\frac{\sigma^2}{B}
 \le\frac{\sigma^2}{B}.
\]
Together these estimates prove
\eqref{eq:app-minibatch-one-round} with
\begin{equation}
 Q\le\frac{2\alpha^2\sigma^2}{K}[1+\log(1+\alpha L_{\mathrm{loc}})]
                         +\frac{2\alpha^2\sigma^2}{B}.
 \label{eq:app-minibatch-large-budget}
\end{equation}
To combine the cases, put \(t=\alpha L_{\mathrm{loc}}\).
For \(0\le t\le1\), \(t\le2\log(1+t)\), so dividing
\eqref{eq:app-minibatch-small-budget} by \(\alpha\) gives
\begin{align*}
 \frac Q\alpha
 &\le\frac{16\alpha\sigma^2}{N}\left(\frac1K+\frac1B\right)
      +6\alpha\sigma^2\left(\frac1K+\frac1B\right)t\\
 &\le\frac{16\alpha\sigma^2}{N}\left(\frac1K+\frac1B\right)
      +12\alpha\sigma^2\left(\frac1K+\frac1B\right)\log(1+t).
\end{align*}
For \(t\ge1\), \(2[1+\log(1+t)]\le5\log(1+t)\) and
\(2\le5\log(1+t)\). Thus \eqref{eq:app-minibatch-large-budget} gives
\begin{align*}
 \frac Q\alpha
 &\le\frac{2\alpha\sigma^2}{K}[1+\log(1+t)]
      +\frac{2\alpha\sigma^2}{B}\\
 &\le5\alpha\sigma^2\left(\frac1K+\frac1B\right)\log(1+t).
\end{align*}
Enlarging 5 to 12 and adding the nonnegative term with coefficient 16
therefore gives the same bound. In both cases,
\eqref{eq:app-minibatch-one-round} holds with
\begin{equation}
 \frac Q\alpha\le
 \frac{16\alpha\sigma^2}{N}\left(\frac1K+\frac1B\right)
 +12\alpha\sigma^2\left(\frac1K+\frac1B\right)\log(1+\alpha L_{\mathrm{loc}}).
 \label{eq:app-minibatch-uniform-budget}
\end{equation}

Since \(\theta=2\alpha^2\delta^2\), the preceding bounds establish
\eqref{eq:app-joint-endpoint-template}. Lemma~\ref{lem:app-server-reduction}
and \eqref{eq:app-minibatch-uniform-budget} therefore give, for every \(B\ge1\),
\begin{equation}
 \E[f(\bar x^R)]-f^\star\le40\left[\frac{D_0^2}{\alpha R}
 +\frac{\alpha\sigma^2}{N}\left(\frac1K+\frac1B\right)+\alpha\sigma^2\left(\frac1K+\frac1B\right)
                         \log(1+\alpha L_{\mathrm{loc}})\right].
\label{eq:app-explicit-convex-rate}
\end{equation}
Indeed, the three coefficients before rounding are
\(113/35\), \(1296/35\), and \(972/35\), all at most 40. This is
Proposition~\ref{prop:fixed-stepsize}.
This bound also applies when \(\sigma=0\).

The distinction between the two cases is needed for the sharp rate.
Using \eqref{eq:app-minibatch-square} alone incurs a term of order
\(\alpha\sigma^2/B\) after server reduction. For
\(\alpha L_{\mathrm{loc}}\le1\),
Lemma~\ref{lem:app-mean-endpoint} instead uses centering of the controls
and gives a contribution of order
\(L_{\mathrm{loc}}\alpha^2\sigma^2/B\), together with a variance term
of order \(\alpha\sigma^2/(NB)\).
For \(\alpha L_{\mathrm{loc}}\ge1\), the direct term is already bounded by
\(\alpha\sigma^2\log(1+\alpha L_{\mathrm{loc}})/B\).

\subsection{A small-stepsize upper bound}

\begin{lemma}[Uniform small-stepsize baseline]
\label{lem:combined-initial-cap}
Suppose Assumptions~\ref{ass:convex-smooth} and~\ref{ass:stochastic-oracle}
hold, and \(D_0>0\). For any \(N,K,R,B\ge1\), Algorithm~\ref{alg:scaffold-option-i}
with
\begin{equation}
    \eta_\ell=\frac{D_0}{16KR(L_{\mathrm{loc}}D_0+\sigma)}
    \label{eq:app-small-stepsize}
\end{equation}
satisfies
\begin{equation}
    \Delta_R\le L_{\mathrm{loc}}D_0^2.
    \label{eq:combined-initial-cap}
\end{equation}
\end{lemma}

\begin{proof}
We show that the prescribed stepsize keeps the root-mean-square distance
of every iterate to \(x^\star\) bounded by a constant multiple of \(D_0\).
Smoothness then converts this distance bound into the claimed objective bound.
At any local update the controls are evaluated at a previous server iterate
\(a=x^{r-2}\), with \(a=x^0\) in the first round.  Smoothness and
\(\nabla f(x^\star)=0\) give the pathwise bound
\[
\begin{aligned}
    \norm{\nabla f_i(y)-\nabla f_i(a)+\nabla f(a)}
       &\le L_{\mathrm{loc}}\norm{y-a}
                +L_{\mathrm{loc}}\norm{a-x^\star}\\
       &\le L_{\mathrm{loc}}\norm{y-x^\star}
                +2L_{\mathrm{loc}}\norm{a-x^\star}.
\end{aligned}
\]
In particular, possibly large individual gradients at \(x^\star\) do
not enter this bound; their constant offsets have canceled. The actual
direction additionally contains the local noise and the control error
\(q_i\) from \eqref{eq:app-control-errors}. By
\eqref{eq:app-control-moments}, \(\E\norm{q_i}^2\le\sigma^2/B\),
including in the first round.

Index the synchronous local-update stages by \(t=0,\ldots,KR\):
before local step \(k\) of round \(r\), the stage is
\(t=(r-1)K+k\). Each stage advances all clients by one local step.
We construct a deterministic bound \(m_t\) on the root-mean-square
distance to \(x^\star\) of every local or server point generated through
stage \(t\), including the server averaging and local restart at a round
boundary. Initially all clients start at \(x^0\), so \(m_0=D_0\).

Suppose the bound holds through stage \(t\), and write
\(y=y_{i,k}^r\) and \(y^+=y_{i,k+1}^r\). The anchor \(a=x^{r-2}\)
is an earlier server point (or \(x^0\) in the first round), so both
\(y\) and \(a\) have root-mean-square distance at most \(m_t\).
The update is
\[
 y^+=y-\eta_\ell\bigl(
     \grad f_i(y)-\grad f_i(a)+\grad f(a)+q_i+\zeta_{i,k}\bigr).
\]
Applying the triangle inequality in \(L^2\), the pathwise gradient bound
above, and the second-moment bounds on the two noise terms gives
\begin{align*}
 \bigl(\E\norm{y^+-x^\star}^2\bigr)^{1/2}
 \le{}&\bigl(\E\norm{y-x^\star}^2\bigr)^{1/2}\\
 &+\eta_\ell L_{\mathrm{loc}}
       \bigl(\E\norm{y-x^\star}^2\bigr)^{1/2}\\
 &+2\eta_\ell L_{\mathrm{loc}}
       \bigl(\E\norm{a-x^\star}^2\bigr)^{1/2}\\
 &+\eta_\ell\bigl(\E\norm{q_i}^2\bigr)^{1/2}
       +\eta_\ell\bigl(\E\norm{\zeta_{i,k}}^2\bigr)^{1/2}\\
 \le{}&(1+3\eta_\ell L_{\mathrm{loc}})m_t
       +\eta_\ell\sigma(1+B^{-1/2}).
\end{align*}
This estimate does not require independence between its terms. Define
\[
 m_{t+1}=(1+3\eta_\ell L_{\mathrm{loc}})m_t
          +\eta_\ell\sigma(1+B^{-1/2}).
\]
This sequence is nondecreasing, so it also bounds all previously
generated points. At the end of round \(r\), the server average satisfies
\begin{align*}
 \bigl(\E\norm{x^r-x^\star}^2\bigr)^{1/2}
 &=\left(\E\norm{\frac1N\sum_i(y_{i,K}^r-x^\star)}^2\right)^{1/2}\\
 &\le\frac1N\sum_i\bigl(\E\norm{y_{i,K}^r-x^\star}^2\bigr)^{1/2}
 \le m_{rK}.
\end{align*}
The next local trajectories restart from \(x^r\), preserving this bound
without adding a local-update stage. This completes the induction.
Solving the scalar recursion yields
\[
    m_{KR}=
       (1+3\eta_\ell L_{\mathrm{loc}})^{KR}D_0
       +\frac{\sigma(1+B^{-1/2})}{3L_{\mathrm{loc}}}
        \bigl((1+3\eta_\ell L_{\mathrm{loc}})^{KR}-1\bigr).
\]
For the stepsize in \eqref{eq:app-small-stepsize},
\(3\eta_\ell L_{\mathrm{loc}}KR\le3/16\).
Using \((1+u)^{KR}\le e^{uKR}\) and \(e^v-1\le2v\) for
\(0\le v\le3/16\), we obtain
\begin{align*}
 m_{KR}-D_0
 &\le 2KR\eta_\ell
       \bigl(3L_{\mathrm{loc}}D_0+\sigma(1+B^{-1/2})\bigr)\\
 &\le 6KR\eta_\ell(L_{\mathrm{loc}}D_0+\sigma)
 =\frac38D_0.
\end{align*}
The second inequality uses \(B\ge1\), so \(1+B^{-1/2}\le2\) and
\(3L_{\mathrm{loc}}D_0+2\sigma\le3(L_{\mathrm{loc}}D_0+\sigma)\).
Applying the same averaging inequality to the returned iterate gives
\[
 \bigl(\E\norm{\bar x^R-x^\star}^2\bigr)^{1/2}
 \le\frac1R\sum_{r=1}^R
       \bigl(\E\norm{x^r-x^\star}^2\bigr)^{1/2}
 \le m_{KR}\le\frac{11}{8}D_0.
\]
Since \(f\) is \(L_{\mathrm{loc}}\)-smooth and
\(\grad f(x^\star)=0\), for every \(x\),
\[
 f(x)-f(x^\star)
 \le\ip{\grad f(x^\star)}{x-x^\star}
      +\frac{L_{\mathrm{loc}}}{2}\norm{x-x^\star}^2
 =\frac{L_{\mathrm{loc}}}{2}\norm{x-x^\star}^2.
\]
Taking expectations with \(x=\bar x^R\) therefore gives
\[
 \Delta_R\le\frac{L_{\mathrm{loc}}}{2}m_{KR}^2
 \le\frac{121}{128}L_{\mathrm{loc}}D_0^2
 \le L_{\mathrm{loc}}D_0^2.
\]
\end{proof}

\subsection{Proof of the tuned rate}
\label{app:convex-tuning}

Set \(B=K\). We first prove \(\Delta_R\le200\mathcal T\) for the choice
\begin{equation}
\begin{aligned}
    \alpha^{-1}
    ={}&
    \max\left\{\frac{L_{\mathrm{loc}}}{K},3\sqrt2\delta\right\}
    +\frac{\sigma\sqrt R}{D_0\sqrt{NK}}\\
    &+\min\left\{
        \left(
            \frac{L_{\mathrm{loc}}\sigma^2R}{D_0^2K}
        \right)^{1/3},
        \frac{\sigma\sqrt{R\kappa}}{D_0\sqrt K}
    \right\},
    \quad
    \eta_\ell=\frac{\alpha}{K}.
\end{aligned}
\label{eq:tuned-stepsize}
\end{equation}
This choice combines the stability restriction with the balances of
the initial-error and noise terms. Let
\begin{equation}
\begin{aligned}
    a&\coloneqq
    \max\left\{\frac{L_{\mathrm{loc}}}{K},3\sqrt2\delta\right\},
    &
    b&\coloneqq
    \frac{\sigma\sqrt R}{D_0\sqrt{NK}},\\
    c&\coloneqq
    \min\left\{
        \left(
            \frac{L_{\mathrm{loc}}\sigma^2R}{D_0^2K}
        \right)^{1/3},
        \frac{\sigma\sqrt{R\kappa}}{D_0\sqrt K}
    \right\}.
\end{aligned}
\label{eq:app-tuning-components}
\end{equation}
Then \(\alpha^{-1}=a+b+c\).  Since \(\alpha\le1/a\),
\[
    \eta_\ell L_{\mathrm{loc}}\le 1,
    \qquad
    \alpha\delta\le\frac1{3\sqrt2},
\]
so the conditions of Proposition~\ref{prop:fixed-stepsize} hold.
Furthermore,
\begin{align}
    \frac{D_0^2a}{R}
    &\le
    3\sqrt2\max\left\{\frac{L_{\mathrm{loc}}}{K},\delta\right\}
        \frac{D_0^2}{R},
    \label{eq:app-tuning-a}\\
    \frac{D_0^2b}{R}
    &=\frac{\sigma D_0}{\sqrt{NKR}},
    \label{eq:app-tuning-b}\\
    \frac{D_0^2c}{R}
    &=
    \min\left\{
        \left(
            \frac{L_{\mathrm{loc}}\sigma^2D_0^4}{KR^2}
        \right)^{1/3},
        \frac{\sigma D_0\sqrt\kappa}{\sqrt{KR}}
    \right\}.
    \label{eq:app-tuning-c}
\end{align}

The client-averaging noise term is controlled using \(\alpha\le1/b\):
\begin{equation}
    \frac{\eta_\ell\sigma^2}{N}
    =\frac{\alpha\sigma^2}{KN}
    \le\frac{\sigma^2}{KNb}
    =\frac{\sigma D_0}{\sqrt{NKR}}.
    \label{eq:app-tuning-client-noise}
\end{equation}
It remains to control the logarithmic noise term.  If the first branch in
the definition of \(c\) is active, then \(\log(1+x)\le x\) and
\begin{equation}
    \frac{\alpha\sigma^2}{K}
        \log(1+\alpha L_{\mathrm{loc}})
    \le
    \frac{L_{\mathrm{loc}}\sigma^2}{Kc^2}
    =\frac{D_0^2c}{R}.
    \label{eq:app-tuning-cubic-branch}
\end{equation}
If the second branch is active, then \(\kappa\ge1\), \(\alpha\le1/c\), and
\[
    \log(1+\alpha L_{\mathrm{loc}})
    \le
    \log\left(
        e+\frac{L_{\mathrm{loc}}D_0\sqrt K}{\sigma\sqrt R}
    \right)
    =\kappa.
\]
Therefore,
\begin{equation}
    \frac{\alpha\sigma^2}{K}
        \log(1+\alpha L_{\mathrm{loc}})
    \le
    \frac{\sigma^2\kappa}{Kc}
    =\frac{D_0^2c}{R}.
    \label{eq:app-tuning-log-branch}
\end{equation}

Substituting \eqref{eq:app-tuning-a}--\eqref{eq:app-tuning-log-branch}
into \eqref{eq:app-explicit-convex-rate} gives
\begin{align*}
    \E[f(\bar x^R)]-f^\star
    \le{}&
    120\sqrt2
    \max\left\{\frac{L_{\mathrm{loc}}}{K},\delta\right\}
        \frac{D_0^2}{R}
    +120\frac{\sigma D_0}{\sqrt{NKR}}\\
    &+120
    \min\left\{
        \left(
            \frac{L_{\mathrm{loc}}\sigma^2D_0^4}{KR^2}
        \right)^{1/3},
        \frac{\sigma D_0\sqrt\kappa}{\sqrt{KR}}
    \right\}.
\end{align*}
Each numerical coefficient is at most \(200\), so the last display is
at most \(200\mathcal T\), with \(\mathcal T\) defined in
\eqref{eq:tuned-rate}.

For Theorem~\ref{thm:scaffold-convex}, use the following complete tuning:
\begin{equation}
 \eta_\ell=
 \begin{cases}
  \text{the stepsize in \eqref{eq:tuned-stepsize}},
      &\mathcal T\le L_{\mathrm{loc}}D_0^2,\\[2mm]
  \displaystyle\frac{D_0}{16KR(L_{\mathrm{loc}}D_0+\sigma)},
      &\mathcal T>L_{\mathrm{loc}}D_0^2.
 \end{cases}
 \label{eq:app-complete-tuning}
\end{equation}
In the first case, the preceding calculation gives
\(\Delta_R\le200\mathcal T\). In the second,
Lemma~\ref{lem:combined-initial-cap} gives
\(\Delta_R\le L_{\mathrm{loc}}D_0^2\).
Thus \(\Delta_R\le200\min\{L_{\mathrm{loc}}D_0^2,\mathcal T\}\),
proving Theorem~\ref{thm:scaffold-convex}.

\subsection{Other control batch sizes}
\label{app:convex-batch-size}
For \(B\ne K\), define, only in this subsection,
\[
 \sigma_B^2=\sigma^2\left(1+\frac KB\right),\qquad
 \kappa_B=\log\left(e+\frac{L_{\mathrm{loc}}D_0\sqrt K}{\sigma_B\sqrt R}\right).
\]
Replace \(\sigma\) and \(\kappa\) by \(\sigma_B\) and \(\kappa_B\) in
\eqref{eq:tuned-stepsize}. Then \eqref{eq:app-explicit-convex-rate} and
the same balancing calculation give
\begin{equation}
\begin{aligned}
 \E[f(\bar x^R)]-f^\star\le200\Bigg[
 &\max\left\{\frac{L_{\mathrm{loc}}}{K},\delta\right\}\frac{D_0^2}{R}
 +\frac{\sigma_BD_0}{\sqrt{NKR}}\\
 &+\min\left\{\left(\frac{L_{\mathrm{loc}}\sigma_B^2D_0^4}{KR^2}\right)^{1/3},
       \frac{\sigma_BD_0\sqrt{\kappa_B}}{\sqrt{KR}}\right\}\Bigg].
\end{aligned}
\label{eq:app-general-batch-rate}
\end{equation}
Thus any \(B\ge cK\), for fixed \(c>0\), yields the rate in
\eqref{eq:tuned-rate} up to constants depending only on \(c\).
The main theorem takes \(B=K\) and uses the original \(\sigma\) and \(\kappa\)
in the two-case tuning \eqref{eq:app-complete-tuning}.

\subsection{Details of the convex comparison table}
\label{app:convex-comparison}
Table~\ref{tab:convex-stochastic-comparison} uses the following
specializations of the cited guarantees.
Exact-gradient evaluations are counted separately from stochastic
samples, with no assumption on their relative cost.
\paragraph{SCAFFOLD Option II.}
The general-convex Option-II analysis of
\citet[Appendix~E.1, Lemma~14]{karimireddy2020scaffold}, with full participation,
\(\eta_g=\sqrt N\), and a tuned local stepsize, gives
\(\mathcal O(LD_0^2/R+\sigma D_0/\sqrt{NKR})\).
The row takes \(K=\Theta(1+\sigma^2/(NL\varepsilon))\).
Controls are initialized using \(K\) stochastic gradients per client,
as in their Remark~10; this cost is included in the gradient count.
\paragraph{FedRed-SGD.}
Communication and exact-gradient counts are in expectation.
Exact controls are cached between communications. The row specializes
Theorem~9 of \citet{jiang2024fedred} to \(\mu=0\).
\paragraph{S-DANE with SGD.}
The row specializes Theorem~20 and Corollary~21 of
\citet{jiang2024sdane} to full
participation, \(\mu=0\), \(\lambda=2\delta\), and single-sample local oracle
calls. The displayed computation bound is for this SGD local solver.
Exact gradients are used for the corrections and server update.

S-DANE's stochastic-gradient bound hides a factor
\(\log(e+L/\delta)\); our bound hides this factor only in the noise term,
as shown in \eqref{eq:convex-stochastic-computation}.

\clearpage

\section{Proof of Theorem~\ref{thm:constant-step-optimality}}
\label{app:lower-bound}

We prove the lower bound for the same initialized algorithm and averaged
server output as in Theorem~\ref{thm:scaffold-convex}. Throughout, \(L_{\mathrm{loc}},
\delta,\sigma\) are parameter upper bounds.  For each prescribed constant
stepsize we may choose a different instance, fixed before any samples are
drawn.  Identical-client examples are admissible for every similarity
upper bound, and deterministic examples for every variance upper bound.

The proof uses three families of hard instances.  Quadratic objectives give
the computation and similarity terms; a noisy flat quadratic gives
the statistical term; asymmetric piecewise quadratics give the nonlinear
stochastic term.  We prove each for the objective at the averaged server
iterate, then combine their worst-case guarantees with the upper bound
of Theorem~\ref{thm:scaffold-convex}. Every construction uses the same fixed sampled
losses for local and control queries, including initialization. The
constructions below apply for every finite control batch size \(B\ge1\);
we take \(B=K\) when combining them with the upper bound.
For the deterministic examples, all sampled losses coincide with their
expectations, so the controls are exact. For the noisy quadratic example,
the control errors cancel after server averaging. For the nonlinear
example, concavity preserves the mean bias despite the control errors.

For identical local objectives, the guarantee of
Theorem~\ref{thm:constant-step-optimality} with \(\delta=0\) holds
for every \(N\ge1\) and \(K\ge2\).
The hard instances satisfy the stronger condition that the gradient
noise has norm at most \(\sigma\) for every sample and every evaluation
point. For \(K=1\), the control corrections cancel in the server
update, giving minibatch SGD, so the additional local-update bias
term is not asserted.

\subsection{The deterministic computation lower bound}
\label{sec:computation}

We first establish a stepsize-independent computation lower bound.
A large-step example must be checked at the averaged server iterate: a multiplier
of \(-1\) can cancel when \(R\) is even. Using the threshold \(3\), rather
than \(2\), avoids that cancellation.

\begin{lemma}[Computation lower bound for the averaged server output]
\label{lem:computation}
For every \(L,D_0>0\), \(N,K,R\ge1\), and \(\eta>0\), identical noiseless
quadratic clients can attain
\[
 f(\bar x^R)-f^\star\ge LD_0^2/(48KR).
\]
If \(\eta L>3\), they can attain \(LD_0^2/2\).
\end{lemma}
\begin{proof}
For identical noiseless clients with \(f(x)=\lambda x^2/2\), the
control corrections cancel on each client. Starting from \(x^0=D_0\),
the local update is \(y_{i,k+1}^r=(1-\eta\lambda)y_{i,k}^r\), hence
\[
 x^r=(1-\eta\lambda)^{Kr}D_0,\qquad
 f(\bar x^R)=\frac{\lambda D_0^2}{2}
     \left(\frac1R\sum_{r=1}^R(1-\eta\lambda)^{Kr}\right)^2.
\]
If \(\eta L\le3\), choose \(\lambda=L/(6KR)\). Since
\(0\le\eta\lambda\le1/2\), Bernoulli's inequality gives, for every
\(1\le r\le R\),
\[
 (1-\eta\lambda)^{Kr}
 \ge(1-\eta\lambda)^{KR}
 \ge1-\eta\lambda KR=1-\eta L/6\ge1/2.
\]
Thus \(\bar x^R\ge D_0/2\), and
\(f(\bar x^R)\ge\lambda D_0^2/8=LD_0^2/(48KR)\).

If \(\eta L>3\), use \(f(x)=Lx^2/2\).
The multiplier over a round is \(q=(1-\eta L)^K\), with
\(Q=|q|>2^K\ge2\).
For positive \(q\), \(R^{-1}\sum_{r=1}^R q^r\ge Q^R/R\ge1\).
For negative \(q\), the absolute average is
\[
 \frac{Q(1+Q^R)}{R(1+Q)}\quad(R\text{ odd}),\qquad
 \frac{Q(Q^R-1)}{R(1+Q)}\quad(R\text{ even}).
\]
Using \(Q/(1+Q)\ge2/3\), \(2^R/R\ge2\), and
\((2^R-1)/R\ge3/2\) for even \(R\), each is at least one.
This gives \(LD_0^2/2\). Every sampled loss equals the deterministic objective.
\end{proof}

\subsection{Lower bound for the similarity term}

We now need a family in which heterogeneity and the one-round delay both
remain visible after averaging the clients.  Start with the simplest
two-client scalar example: local Hessians \(0\) and \(2h\), with mean \(h\).
One client is flat and the other has curvature \(2h\); their centered
Hessians are \(-h\) and \(h\).  Thus the similarity scale is exactly \(h\).
This example suggests the delayed recurrence, but repeating the two client
types equally would restrict us to even \(N\).

To avoid that restriction, give every client a two-dimensional quadratic
with curvature \(2h\) along one direction and zero curvature perpendicular
to it.  Choose the directions evenly around a half-circle.  Their
projection matrices average to \(I_2/2\), so the averaged dynamics are
exactly those of the scalar pair above, for every \(N\ge2\).
The term \emph{tight frame} below refers to this elementary projection identity.

The remaining issue is averaging across rounds: a large final iterate
alone does not prevent oscillating server iterates from canceling in
\(\bar x^R\). We therefore derive the delayed recurrence, then choose
its curvature so that the average of its first \(R\) multipliers stays
away from zero. The deterministic lower-bound theorem below combines
this construction with a slowly contracting quadratic for smaller stepsizes.

\begin{lemma}[Delayed recurrence for the quadratic construction]
\label{lem:stale-quadratic-recurrence}
Fix an integer \(N\ge2\), and define
\begin{equation}
    \varphi_i\coloneqq\frac{\pi(i-1)}N,
    \qquad
    u_i\coloneqq
    \begin{pmatrix}
        \cos\varphi_i\\
        \sin\varphi_i
    \end{pmatrix},
    \qquad
    P_i\coloneqq u_i u_i^\top.
    \label{eq:stale-lower-frame-projectors}
\end{equation}
For \(h>0\), consider the two-dimensional local objectives
\begin{equation}
    \phi_i(z)=h z^\top P_i z,
    \qquad
    \phi(z)=\frac h2\norm{z}^2,
    \label{eq:stale-lower-frame-objectives}
\end{equation}
and set
\begin{equation}
    v\coloneqq K\eta_\ell h,
    \qquad
    \tau_K(v)\coloneqq
    \left(1-\frac{2v}{K}\right)^K.
    \label{eq:stale-lower-v-tau}
\end{equation}
Define
\begin{equation}
    B_K(v)\coloneqq\frac{1+\tau_K(v)}2,
    \qquad
    C_K(v)\coloneqq\frac{\tau_K(v)+2v-1}{4},
    \label{eq:stale-lower-BC}
\end{equation}
and the scalar sequence
\begin{equation}
    Q_{-1}(v)=Q_0(v)=1,
    \qquad
    Q_r(v)=B_K(v)Q_{r-1}(v)-C_K(v)Q_{r-2}(v),
    \quad r\ge1.
    \label{eq:stale-lower-Q-recurrence}
\end{equation}
Then the initialized stale-control method applied to
\eqref{eq:stale-lower-frame-objectives} satisfies the vector identity
\begin{equation}
    z^r=Q_r(v)z^0
    \qquad
    \text{for every }r\ge0.
    \label{eq:stale-lower-Q-iterate}
\end{equation}
In particular,
\begin{equation}
    Q_1(v)
    =B_K(v)-C_K(v)
    =\frac{3+\tau_K(v)-2v}{4}.
    \label{eq:stale-lower-Q-first}
\end{equation}
\end{lemma}

\begin{proof}
The vectors \(u_i\) form a finite unit-norm tight frame. To verify the associated projector identity, write
\[
    P_i
    =
    \frac12
    \begin{pmatrix}
        1+\cos(2\varphi_i)&\sin(2\varphi_i)\\
        \sin(2\varphi_i)&1-\cos(2\varphi_i)
    \end{pmatrix},
\]
and
\[
    \sum_{i=1}^N e^{2\mathrm{i}\varphi_i}
    =
    \sum_{i=1}^N e^{2\pi\mathrm{i}(i-1)/N}
    =
    0.
\]
Hence
\begin{equation}
    \frac1N\sum_{i=1}^N P_i
    =
    \frac12 I_2.
    \label{eq:stale-lower-tight-frame}
\end{equation}
The local Hessians \(A_i=2hP_i\) therefore have mean
\(\bar A=hI_2\), which proves the formula for \(\phi\).  Moreover,
\[
    A_i-\bar A
    =
    h(2P_i-I_2),
\]
and \(P_i^2=P_i\) gives
\[
    (2P_i-I_2)^2=I_2.
\]
Thus every \(2P_i-I_2\) is an orthogonal reflection, and
\begin{equation}
    \frac1N\sum_{i=1}^N
    \norm{(A_i-\bar A)z}^2
    =
    h^2\norm{z}^2.
    \label{eq:stale-lower-frame-similarity}
\end{equation}
The similarity constant of this block is exactly \(h\), independently of
\(N\).

Define
\[
    M_i\coloneqq(I_2-\eta_\ell A_i)^K,
    \qquad
    S_i\coloneqq
    \sum_{k=0}^{K-1}(I_2-\eta_\ell A_i)^k.
\]
The matrices \(P_i\) and \(I_2-P_i\) project onto orthogonal directions.
Consequently
\[
    I_2-\eta_\ell A_i
    =(I_2-P_i)+(1-2\eta_\ell h)P_i,
\]
and taking any power just raises the second coefficient to that power.
Summing the resulting scalar geometric progression gives
\begin{align*}
    M_i
    &=
    (I_2-P_i)+\tau_K(v)P_i,\\
    S_i
    &=
    K(I_2-P_i)
    +
    \frac{1-\tau_K(v)}{2\eta_\ell h}P_i.
\end{align*}
Using \eqref{eq:stale-lower-tight-frame}, their client averages satisfy
\begin{equation}
\begin{aligned}
    \frac1N\sum_{i=1}^N M_i
    &=
    B_K(v)I_2,\\
    \eta_\ell h
    \frac1N\sum_{i=1}^N S_i
    &=
    \left[
        \frac v2+\frac{1-\tau_K(v)}4
    \right]I_2.
\end{aligned}
    \label{eq:stale-lower-frame-averages}
\end{equation}

To see where the delay enters, write the local update itself:
\[
    y_{i,k+1}^r
    =y_{i,k}^r-\eta_\ell
      \bigl[A_i(y_{i,k}^r-z^{r-2})+hz^{r-2}\bigr],
    \qquad y_{i,0}^r=z^{r-1}.
\]
Subtracting the fixed anchor \(z^{r-2}\) from both sides yields
\[
    y_{i,k+1}^r-z^{r-2}
    =(I_2-\eta_\ell A_i)(y_{i,k}^r-z^{r-2})
      -\eta_\ell hz^{r-2}.
\]
Iterate this identity \(K\) times.  The initial difference is multiplied
by \(M_i\), and the repeated constant forcing is summed by \(S_i\).
Therefore
\[
    y_{i,K}^r
    =
    z^{r-2}
    +M_i(z^{r-1}-z^{r-2})
    -\eta_\ell S_i h z^{r-2}.
\]
Using \eqref{eq:stale-lower-frame-averages}, averaging the endpoints gives
\[
    z^r
    =
    B_K(v)z^{r-1}-C_K(v)z^{r-2}.
\]
This is a second-order recurrence because the local run starts at
\(z^{r-1}\), whereas the correction is anchored at \(z^{r-2}\).
The convention \(z^{-1}=z^0\) incorporates the exactly initialized first
round.  Indeed, the preceding recurrence at \(r=1\) gives
\[
    z^1
    =
    \bigl(B_K(v)-C_K(v)\bigr)z^0,
\]
which is precisely the fresh-control first-round endpoint.  Induction now
proves \eqref{eq:stale-lower-Q-iterate}.  Thus every \(N\ge2\) gives
the same scalar recurrence as the two-client case.
\end{proof}

\begin{lemma}[Phase parameters for the stale recurrence]
\label{lem:lb-phase-parameters}
Let \(a\ge3\) and \(K\ge4a\). For \(v\in[a,2a]\), define
\(\psi(v)\in(0,\pi/2)\) and \(G(v)\) by
\begin{equation}
\begin{aligned}
    \cos\psi(v)
    &\coloneqq
    \frac{B_K(v)}{2\sqrt{C_K(v)}},\\
    G(v)
    &\coloneqq
    \frac{C_K(v)-B_K(v)/2}
         {\sqrt{C_K(v)-B_K(v)^2/4}}.
\end{aligned}
    \label{eq:stale-lower-phase-parameters}
\end{equation}
Then the recurrence \eqref{eq:stale-lower-Q-recurrence} satisfies
\begin{equation}
    Q_R(v)
    =
    C_K(v)^{R/2}
    \bigl[
        \cos(R\psi(v))-G(v)\sin(R\psi(v))
    \bigr].
    \label{eq:stale-lower-phase-solution}
\end{equation}
Both \(\psi\) and \(G\) are nondecreasing,
and \(\arctan G(2a)-\arctan G(a)\ge1/(32\sqrt a)\).
\end{lemma}
\begin{proof}
For \(v\in[a,2a]\), the condition \(4a\le K\) gives
\[
    0\le1-\frac{2v}{K}\le1.
\]
Consequently, \(0\le\tau_K(v)\le1\), \(1/2\le B_K(v)\le1\), and
\[
    C_K(v)
    \ge
    \frac{2a-1}{4}
    >
    \frac{B_K(v)^2}{4}.
\]
The roots of the characteristic polynomial are therefore a complex
conjugate pair, so the phase parameters are well-defined.
The roots are \(\sqrt{C_K(v)}e^{\pm\mathrm{i}\psi(v)}\), where
\(0<\psi(v)<\pi/2\).  Any real solution is therefore
\(C_K(v)^{r/2}\) times a linear combination of
\(\cos(r\psi(v))\) and \(\sin(r\psi(v))\).
The condition \(Q_0=1\) sets the cosine coefficient to one.  To determine
the other coefficient, use \(Q_1=B_K-C_K\):
\[
    \frac{Q_1}{\sqrt{C_K}}
      -\cos\psi
    =\frac{B_K/2-C_K}{\sqrt{C_K}},
    \qquad
    \sin\psi=\frac{\sqrt{C_K-B_K^2/4}}{\sqrt{C_K}}.
\]
The initial conditions therefore give the sine coefficient \(-G(v)\), proving
\eqref{eq:stale-lower-phase-solution}.

We next show that the total phase changes by a controlled amount on
\([a,2a]\).  Differentiating \(\tau_K\) gives
\[
    \tau_K'(v)
    =-2\left(1-\frac{2v}{K}\right)^{K-1}\le0.
\]
Thus \(B_K\) is nonincreasing, whereas
\[
    C_K'(v)
    =
    \frac12
    \left[
        1-\left(1-\frac{2v}{K}\right)^{K-1}
    \right]
    \ge0.
\]
It follows directly from \eqref{eq:stale-lower-phase-parameters} that
\(\psi\) is nondecreasing.  To treat \(G\), regard it as a function of
\(B=B_K(v)\) and \(C=C_K(v)\).  Since
\[
    G^2=\frac{(C-B/2)^2}{C-B^2/4},
\]
we have
\begin{align*}
    \frac{\partial}{\partial C}\log G^2
    &=
    \frac{2}{C-B/2}-\frac1{C-B^2/4}>0,\\
    \frac{\partial}{\partial B}\log G^2
    &=
    -\frac1{C-B/2}
    +\frac{B}{2(C-B^2/4)}<0.
\end{align*}
Indeed, after taking common denominators, the two numerators are respectively
\(C+B/2-B^2/2>0\) and \(C(B/2-1)<0\).  Hence \(G(v)\) is nondecreasing.

The elementary comparisons
\begin{equation}
    G^2\le C-\frac B2,
    \qquad
    G^2\ge C-B
    \label{eq:stale-lower-G-comparisons}
\end{equation}
follow by multiplying through by \(C-B^2/4\).  For the first inequality,
the remaining condition is \(B\le2\); for the second, the difference of
the two numerators is
\(B^2(1+C-B)/4\ge0\).  Therefore
\begin{equation}
    G(a)^2\le\frac a2-\frac14,
    \qquad
    G(2a)^2\ge a-\frac54.
    \label{eq:stale-lower-G-endpoints}
\end{equation}
Set
\[
    x\coloneqq\sqrt{\frac a2-\frac14},
    \qquad
    y\coloneqq\sqrt{a-\frac54}.
\]
Since \(a\ge3\),
\[
    y-x
    =
    \frac{a/2-1}{x+y}
    \ge
    \frac{\sqrt a}{12},
    \qquad
    1+xy\le1+a\le\frac{4a}{3}.
\]
It follows that
\begin{equation}
\begin{aligned}
    \arctan G(2a)-\arctan G(a)
    &\ge
    \arctan y-\arctan x\\
    &=
    \arctan\left(\frac{y-x}{1+xy}\right)\\
    &\ge
    \frac1{32\sqrt a}.
\end{aligned}
    \label{eq:stale-lower-phase-width}
\end{equation}
Here the last step uses the monotonicity of \(\arctan\) and
\(\arctan t\ge t/2\) for \(0\le t\le1\).

\end{proof}
\begin{lemma}[Preventing cancellation in the averaged recurrence]
\label{lem:averaged-stale-phase}
There is a universal numerical constant \(c_{\mathrm{av}}>0\) such that,
for every \(R\ge1\), \(a\ge3\), and integer \(K\ge4a\), the recurrence in
\eqref{eq:stale-lower-Q-recurrence} satisfies
\begin{equation}
    \max_{v\in[a,2a]}
    \left(
        \frac1R\sum_{r=1}^R Q_r(v)
    \right)^2
    \ge
    \frac{c_{\mathrm{av}}}{R}.
    \label{eq:averaged-stale-phase}
\end{equation}
\end{lemma}

\begin{proof}
The argument has four parts.  First, sum the recurrence exactly so that
all cancellation is retained.  Second, vary \(v\) to keep the oscillatory
part of that sum away from zero.  Third, use its exponentially growing
amplitude to handle sufficiently large \(R\). Finally, find a positive
constant that works for all smaller horizons and all allowed \(a,K\).
This last step is needed because \(a\) and \(K\) can still be arbitrarily
large, even though only finitely many values of \(R\) remain.

\paragraph{Step 1: an exact formula for the whole sum.}
Fix \(v\in[a,2a]\), and abbreviate
\[
    B=B_K(v),
    \qquad
    C=C_K(v),
    \qquad
    \tau=\tau_K(v).
\]
As in Lemma~\ref{lem:lb-phase-parameters}, the condition \(K\ge4a\)
implies \(0\le\tau\le1\), \(1/2\le B\le1\), and
\begin{equation}
    \frac{5a}{12}
    \le
    C
    \le
    a.
    \label{eq:averaged-stale-C-bounds}
\end{equation}
Indeed,
\[
    C\ge\frac{2a-1}{4}\ge\frac{5a}{12},
    \qquad
    C\le\frac{1+4a-1}{4}=a.
\]
In particular, \(C>B^2/4\).

Retain \(\psi(v)\) and \(G(v)\) from
\eqref{eq:stale-lower-phase-parameters}, and set
\[
    \rho\coloneqq\sqrt C,
    \qquad
    d\coloneqq\sqrt{C-\frac{B^2}{4}},
    \qquad
    w\coloneqq\frac B2+\mathrm{i}d
    =\rho e^{\mathrm{i}\psi}.
\]
Equation~\eqref{eq:stale-lower-phase-solution}, applied at every index,
can be written as
\begin{equation}
    Q_r(v)
    =
    \operatorname{Re}
    \left[
        (1+\mathrm{i}G)w^r
    \right].
    \label{eq:averaged-stale-complex-Q}
\end{equation}
Define
\[
    S_R(v)\coloneqq\sum_{r=1}^R Q_r(v),
    \qquad
    \mathcal D(v)\coloneqq|1-w|^2=1-B+C.
\]
Summing \eqref{eq:averaged-stale-complex-Q} and multiplying by
\(\mathcal D=(1-w)(1-\overline w)\) gives the exact decomposition
\begin{equation}
\begin{aligned}
    \mathcal D S_R
    &=
    B-2C+\mathcal L_R,\\
    \mathcal L_R
    &\coloneqq
    \operatorname{Re}
    \left[
        (1+\mathrm{i}G)
        w^{R+1}(\overline w-1)
    \right].
\end{aligned}
    \label{eq:averaged-stale-sum-decomposition}
\end{equation}
For completeness, we can check the boundary terms without complex
notation.  The initialization \(Q_{-1}=Q_0=1\) gives
\[
    \sum_{r=1}^R Q_{r-1}=1+S_R-Q_R,\qquad
    \sum_{r=1}^R Q_{r-2}=2+S_R-Q_R-Q_{R-1}.
\]
These identities hold also at \(R=1\).
Substitute them into \(S_R=B\sum Q_{r-1}-C\sum Q_{r-2}\) and rearrange:
\[
    (1-B+C)S_R
    =
    B-2C+(C-B)Q_R+CQ_{R-1},
\]
and
\[
    (C-B)w^R+Cw^{R-1}
    =
    w^{R+1}(\overline w-1).
\]
Thus the first term in
\eqref{eq:averaged-stale-sum-decomposition} is independent of \(R\),
whereas \(\mathcal L_R\) contains the growing oscillatory part.

\paragraph{Step 2: preventing cancellation in the oscillatory term.}
The extra factor \(\overline w-1\) in \(\mathcal L_R\) comes from
summing a geometric progression.  Its argument contributes an additional
phase, which we now bound.
Write
\begin{equation}
    \chi(v)
    \coloneqq
    \arctan
    \left(
        \frac{d}{1-B/2}
    \right).
    \label{eq:averaged-stale-chi}
\end{equation}
Up to an irrelevant integer multiple of \(\pi\), the phase of
\(\mathcal L_R\) is
\begin{equation}
    \Phi_R(v)
    \coloneqq
    (R+1)\psi(v)+\arctan G(v)+\chi(v).
    \label{eq:averaged-stale-total-phase}
\end{equation}
The proof of Lemma~\ref{lem:lb-phase-parameters} already showed that
\(\psi\) and \(G\) are nondecreasing and that
\begin{equation}
    \arctan G(2a)-\arctan G(a)
    \ge
    \frac1{32\sqrt a}.
    \label{eq:averaged-stale-reused-width}
\end{equation}
We now verify that the additional phase \(\chi\) is also nondecreasing.

Let
\[
    s(v)\coloneqq
    \left(1-\frac{2v}{K}\right)^{K-1},
    \qquad
    \tau'(v)=-2s(v).
\]
Using \(B=(1+\tau)/2\) and
\(C=(\tau+2v-1)/4\), we obtain
\[
    \tan^2\chi
    =
    \frac{8v+2\tau-\tau^2-5}{(3-\tau)^2}.
\]
Differentiation gives
\begin{equation}
    \frac{\mathrm d}{\mathrm dv}\tan^2\chi
    =
    \frac{
        8\left[
            (3-\tau)-s(4v-\tau-1)
        \right]
    }{
        (3-\tau)^3
    }.
    \label{eq:averaged-stale-chi-derivative}
\end{equation}
Because \(K\ge4a\ge12\) and \(v\ge3\),
\[
    s(v)
    \le
    e^{-2v(K-1)/K}
    \le
    e^{-v}.
\]
Hence \(4vs(v)<2\), and therefore
\[
    s(4v-\tau-1)<4vs<2\le3-\tau.
\]
The derivative in
\eqref{eq:averaged-stale-chi-derivative} is positive.  Consequently,
\(\chi\) is nondecreasing, and
\eqref{eq:averaged-stale-reused-width} implies
\begin{equation}
    \Phi_R(2a)-\Phi_R(a)
    \ge
    \frac1{32\sqrt a}.
    \label{eq:averaged-stale-total-width}
\end{equation}

Every interval of length \(d_{\mathrm{ph}}\le1\) contains a point at which the
absolute cosine is at least
\(\sin(d_{\mathrm{ph}}/2)\ge d_{\mathrm{ph}}/\pi\).  Thus there exists
\(v_\star\in[a,2a]\) such that
\[
    |\cos\Phi_R(v_\star)|
    \ge
    \frac1{32\pi\sqrt a}.
\]
Since
\[
    |1+\mathrm{i}G|\ge1,
    \qquad
    |w|^{R+1}=C^{(R+1)/2},
    \qquad
    |\overline w-1|=\sqrt{\mathcal D}\ge\sqrt C,
\]
we have
\begin{equation}
    |\mathcal L_R(v_\star)|
    \ge
    \frac1{32\pi\sqrt a}
    \left(\frac{5a}{12}\right)^{(R+2)/2}.
    \label{eq:averaged-stale-oscillatory-lower}
\end{equation}

\paragraph{Step 3: the growing term dominates at large horizons.}
The lower bound on \(\lvert\mathcal L_R\rvert\) does not immediately
bound \(\lvert S_R\rvert\), since the fixed term \(B-2C\) could have
the opposite sign. We therefore show that \(\lvert\mathcal L_R\rvert\)
is sufficiently larger than \(\lvert B-2C\rvert\). We first treat
\(R\ge64\); Step 4 handles
the remaining horizons.  Put
\[
    A_R(a)
    \coloneqq
    \frac1{32\pi\sqrt a}
    \left(\frac{5a}{12}\right)^{(R+2)/2}.
\]
The bounds \(C\le a\) and \(1/2\le B\le1\) imply
\[
    |B-2C|\le2a,
    \qquad
    \mathcal D\le1+a.
\]
It follows from
\eqref{eq:averaged-stale-sum-decomposition} and
\eqref{eq:averaged-stale-oscillatory-lower} that
\begin{equation}
    \max_{v\in[a,2a]}|S_R(v)|
    \ge
    \frac{A_R(a)-2a}{1+a}.
    \label{eq:averaged-stale-large-R-S}
\end{equation}
Define
\[
    s_\star
    \coloneqq
    \frac14
    \left[
        \frac1{32\pi\sqrt3}
        \left(\frac54\right)^{33}
        -6
    \right]
    >0.
\]
The positivity follows, for example, from
\(\pi<22/7\), \(\sqrt3<7/4\), and
\((5/4)^{10}>9\).  Here \(A_{64}(a)\) is a positive constant times \(a^{65/2}\).
Consequently,
\[
    \frac{\mathrm d}{\mathrm da}
       \frac{A_{64}(a)-2a}{1+a}
    =
    \frac{
      A_{64}(a)\left[\frac{65(1+a)}{2a}-1\right]-2
    }{(1+a)^2}>0.
\]
For the last inequality, \(A_{64}(a)\ge A_{64}(3)>6\) and the expression
in square brackets exceeds \(63/2\).
Thus the ratio is increasing on \([3,\infty)\), and its value is at
least \(s_\star\).
Moreover,
\[
    A_R(a)
    =
    A_{64}(a)
    \left(\frac{5a}{12}\right)^{(R-64)/2}.
\]
Since \(5a/12\ge5/4\), equation
\eqref{eq:averaged-stale-large-R-S} gives
\[
    \max_{v\in[a,2a]}|S_R(v)|
    \ge
    s_\star
    \left(\frac54\right)^{(R-64)/2}.
\]
The function
\[
    R\longmapsto\frac{(5/4)^{R-64}}{R}
\]
is increasing for integers \(R\ge64\).  Consequently,
\begin{equation}
    \max_{v\in[a,2a]}
    \left(\frac{S_R(v)}R\right)^2
    \ge
    \frac{s_\star^2}{64R},
    \qquad
    R\ge64.
    \label{eq:averaged-stale-large-R-final}
\end{equation}

\paragraph{Step 4: uniformity over the remaining horizons.}
For \(1\le R\le63\), we use a different argument: the maximum absolute
value of the averaged multiplier on the curvature interval cannot approach zero.
This gives a positive lower bound independent of \(R,a,K\), without
specifying its numerical value.
Define
\begin{equation}
    m_{\mathrm{fin}}
    \coloneqq
    \inf_{\substack{
        1\le R\le63,\ a\ge3,\\
        K\in\mathbb N,\ K\ge4a
    }}
    \quad
    \max_{v\in[a,2a]}
    \left|
        \frac1R\sum_{r=1}^R Q_r(v)
    \right|.
    \label{eq:averaged-stale-finite-infimum}
\end{equation}
We claim that \(m_{\mathrm{fin}}>0\).
For fixed \(K,R\), the average is an analytic function of \(v\) and
equals one at \(v=0\), so it cannot be zero throughout \([a,2a]\).
We must also rule out its maximum approaching zero as \(a\) and \(K\)
vary. We do this by separating the cases where these parameters stay
bounded or grow without bound.

Suppose otherwise.  There would be sequences \(R_j,a_j,K_j\) for which
the maxima in
\eqref{eq:averaged-stale-finite-infimum} tend to zero.  Passing to a
subsequence, \(R_j=R\) is constant.  First suppose that \(a_j\) is
bounded, and pass to a further subsequence with
\(a_j\to a_\infty\ge3\). Write \(v=a_jt\), so that all maxima are
taken over the fixed interval \(1\le t\le2\).
If \(K_j\) is bounded, it has a constant subsequence.  The corresponding analytic function
\[
    v\longmapsto\frac1R\sum_{r=1}^R Q_r(v)
\]
would vanish on \([a_\infty,2a_\infty]\), hence would vanish
identically.  This is impossible because its value at \(v=0\) is one.
If \(K_j\to\infty\), then
\[
    \left(1-\frac{2v}{K_j}\right)^{K_j}
    \longrightarrow e^{-2v}
\]
uniformly on bounded intervals. Thus the coefficients \(B_K,C_K\)
converge uniformly there. Since \(R\) is fixed, applying the recurrence
successively shows that \(Q_1,\ldots,Q_R\) also converge uniformly,
to the analytic functions obtained by replacing \(\tau_K(v)\) with
\(e^{-2v}\). After the same substitution \(v=a_jt\), the limiting averaged
multiplier would vanish on \([a_\infty,2a_\infty]\).
A real-analytic function on \(\R\) that vanishes on a nonempty interval
vanishes identically, contradicting its value one at \(v=0\).

Finally, suppose that \(a_j\to\infty\), and evaluate at \(v=a_j\).  Then
\[
    0\le\tau_K(a_j)\le e^{-2a_j}\longrightarrow0,
    \qquad
    B_K(a_j)\longrightarrow\frac12,
    \qquad
    C_K(a_j)\sim\frac{a_j}{2}.
\]
For fixed \(R\), the recurrence is a polynomial in \(B,C\), so we can
read off what happens as \(C\) grows.  The first two values are
\[
    Q_1=B-C,\qquad Q_2=B^2-BC-C.
\]
They have leading terms \(-C\) and \(-(1+B)C\).
If the leading terms of \(Q_{2m-1},Q_{2m}\) have the form below,
then \(Q_{2m+1}=BQ_{2m}-CQ_{2m-1}\) has leading term
\((-1)^{m+1}C^{m+1}\).  Applying the recurrence once more gives
the coefficient \((-1)^{m+1}(1+(m+1)B)\) for \(Q_{2m+2}\).
Induction therefore gives the following expansions. The constants in the
\(O(C^{m-1})\) terms can be chosen independently of \(B\in[1/2,1]\):
\begin{align*}
    Q_{2m-1}
    &=
    (-1)^m C^m+O(C^{m-1}),\\
    Q_{2m}
    &=
    (-1)^m(1+mB)C^m+O(C^{m-1}).
\end{align*}
Therefore
\begin{align*}
    \sum_{r=1}^{2m-1}Q_r
    &=
    (-1)^m C^m+O(C^{m-1}),\\
    \sum_{r=1}^{2m}Q_r
    &=
    (-1)^m(2+mB)C^m+O(C^{m-1}).
\end{align*}
The exponentially small difference
\(\tau_K(a_j)=O(e^{-2a_j})\) does not affect these leading terms.
Hence the averaged multiplier at \(v=a_j\) diverges in absolute value,
contradicting convergence of the maximum to zero.  This proves
\(m_{\mathrm{fin}}>0\).

Taking
\[
    c_{\mathrm{av}}
    \coloneqq
    \min\left\{
        m_{\mathrm{fin}}^2,
        \frac{s_\star^2}{64}
    \right\}
    >0
\]
and combining the two horizon ranges proves
\eqref{eq:averaged-stale-phase}.
\end{proof}

\begin{theorem}[Deterministic lower bound for the averaged server output]
\label{thm:stale-convex-averaged-lower-bound}
For \(L_{\mathrm{loc}},D_0>0\), \(0\le\delta\le L_{\mathrm{loc}}/2\), and integers \(N\ge2\), \(K,R\ge1\), for every local stepsize
\(\eta_\ell>0\), there exist \(N\) convex quadratic local objectives in
dimension at most five and an initial point \(x^0\) such that each local objective
is \(L_{\mathrm{loc}}\)-smooth, the smallest constant in
Assumption~\ref{ass:similarity} is exactly \(\delta\), the global
minimizer is \(x^\star=0\), and
\(\norm{x^0-x^\star}=D_0\).  For this instance, the initialized
stale-control method satisfies
\begin{equation}
    f(\bar x^R)-f^\star
    \ge
    c_{\mathrm{avg,lb}}
    \left(
        \frac{L_{\mathrm{loc}}}{K}+\delta
    \right)
    \frac{D_0^2}{R},
    \label{eq:stale-convex-averaged-lower-bound}
\end{equation}
where \(c_{\mathrm{avg,lb}}>0\) is a universal numerical constant.
\end{theorem}

\begin{proof}
Write \(L=L_{\mathrm{loc}}\), \(T=KR\), and
\[
    p\coloneqq K\eta_\ell\delta.
\]
Let \(P_1,\ldots,P_N\) be the tight-frame projectors in
\eqref{eq:stale-lower-frame-projectors}.  Add an inactive two-dimensional block,
give client \(i\) the Hessian \(2\delta P_i\) there, and initialize it at
zero.  It remains zero and makes the smallest similarity constant exactly
\(\delta\).

First suppose that \(\eta_\ell L>3\).  Use a common quadratic of curvature \(L\), and place all the initial distance in this coordinate.  The inactive
block remains at zero.  By
Lemma~\ref{lem:computation},
\[
    f(\bar x^R)-f^\star
    \ge
    \frac{LD_0^2}{2}.
\]
Since \(\delta\le L/2\), \(K,R\ge1\), and therefore
\[
    \left(
        \frac LK+\delta
    \right)\frac1R
    \le
    \frac{3L}{2},
\]
this is at least one third of the right-hand scale in
\eqref{eq:stale-convex-averaged-lower-bound}.

It remains to consider \(\eta_\ell L\le3\).  In this regime,
\begin{equation}
    p
    =
    K\eta_\ell\delta
    \le
    \frac{3K}{2}.
    \label{eq:averaged-stale-p-upper}
\end{equation}
Give a common quadratic coordinate of curvature \(L/(6T)\) and initial magnitude \(D_0/\sqrt2\).  Give an active plane an initial
vector of norm \(D_0/\sqrt2\).  The inactive block is
initialized at zero.  Lemma~\ref{lem:computation} gives
\begin{equation}
    f_{\mathrm{av,cmp}}(\bar x^R)
    -
    f_{\mathrm{av,cmp}}^\star
    \ge
    \frac{LD_0^2}{96KR}.
    \label{eq:averaged-stale-computation-contribution}
\end{equation}

We choose the active block according to the size of \(p\).
By \eqref{eq:averaged-stale-p-upper}, choosing \(a=p/6\) ensures
\(4a\le K\); the phase lemma then
requires \(p>18\).  For the complementary regime \(p\le18\),
the curvature \(\delta/(36R)\) makes the total contraction parameter
at most one half.  These choices cover every stepsize while controlling the averaged output.

\paragraph{The regime \(p\le18\).}
Give every client the active Hessian
\[
    \mu I_2,
    \qquad
    \mu\coloneqq\frac{\delta}{36R}.
\]
The correction cancels in this block, and
\[
    KR\eta_\ell\mu
    =
    \frac p{36}
    \le
    \frac12.
\]
Every server endpoint is a nonnegative scalar multiple of its initial
vector, and Bernoulli's inequality gives
\[
    \left(1-\eta_\ell\mu\right)^{Kr}
    \ge
    \left(1-\eta_\ell\mu\right)^{KR}
    \ge
    \frac12.
\]
The same lower bound holds for the norm of the averaged active vector.
Since its initial squared distance is \(D_0^2/2\), this
block contributes
\begin{equation}
    f_{\mathrm{act}}(\bar x^R)
    -
    f_{\mathrm{act}}^\star
    \ge
    \frac{\delta D_0^2}{576R}.
    \label{eq:averaged-stale-small-p-contribution}
\end{equation}

\paragraph{The regime \(p>18\).}
Set
\[
    a\coloneqq\frac p6.
\]
Then \(a>3\), and
\eqref{eq:averaged-stale-p-upper} gives
\[
    4a=\frac{2p}{3}\le K.
\]
Lemma~\ref{lem:averaged-stale-phase} supplies
\(v\in[a,2a]\) such that
\[
    \left(
        \frac1R\sum_{r=1}^R Q_r(v)
    \right)^2
    \ge
    \frac{c_{\mathrm{av}}}{R}.
\]
Set
\[
    h
    \coloneqq
    \frac{v}{K\eta_\ell}
    =
    \frac vp\delta.
\]
Then
\[
    \frac\delta6\le h\le\frac\delta3.
\]
Give client \(i\) the active Hessian \(2hP_i\).  By
Lemma~\ref{lem:stale-quadratic-recurrence}, its averaged server vector
equals
\[
    \frac1R\sum_{r=1}^R z^r
    =
    \left(
        \frac1R\sum_{r=1}^R Q_r(v)
    \right)z^0.
\]
The active block therefore contributes
\begin{equation}
\begin{aligned}
    f_{\mathrm{act}}(\bar x^R)
    -
    f_{\mathrm{act}}^\star
    &=
    \frac h2\frac{D_0^2}{2}
    \left(
        \frac1R\sum_{r=1}^R Q_r(v)
    \right)^2\\
    &\ge
    \frac{c_{\mathrm{av}}}{24}
    \frac{\delta D_0^2}{R}.
\end{aligned}
    \label{eq:averaged-stale-large-p-contribution}
\end{equation}

We verify the parameter bounds directly.  The common computation block
does not contribute to similarity.  In the small-\(p\) regime, the active
Hessians coincide.  In the large-\(p\) regime, the active and inactive
centered Hessians are \(h(2P_i-I_2)\) and
\(\delta(2P_i-I_2)\), respectively.  Since these normalized matrices are
orthogonal reflections and \(h\le\delta\), the smallest similarity
constant is exactly \(\delta\).  All local Hessian eigenvalues lie in
\([0,L]\), because
\[
    2\delta\le L,
    \qquad
    \mu\le\delta\le L,
    \qquad
    2h\le\frac{2\delta}{3}\le L.
\]

Combining
\eqref{eq:averaged-stale-computation-contribution} with either
\eqref{eq:averaged-stale-small-p-contribution} or
\eqref{eq:averaged-stale-large-p-contribution} proves
\eqref{eq:stale-convex-averaged-lower-bound}, for example with
\[
    c_{\mathrm{avg,lb}}
    \coloneqq
    \min\left\{
        \frac13,
        \frac1{96},
        \frac1{576},
        \frac{c_{\mathrm{av}}}{24}
    \right\}.
\]
\end{proof}

\subsection{The statistical lower bound}
\label{sec:statistical-lower}

To obtain the statistical term, we track the variance of client-averaged
noise. A flat quadratic makes the weights of independent random signs
in the returned averaged server iterate explicit. The following
recursion also shows that the control-error contributions cancel after
server averaging, for every control batch size \(B\ge1\).

\begin{theorem}[Statistical lower bound for every constant stepsize]
\label{thm:statistical-lower}
Fix \(L_{\mathrm{loc}},D_0,\sigma>0\) and integers \(N,K,R\ge1\).
For every constant local stepsize \(\eta_\ell>0\), there exist identical
one-dimensional convex quadratic client objectives, two fixed convex
\(L_{\mathrm{loc}}\)-smooth sampled losses on each client, and an initial
point at distance \(D_0\) from the common minimizer such that
Algorithm~\ref{alg:scaffold-option-i} satisfies
\begin{equation}
    \E[f(\bar x^R)]-f^\star
    \ge \frac1{128}
       \min\left\{
           L_{\mathrm{loc}}D_0^2,\,
           \frac{\sigma D_0}{\sqrt{NKR}}
       \right\}.
\label{eq:statistical-lower}
\end{equation}
The similarity constant is zero.  The sampled losses are drawn
independently across clients and calls, and their gradient noise has
norm exactly \(\sigma\) at every input.
\end{theorem}

\begin{proof}
The proof must prevent two possible ways of making the returned error
small: nearly eliminating the initial displacement, and averaging away
the stochastic perturbations.  A sufficiently flat quadratic lets us
control both quantities explicitly.  We first derive them and only then
compare the prescribed stepsize with the target error scale.

\paragraph{Choice of curvature.}
For the fixed \(\eta_\ell\), set
\begin{equation}
    \lambda\coloneqq
       \min\left\{L_{\mathrm{loc}},
                  \frac1{2\eta_\ell KR}\right\},
    \qquad q\coloneqq1-\eta_\ell\lambda.
    \label{eq:statistical-curvature}
\end{equation}
The curvature is fixed throughout the run.  Since
\(0<\eta_\ell\lambda\le1/(2KR)\), Bernoulli's inequality gives
\begin{equation}
    0<q<1,
    \qquad
    q^{KR}\ge1-\eta_\ell\lambda KR\ge\frac12.
    \label{eq:statistical-contraction}
\end{equation}
Consequently, even after the full horizon, at least half of a
deterministic displacement remains.  The same bound will control the
weight of each noise sample in later iterates.

On every client, take
\begin{equation}
\begin{aligned}
    f_i(x)&=f(x)=\frac{\lambda}{2}x^2,\\
    f_i(x;\omega)&=\frac{\lambda}{2}x^2+\sigma\omega x,
    \qquad
    \mathbb P(\omega=1)=\mathbb P(\omega=-1)=\frac12,\\
    x^0&=D_0,\qquad x^\star=0,\qquad f^\star=0.
\end{aligned}
\label{eq:statistical-instance}
\end{equation}
Both sampled losses are convex and \(L_{\mathrm{loc}}\)-smooth because
their curvature is \(\lambda\le L_{\mathrm{loc}}\).  Their mean is \(f_i\);
their gradient noise is \(\sigma\omega\).  Fresh independent signs
therefore satisfy all the conditional moment and client-independence
assumptions.

The client objectives are identical, but the independently sampled
controls need not agree. The local update is
\begin{equation}
    y_{i,k+1}^r=qy_{i,k}^r
       -\eta_\ell\bigl(c^{r-1}-c_i^{r-1}+\sigma\omega_{i,k}^r\bigr),
    \qquad y_{i,0}^r=x^{r-1}.
    \label{eq:statistical-local-recursion}
\end{equation}
Unrolling this recursion, the contribution of the control correction to
client \(i\)'s endpoint is
\[
 -\eta_\ell\left(\sum_{k=0}^{K-1}q^{K-1-k}\right)
                 (c^{r-1}-c_i^{r-1}).
\]
The coefficient is common to all clients, and
\(\sum_i(c^{r-1}-c_i^{r-1})=0\). Thus these contributions cancel
exactly after server averaging, including in the first round, regardless
of the control batch size.

\paragraph{Decomposition of the returned average.}
Unrolling the local recursion and then averaging over clients gives
\begin{equation}
    x^r=q^Kx^{r-1}
       -\frac{\eta_\ell\sigma}{N}
        \sum_{i=1}^N\sum_{k=0}^{K-1}
             q^{K-1-k}\omega_{i,k}^r.
    \label{eq:statistical-server-recursion}
\end{equation}
Unrolling this server recursion and averaging \(x^1,\ldots,x^R\)
therefore yields
\begin{equation}
    \bar x^R
       =\frac{D_0}{R}\sum_{r=1}^R q^{Kr}
        -\frac{\eta_\ell\sigma}{N}
          \sum_{i=1}^N\sum_{s=1}^R\sum_{k=0}^{K-1}
               w_{s,k}\omega_{i,k}^s,
    \qquad
    w_{s,k}\coloneqq\frac1R
          \sum_{r=s}^R q^{K(r-s)+K-1-k}.
    \label{eq:statistical-output}
\end{equation}
Here \(w_{s,k}\) is the weight of a noise sampled in round \(s\),
local step \(k\), in the returned averaged server iterate.  The signs are
independent, centered, and have variance one.  Thus
\begin{align}
    \E\bar x^R
       &=\frac{D_0}{R}\sum_{r=1}^R q^{Kr}
         \ge\frac{D_0}{2},
       \label{eq:statistical-mean}\\
    \operatorname{Var}(\bar x^R)
       &=\frac{\eta_\ell^2\sigma^2}{N}
         \sum_{s=1}^R\sum_{k=0}^{K-1}w_{s,k}^2.
       \label{eq:statistical-variance}
\end{align}
The first inequality follows from \eqref{eq:statistical-contraction}.
The second identity is where the client-averaging factor \(1/N\)
appears.  No independence between server iterates is assumed.

\paragraph{Variance of the averaged server iterate.}
For \(1\le s\le\lceil R/2\rceil\), the sum defining \(w_{s,k}\)
contains \(R-s+1\ge R/2\) terms.  Each exponent lies between zero
and \(KR-1\), so every term is at least \(q^{KR}\ge1/2\).
Consequently,
\[
    w_{s,k}\ge\frac14,
    \qquad
    1\le s\le\lceil R/2\rceil,\quad 0\le k\le K-1.
\]
There are at least \(KR/2\) such pairs \((s,k)\).  Substituting into
\eqref{eq:statistical-variance} gives
\begin{equation}
    \operatorname{Var}(\bar x^R)
       \ge\frac{\eta_\ell^2\sigma^2KR}{32N}.
    \label{eq:statistical-variance-lower}
\end{equation}
This counting argument includes \(R=1\) and \(K=1\).
Early perturbations affect many subsequent server endpoints, so their
contribution cannot be bounded by treating those endpoints as independent.

Because \(f\) is quadratic, the objective at the returned average has the
exact mean--variance decomposition
\begin{equation}
\begin{aligned}
    \E[f(\bar x^R)]-f^\star
       &=\frac{\lambda}{2}
         \left((\E\bar x^R)^2+\operatorname{Var}(\bar x^R)\right)\\
       &\ge\frac{\lambda D_0^2}{8}
             +\frac{\lambda\eta_\ell^2\sigma^2KR}{64N}.
\end{aligned}
\label{eq:statistical-bias-variance}
\end{equation}
\paragraph{Comparison across stepsizes.}
For this comparison only, denote the target scale by
\[
    \mathcal E\coloneqq
       \min\left\{L_{\mathrm{loc}}D_0^2,
                  \frac{\sigma D_0}{\sqrt{NKR}}\right\}>0.
\]
If \(\eta_\ell\le D_0^2/(KR\mathcal E)\), the first term in
\eqref{eq:statistical-bias-variance} gives
\[
    \E[f(\bar x^R)]-f^\star
       \ge\frac18
          \min\left\{L_{\mathrm{loc}}D_0^2,
                     \frac{D_0^2}{2\eta_\ell KR}\right\}
       \ge\frac{\mathcal E}{16}.
\]
Otherwise,
\(\eta_\ell>D_0^2/(KR\mathcal E)\).  Since \(\mathcal E\le L_{\mathrm{loc}}D_0^2\),
this implies \(\eta_\ell L_{\mathrm{loc}}KR>1\), and hence the
curvature in \eqref{eq:statistical-curvature} is
\(\lambda=1/(2\eta_\ell KR)\).  The second term in
\eqref{eq:statistical-bias-variance} now gives
\[
    \E[f(\bar x^R)]-f^\star
       \ge\frac{\eta_\ell\sigma^2}{128N}
       >\frac{\sigma^2D_0^2}{128NKR\mathcal E}
       \ge\frac{\mathcal E}{128},
\]
where the last inequality uses
\(\mathcal E^2\le\sigma^2D_0^2/(NKR)\).
The two cases prove \eqref{eq:statistical-lower}.
\end{proof}

\subsection{Lower bound from local stochastic updates}
\label{sec:nonlinear-lower}

The upper-bound proof distinguishes short suffixes from longer ones. A matching
construction must distinguish the same time scales. It must also give a
lower bound at the returned average: a lower bound for a single final iterate, or
for an average of objective values, would not suffice.

We adapt the asymmetric piecewise-quadratic analysis of
\citet{glasgow2022}, building on the construction of
\citet{woodworth2020localsgd}, but assign a
different curvature and a different noise-injection time to each coordinate.
We construct the fixed sampled losses and prove their lower bound directly,
rather than applying Glasgow et al.'s final-iterate theorem.

\begin{theorem}[Nonlinear stochastic lower bound for every constant stepsize]
\label{thm:nonlinear-lower}
Fix \(L_{\mathrm{loc}},D_0,\sigma>0\), \(K\ge2\), and \(N,R\ge1\).
For every \(\eta_\ell>0\), an identical-client instance satisfying the
assumptions has
\begin{equation}
\begin{aligned}
 \E[f(\bar x^R)]-f^\star
 \ge c_{\mathrm b}\min\Bigg\{
L_{\mathrm{loc}}D_0^2,\left(\frac{L_{\mathrm{loc}}\sigma^2D_0^4}{KR^2}\right)^{1/3},
 \frac{\sigma D_0}{\sqrt{KR}}
 \sqrt{\log\left(e+
       \frac{L_{\mathrm{loc}}D_0\sqrt K}{\sigma\sqrt R}\right)}
 \Bigg\}
 \label{eq:nonlinear-lower}
\end{aligned}
\end{equation}
for a universal \(c_{\mathrm b}>0\).
The noisy construction has dimension \(K+1\), and can be realized using
two fixed differentiable sampled losses with independent symmetric
sampling and gradient-noise norm at most \(\sigma\) everywhere.
\end{theorem}

\subsubsection{Choice of coordinate curvatures}

In the proof write \(L=L_{\mathrm{loc}}\), \(\eta=\eta_\ell\), and
\(t=\eta LK\). Define
\begin{equation}
 g_\lambda(v)=
 \begin{cases}\lambda v^2/4,&v<0,\\
               \lambda v^2/2,&v\ge0,
 \end{cases}
 \qquad
 g_\lambda'(v)=\frac{\lambda}{4}(3v+|v|).
 \label{eq:nlb-piecewise}
\end{equation}
The derivative is continuous, monotone, and \(\lambda\)-Lipschitz, including
across zero. Thus \(g_\lambda\) is convex and \(\lambda\)-smooth, and
\(g_\lambda(v)\ge\lambda v^2/4\).

Give every client the objective
\begin{equation}
\begin{aligned}
 f(z)&=\frac{\mu}{2}z_0^2+\sum_{q=1}^K g_{\lambda_q}(z_q),\\
 \mu&=\min\{L,1/(2\eta KR)\},\qquad
 \lambda_q=\min\{L,1/(2\eta q)\},\\
 x^0&=D_0e_0,\qquad x^\star=0,\qquad a=\sigma/5.
\end{aligned}
\label{eq:multiscale-objective}
\end{equation}
Its coordinate gradient Lipschitz constants are at most \(L\), hence
its Euclidean gradient Lipschitz constant is at most \(L\), independently
of the dimension. The expected client objectives are identical, so the
control corrections contain only sampling error.

We first describe the desired noise directions as an explicit schedule to
simplify the analysis. We then construct fixed sampled losses whose
stochastic gradients reproduce this schedule along every possible
trajectory. In this auxiliary process, the noise at local step \(k\)
is \(a\omega_{i,k}^r e_{K-k}\), and the noise at every control query is
\(a\omega e_K\). With \(x^{-1}=x^0\), the controls are
\begin{equation}
 c_i^s=\grad f(x^{s-1})+\frac aB\sum_{b=1}^B\omega_{i,b}^{s,c}e_K,
 \qquad c^s=\frac1N\sum_i c_i^s,\qquad 0\le s<R.
 \label{eq:multiscale-controls}
\end{equation}
The local updates and server average satisfy
\begin{equation}
\begin{aligned}
 y_{i,k+1}^r&=y_{i,k}^r-\eta
     \bigl(\grad f(y_{i,k}^r)+c^{r-1}-c_i^{r-1}
                        +a\omega_{i,k}^r e_{K-k}\bigr),\\
 y_{i,0}^r&=x^{r-1},\qquad x^r=\frac1N\sum_i y_{i,K}^r.
\end{aligned}
 \label{eq:multiscale-process}
\end{equation}
All signs are independently symmetric. There is no reset of the noisy
coordinates. The local noise gives coordinate \(q=K-k\) one perturbation
per round, followed by exactly \(q-1\) further local steps. The control
correction is centered, supported on coordinate \(K\), and constant
within the round. The schedule is only a device for analyzing the construction;
the fixed sampled losses introduced below do not depend on the iteration.

The noise-free coordinate is deterministic:
\[
 x_0^r=D_0(1-\eta\mu)^{Kr}\ge D_0/2.
\]
Indeed \(\eta\mu KR\le1/2\), so Bernoulli's inequality applies over the
whole horizon. Its contribution to the returned objective is at least
\begin{equation}
 \frac{\mu}{2}\left(\frac1R\sum_{r=1}^R x_0^r\right)^2
 \ge\frac1{16}\min\{LD_0^2,D_0^2/(\eta KR)\}.
 \label{eq:nlb-slow-coordinate}
\end{equation}

\subsubsection{Expected last local iterate with random starting points and control errors}

In coordinate \(q\), put
\[
 \Phi_q(v)=v-\eta g_{\lambda_q}'(v)
   =(1-3\eta\lambda_q/4)v-\eta\lambda_q|v|/4.
\]
Since \(0<\eta\lambda_q\le1/(2q)\), this map is increasing and concave.
For \(x\le y\),
\[
 \Phi_q(y)-\Phi_q(x)\ge(1-\eta\lambda_q)(y-x).
\]
Let \(M_q(s,u)\) be the expected endpoint of one local round in coordinate
\(q\), starting at \(s\), with constant additional gradient \(u\) and
the scheduled local noise. Write \(M_q(s)=M_q(s,0)\).
For fixed local signs, write \(H_k(s,u)\) for the coordinate after
\(k\) steps. Then \(H_0(s,u)=s\), and
\[
 H_{k+1}(s,u)=\Phi_q(H_k(s,u))-\eta u
       -\eta a\omega_{i,k}^r\mathbf 1_{\{k=K-q\}}.
\]
If \(H_k\) is jointly concave, then for any two pairs \((s,u),(s',u')\)
and \(\theta\in[0,1]\), monotonicity followed by concavity of \(\Phi_q\)
gives
\begin{align*}
 &\Phi_q\bigl(H_k(\theta s+(1-\theta)s',\theta u+(1-\theta)u')\bigr)\\
 &\quad\ge\Phi_q\bigl(\theta H_k(s,u)+(1-\theta)H_k(s',u')\bigr)\\
 &\quad\ge\theta\Phi_q(H_k(s,u))+(1-\theta)\Phi_q(H_k(s',u')).
\end{align*}
The remaining terms in the update are affine in \((s,u)\), so induction
proves joint concavity of \(H_K\). Expectation preserves concavity,
giving joint concavity of \(M_q(s,u)=\E H_K(s,u)\). Monotonicity in
\(s\) follows from the same recursion and monotonicity of \(\Phi_q\).
In particular, a centered control correction satisfies
\begin{equation}
 \E_u M_q(s,u)\le M_q(s,0)=M_q(s).
 \label{eq:multiscale-control-jensen}
\end{equation}
For the zero-control map, coupling two starts with the same local sign gives
\begin{equation}
 M_q(y)-M_q(x)\ge\beta_q(y-x),\qquad
 \beta_q=(1-\eta\lambda_q)^K\in(0,1),\quad x\le y.
 \label{eq:multiscale-map-slope}
\end{equation}

Starting from zero with \(u=0\), the coordinate stays zero until its single
local perturbation. It then undergoes \(q-1\) further gradient steps. Therefore
\begin{equation}
 M_q(0)=-b_q,\qquad
 b_q=\frac{\eta a}{2}
 \left[(1-\eta\lambda_q/2)^{q-1}
              -(1-\eta\lambda_q)^{q-1}\right]\ge0.
 \label{eq:multiscale-bias}
\end{equation}
The positive perturbation contracts faster than the negative one.

Write \(m_{r,q}=\E x_q^r\). Condition on \(\mathcal H_r\), as defined
in Appendix~\ref{app:scaffold-convex-proof}, before revealing the control
batch used in round \(r\). The start \(x^{r-1}\) is fixed under this
conditioning, while each correction \(c^{r-1}-c_i^{r-1}\) is centered
and independent of the current local signs. Apply
\eqref{eq:multiscale-control-jensen} to its \(q\)-th coordinate and
average over clients. Taking expectations, then using concavity in the
starting point, gives
\[
 m_{r,q}\le\E M_q(x_q^{r-1})\le M_q(m_{r-1,q}).
\]
The control corrections may be dependent across clients; this argument
uses their individual conditional means only.
Starting from \(m_{0,q}=0\), monotonicity proves \(m_{r,q}\le-b_q\le0\).
Moreover, \eqref{eq:multiscale-map-slope} at a nonpositive start gives
\[
 M_q(s)\le\beta_qs-b_q\quad(s\le0),\qquad
 m_{r,q}\le-b_q\sum_{j=0}^{r-1}\beta_q^j.
\]
Consequently,
\begin{equation}
 -\E\bar x_q^R\ge b_q,\qquad
 -\E\bar x_q^R\ge
       \frac{b_q}{8}\min\{R,(1-\beta_q)^{-1}\}.
 \label{eq:multiscale-averaged-mean}
\end{equation}
To verify the second bound, for \(0<\beta<1\), Bernoulli's inequality gives
\(\beta^{-r}\ge1+r(\beta^{-1}-1)\ge1+r(1-\beta)\). Therefore
\[
 \sum_{j=0}^{r-1}\beta^j
 =\frac{1-\beta^r}{1-\beta}
 \ge\frac{r}{1+r(1-\beta)}.
\]
Retaining the at least \(R/2\) indices \(r\ge\lceil R/2\rceil\), we obtain
\[
 \frac1R\sum_{r=1}^R\sum_{j=0}^{r-1}\beta^j
 \ge\frac{R}{4(1+R(1-\beta))}
 \ge\frac18\min\{R,(1-\beta)^{-1}\}.
\]
The last inequality uses \(1+x\le2\max\{1,x\}\). Substituting
\(\beta=\beta_q\) proves the stated bound.
This argument does not assume independence between server iterates.

Finally,
\begin{equation}
 \E f(\bar x^R)\ge f(\E\bar x^R)
 \ge\frac{\mu}{2}(\E\bar x_0^R)^2
                  +\frac14\sum_{q=1}^K\lambda_q(\E\bar x_q^R)^2.
 \label{eq:multiscale-jensen}
\end{equation}
This is a lower bound for the objective at the returned average itself.

\subsubsection{The case \texorpdfstring{\(\eta LK\ge1\)}{eta LK >= 1}: the logarithmic term}

Suppose \(t=\eta LK\ge1\), and let
\(q_0=\max\{2,\lceil K/(2t)\rceil\}\).
For \(q\ge q_0\), \(\lambda_q=1/(2\eta q)\). The difference-of-powers
identity and Bernoulli's inequality give
\[
 (1-1/(4q))^{q-1}-(1-1/(2q))^{q-1}
 \ge\frac{q-1}{4q}(1-1/(2q))^{q-2}\ge\frac1{16}.
\]
Thus \(b_q\ge\eta a/32\). Using
\(-\E\bar x_q^R\ge b_q\) and \(\lambda_q=1/(2\eta q)\), each such
coordinate contributes to \eqref{eq:multiscale-jensen} at least
\[
 \frac{\lambda_q}{4}(\E\bar x_q^R)^2
 \ge\frac1{8\eta q}\left(\frac{\eta a}{32}\right)^2
 =\frac{\eta a^2}{8192q}.
\]
Summing these contributions produces the harmonic sum below.

For every \(K\ge2,t\ge1\),
\begin{equation}
 \sum_{q=q_0}^K\frac1q
 \ge\frac14\log(1+\min\{t,K\}).
 \label{eq:multiscale-tail}
\end{equation}
The bound holds for every \(K\ge2\). For \(K=2\), the sum
is \(1/2\), which suffices. For \(K\ge3\) and \(1\le t\le2\),
\(q_0\le\lceil K/2\rceil\), so integration gives a lower bound \(\log2\).
For \(2\le t\le K/2\), \(q_0\le K/t\), giving \(\log t\).
For \(t\ge K/2\), \(q_0=2\), giving
\(\log((K+1)/2)\ge\frac12\log(K+1)\).
These bounds imply \eqref{eq:multiscale-tail} on their respective ranges.

Since \(a=\sigma/5\), we have proved
\begin{equation}
 \E f(\bar x^R)\ge
 \frac{\eta\sigma^2}{819200}
           \log(1+\min\{t,K\}),\qquad t\ge1.
 \label{eq:multiscale-log-noise}
\end{equation}
The logarithmic factor arises from the mean displacement and persists
under averaging across clients and rounds.

\subsubsection{The case \texorpdfstring{\(\eta LK\le1\)}{eta LK <= 1}: the cube-root term}

Suppose \(t\le1\). Retain
\(q\ge\max\{2,\lceil K/2\rceil\}\); there are at least \(K/2\) such
coordinates, and \(\lambda_q\ge L/2\), \(q-1\ge K/4\).
Writing \(u=\eta\lambda_q\), we have \(uq\le1/2\), so
\[
 (1-u/2)^{q-1}-(1-u)^{q-1}
 \ge(q-1)\frac u2(1-u)^{q-2}\ge(q-1)u/4.
\]
Equation~\eqref{eq:multiscale-bias} gives \(b_q\ge\eta a t/64\).
Also \(1-\beta_q\le K\eta\lambda_q\le t\), so
\eqref{eq:multiscale-averaged-mean} yields
\[
 -\E\bar x_q^R\ge\frac{\eta a}{512}\min\{tR,1\}.
\]
Summing \eqref{eq:multiscale-jensen} over the retained coordinates proves
\begin{equation}
 \E f(\bar x^R)\ge
 \frac{L\eta^2\sigma^2K}{104857600}\min\{tR,1\}^2,\qquad t\le1.
 \label{eq:multiscale-cubic-noise}
\end{equation}
When \(tR\ge1\), the bias has had enough rounds to accumulate.
When \(tR<1\), the deterministic contribution
\eqref{eq:nlb-slow-coordinate} already gives at least \(LD_0^2/16\).

\subsubsection{Two fixed sampled losses for all rounds}

For the prescribed finite \(N,K,R,B\), let \(\mathcal P_{\mathrm{query}}\)
be all possible local and control query points of
\eqref{eq:multiscale-controls}--\eqref{eq:multiscale-process}, over all
local and control sign histories. Include the initialization queries and
the server averages that determine later starts. This is a finite set. At local step \(k\) of round \(r\), its
noise-free coordinate is
\[
 D_0(1-\eta\mu)^{K(r-1)+k}.
\]
These values strictly decrease with the local update index. Thus the
noise-free coordinate also identifies the local phase \(k\) from the
query point itself, allowing the scheduled noise direction to be encoded
by fixed sampled losses. A control query is at a server point, whose
index is a multiple of \(K\).
Assign direction \(e_{K-k}\) at each local query and \(e_K\) at each
control query. If a point is queried in both ways, its local phase is
\(k=0\), so the assignments agree. Repeated control queries, including
those at initialization, also receive the same direction. Thus the
assigned direction \(e_{q(p)}\) is well defined at every query point.

Choose \(\rho>0\) smaller than half the minimum separation of distinct
query points. Define the single fixed function
\begin{equation}
 \psi_0(z)=\sum_{p\in\mathcal P_{\mathrm{query}}}
     \ip{e_{q(p)}}{z-p}
       \left(1-\frac{\norm{z-p}^2}{\rho^2}\right)_+^2.
 \label{eq:multiscale-fixed-cutoff}
\end{equation}
The supporting balls are disjoint. Each summand is continuously
differentiable across its boundary, and
\(\nabla\psi_0(p)=e_{q(p)}\).
Inside one ball, with \(s=\norm{z-p}^2/\rho^2\in[0,1]\),
\[
 \nabla\psi_0(z)
 =(1-s)^2e_{q(p)}
 -\frac{4(1-s)}{\rho^2}\ip{e_{q(p)}}{z-p}(z-p).
\]
Its norm is at most \((1-s)^2+4s(1-s)\le5\).
Outside the balls the gradient is zero. Hence the fixed losses
\begin{equation}
 f_i(z;\omega)=f(z)+a\omega\psi_0(z),
 \qquad \Pr(\omega=1)=\Pr(\omega=-1)=1/2
 \label{eq:multiscale-fixed-losses}
\end{equation}
have mean \(f\), unbiased gradients, and gradient-noise norm at most
\(5a=\sigma\) at every input.

Use the same signs for the fixed-loss oracle and the auxiliary process.
Their gradients agree at each query point. Starting at initialization,
each control estimate, local update, and server average therefore agrees
as well. Thus the two processes produce the same trajectories in every
round, for every finite \(B\).
Including the control sign histories enlarges the finite query set but
does not change the dimension or the gradient-noise bound.
The functions depend on the prescribed parameters and stepsize, but are
fixed before the actual signs are sampled. Their distribution never changes.

The concave maps \(M_q(s,u)\) above describe the auxiliary process for
arbitrary scalar starts and constant control corrections. Since the two
processes have the same trajectories, these bounds also apply to the
fixed-loss oracle. This does not require the cutoff losses to be concave.
The sampled losses are not assumed to be convex or
\(L_{\mathrm{loc}}\)-smooth; these conditions apply to their expectations.
Thus the construction satisfies our local-objective assumptions, not the
stronger sample-wise convex/smooth assumptions of Glasgow et al.

\subsubsection{Combining the lower bounds}

The next lemma gives a scalar comparison valid for every positive stepsize.

\begin{lemma}[The scalar noise--progress tradeoff]
\label{lem:scalar-tradeoff}
For \(v>0\), put
\[
 \mathfrak m(v)=\min\{v^{2/3},v\sqrt{\log(e+1/v)}\}.
\]
For every \(t>0\),
\begin{equation}
 \frac1t+v^2t\log(1+t)\ge\frac14\mathfrak m(v).
 \label{eq:scalar-tradeoff}
\end{equation}
Conversely, \(t=1/\mathfrak m(v)\) makes the left side at most \(2\mathfrak m(v)\).
\end{lemma}
\begin{proof}
For the converse, if the first term in the minimum defines \(\mathfrak m\), use
\(\log(1+t)\le t\). If the second term defines it, use
\(t\le1/v\) and \(\log(1+t)\le\log(e+1/v)\).

For the lower bound, first suppose \(v\ge1\), so \(\mathfrak m(v)=v^{2/3}\).
If \(t\ge1\), the noise term is at least \(v^2\log2\).
If \(t\le1\), use \(\log(1+t)\ge t/2\); minimizing
\(1/t+v^2t^2/2\) gives \(3v^{2/3}/2\).

Now suppose \(v\le1\), and put \(\ell_v=\log(e+1/v)\),
\(b=v\sqrt{\ell_v}\). If \(t\le1\), then \(1/t\ge1\ge\mathfrak m(v)\).
If \(1/t\ge b/4\), the claim also follows.
Otherwise \(t>4/b\), and \(\log(1+t)\ge\ell_v/2\).
Indeed, when \(\ell_v\le2\), \(b\le\sqrt2\), which gives the inequality.
When \(\ell_v\ge2\), use \(1/v=e^{\ell_v}-e\ge e^{\ell_v}/2\) and
\(\sqrt{\ell_v}\le e^{\ell_v/2}\).
Thus the noise term is at least \(2v^2\ell_v/b=2b\ge2\mathfrak m(v)\).
\end{proof}

\begin{proof}[Proof of Theorem~\ref{thm:nonlinear-lower}]
Fix the stepsize. If \(\eta L>3\), the deterministic construction in
Lemma~\ref{lem:computation} gives at least \(LD_0^2/2\).
Otherwise \(t\le3K\). If \(tR<1\), use the slow coordinate as already noted.

It remains to consider \(tR\ge1\), \(t\le3K\).
For \(t\le1\), \eqref{eq:multiscale-cubic-noise} and
\(\log(1+t)\le t\) give a universal constant times
\(\eta\sigma^2\log(1+t)\).
For \(t\ge1\), use \eqref{eq:multiscale-log-noise} and
\[
 \log(1+\min\{t,K\})\ge\tfrac12\log(1+t)
 \quad(1\le t\le3K,\ K\ge2).
\]
Together with the slow coordinate, these give, for one universal \(c_0>0\),
\[
 \E f(\bar x^R)\ge
 c_0\left[\frac{D_0^2}{\eta KR}
                   +\eta\sigma^2\log(1+t)\right]
 =c_0\frac{LD_0^2}{R}
       \left[\frac1t+\nu^2t\log(1+t)\right],
 \quad
 \nu=\frac{\sigma\sqrt R}{LD_0\sqrt K}.
\]
Lemma~\ref{lem:scalar-tradeoff} now proves the desired minimum.
The two excluded stepsize regimes supplied \(LD_0^2\) up to constants,
which accounts for the term \(LD_0^2\) in the minimum in \eqref{eq:nonlinear-lower}.
All comparison constants are independent of \(N,K,R,L,D_0,\sigma\).
\end{proof}

\subsection{Assembling the matching theorem}
\label{sec:assembly}

We combine the nonlinear, statistical, and deterministic lower bounds
for the same prescribed stepsize and averaged server output. The matching
upper bound follows from Theorem~\ref{thm:scaffold-convex}.

\subsubsection{Combining the hard instances}

All parameters in \(\mathcal P\) are upper bounds on smoothness,
similarity, and variance. An identical-client instance is admissible for
any prescribed similarity upper bound; a noiseless instance is admissible
for any prescribed variance upper bound.

For the nonzero-similarity contribution,
Theorem~\ref{thm:stale-convex-averaged-lower-bound} gives
\begin{equation}
 \forall\eta_\ell>0\quad\exists\mathcal I_{\mathrm{det}}\in\mathcal P:
 \quad f(\bar x^R)-f^\star
 \ge c_{\mathrm{avg,lb}}
 \left(\frac{L_{\mathrm{loc}}}{K}+\delta\right)\frac{D_0^2}{R}.
 \label{eq:combined-deterministic-input}
\end{equation}
For the homogeneous class, Lemma~\ref{lem:computation} supplies the
needed deterministic term for every \(N\ge1\).

Theorems~\ref{thm:nonlinear-lower} and~\ref{thm:statistical-lower} have
the same quantifier structure:
\[
 \forall\eta_\ell>0\quad\exists\mathcal I_{\eta_\ell}\in\mathcal P.
\]
Thus, for each prescribed stepsize and horizon, we may select the hardest
admissible instance before drawing any samples.

\begin{proof}[Proof of Theorem~\ref{thm:constant-step-optimality}]
For a fixed stepsize, denote the three nonnegative summands of
\(\mathcal T\) by \(A_1,A_2,A_3\), and write \(U=L_{\mathrm{loc}}D_0^2\)
only in this proof. The deterministic lower bound, the statistical theorem,
and the nonlinear theorem give, with a common universal \(c_1>0\),
\[
 \sup_{\mathcal I\in\mathcal P}
       \bigl(\E f(\bar x^R)-f^\star\bigr)
 \ge c_1\max_{j=1,2,3}\min\{U,A_j\}.
\]
Since
\[
 \max_{j=1,2,3}\min\{U,A_j\}
 \ge\frac13\min\{U,A_1+A_2+A_3\},
\]
the lower bound holds uniformly for every prescribed stepsize. Taking
its infimum gives the first inequality of \eqref{eq:lb-optimality}.

The upper inequality follows from Theorem~\ref{thm:scaffold-convex}.
In the homogeneous class,
replace \eqref{eq:combined-deterministic-input} by
Lemma~\ref{lem:computation}; every other argument is unchanged.
\end{proof}

\section{Proofs of Proposition~\ref{prop:nc-fixed-stepsize} and Theorem~\ref{thm:scaffold-nonconvex}}
\label{app:nonconvex}

Selection in Algorithm~\ref{alg:scaffold-nonconvex} precedes
the round's local-gradient sampling and may be deterministic or use
independent randomness, subject to the conditional independence
requirement stated there.
The selected client returns a uniformly sampled local iterate as one
additional vector in its response. The server stores the \(R\) returned
iterates and samples one uniformly after the last round. This output is available without an additional
communication exchange.

We first prove Proposition~\ref{prop:nc-fixed-stepsize}.
Run Algorithm~\ref{alg:scaffold-nonconvex} with \(B=K\),
\(R,K\ge1\), and a constant stepsize satisfying
\begin{equation}
    \eta_\ell L_g\le\frac14,\qquad
    \eta_\ell K\bar\delta_S\le\frac18.
    \label{eq:nc-stepsize}
\end{equation}
Then
\begin{equation}
\begin{aligned}
    \E\norm{\grad f(\widehat x)}^2
    \le\frac{2F_0}{\eta_\ell RK}
       +32\bigl(L_g\eta_\ell+\bar\delta_S^2K\eta_\ell^2\bigr)\sigma^2.
\end{aligned}
\label{eq:nc-fixed-rate}
\end{equation}
This is Proposition~\ref{prop:nc-fixed-stepsize}.

We first derive a bound for arbitrary control batch size \(B\ge1\),
then set \(B=K\) to obtain the displayed bound.
Set \(x^{-1}=x^0\). Let \(\mathcal H_r\) contain the preceding
iterates, previously used controls, and the selected index \(i_r\), but
not the control batch \((c_i^{r-1})_{i=1}^N\) used in round \(r\).
This conditioning is valid under the stated selection rule: that batch
did not generate \(x^{r-1}\) or determine \(i_r\). Define its error in the
selected client's correction by
\begin{equation}
 q^r=c^{r-1}-c_{i_r}^{r-1}
       -\bigl(\grad f(x^{r-2})-\grad f_{i_r}(x^{r-2})\bigr).
 \label{eq:nc-control-error}
\end{equation}
The independent control samples and their conditional moment bounds give
\begin{equation}
 \E[q^r\mid\mathcal H_r]=0,\qquad
 \E[\norm{q^r}^2\mid\mathcal H_r]
 \le\left(1-\frac1N\right)\frac{\sigma^2}{B}\le\frac{\sigma^2}{B}.
 \label{eq:nc-control-moments}
\end{equation}
Indeed, writing each control as its expected gradient plus its centered
batch error, the selected error has coefficient \(-(1-1/N)\), and every
other error has coefficient \(1/N\). Their variances therefore sum to at
most \((1-1/N)\sigma^2/B\). This applies to initialization as well.

For the selected client, define the corrected direction with exact
gradients and controls:
\begin{equation}
 \mu_k^r=\grad f_{i_r}(y_{i_r,k}^r)
       -\grad f_{i_r}(x^{r-2})+\grad f(x^{r-2}).
 \label{eq:nc-mean-direction}
\end{equation}
By uniform similarity,
\begin{equation}
 \norm{\mu_k^r-\grad f(y_{i_r,k}^r)}
 \le\delta_{i_r}\norm{y_{i_r,k}^r-x^{r-2}}
 \le\bar\delta_S\norm{y_{i_r,k}^r-x^{r-2}}.
 \label{eq:nc-correction-error}
\end{equation}
Define the local noise by
\(\zeta_{i_r,k}^r=\grad f_{i_r}(y_{i_r,k}^r;\xi_{i_r,k}^r)
-\grad f_{i_r}(y_{i_r,k}^r)\). Then
\[
 y_{i_r,k+1}^r=y_{i_r,k}^r
                  -\eta_\ell(\mu_k^r+q^r+\zeta_{i_r,k}^r).
\]
Given the controls and preceding local samples, \(\mu_k^r+q^r\) is the
conditional mean direction. The local noise is centered and has second
moment at most \(\sigma^2\). For \(s<t\), conditioning on the controls
and local samples preceding step \(t\) gives
\[
 \E\ip{\zeta_{i_r,s}^r}{\zeta_{i_r,t}^r}=0,
\]
because the earlier noise is known and the new noise has conditional
mean zero. Expanding the squared norm therefore gives
\begin{equation}
 \E\norm{\sum_{t=0}^{k-1}\zeta_{i_r,t}^r}^2
 =\sum_{t=0}^{k-1}\E\norm{\zeta_{i_r,t}^r}^2\le k\sigma^2.
 \label{eq:nc-noise-sum}
\end{equation}
The same control error \(q^r\), in contrast, is used at every local step.
All expectations below are unconditional unless a conditioning is shown.

\subsection{The case \texorpdfstring{\(\eta_\ell KL_g\le1\)}{eta K Lg <= 1}}

\paragraph{Expected decrease in one local step.}
Smoothness and the local noise bounds give
\begin{align*}
 \E f(y_{i_r,k+1}^r)\le{}\E f(y_{i_r,k}^r)
 -\eta_\ell\E\ip{\grad f(y_{i_r,k}^r)}{\mu_k^r+q^r}+\frac{L_g\eta_\ell^2}{2}\E\norm{\mu_k^r+q^r}^2
 +\frac{L_g\eta_\ell^2}{2}\sigma^2.
\end{align*}
Use \(\norm{\mu_k^r+q^r}^2\le2\norm{\mu_k^r}^2+2\norm{q^r}^2\),
\(-2\ip{a}{b}=\norm{a-b}^2-\norm a^2-\norm b^2\),
\(\eta_\ell L_g\le1/4\), and \eqref{eq:nc-correction-error}.
In particular, the terms involving \(\norm{\mu_k^r}^2\) satisfy
\[
 -\frac{\eta_\ell}{2}\norm{\mu_k^r}^2
 +L_g\eta_\ell^2\norm{\mu_k^r}^2
 \le-\frac{\eta_\ell}{4}\norm{\mu_k^r}^2.
\]
Together with \(\E\norm{q^r}^2\le\sigma^2/B\), this gives
\begin{align}
 \E f(y_{i_r,k+1}^r)\le{}&\E f(y_{i_r,k}^r)
 -\frac{\eta_\ell}{2}\E\norm{\grad f(y_{i_r,k}^r)}^2
 -\frac{\eta_\ell}{4}\E\norm{\mu_k^r}^2\notag\\
 &+\frac{\eta_\ell\bar\delta_S^2}{2}\E\norm{y_{i_r,k}^r-x^{r-2}}^2
 -\eta_\ell\E\ip{\grad f(y_{i_r,k}^r)}{q^r}
 +\frac{L_g\eta_\ell^2\sigma^2}{B}
 +\frac{L_g\eta_\ell^2\sigma^2}{2}.
 \label{eq:nc-one-step}
\end{align}

\paragraph{Using the centered control error.}
Equation~\eqref{eq:nc-control-moments} implies
\(\E\ip{\grad f(x^{r-1})}{q^r}=0\). Thus
\[
 \E\ip{\grad f(y_{i_r,k}^r)}{q^r}
 =\E\ip{\grad f(y_{i_r,k}^r)-\grad f(x^{r-1})}{q^r}.
\]
Unroll the local updates:
\[
 y_{i_r,k}^r-x^{r-1}
 =-\eta_\ell\sum_{t=0}^{k-1}\mu_t^r-\eta_\ell kq^r
       -\eta_\ell\sum_{t=0}^{k-1}\zeta_{i_r,t}^r.
\]
Global smoothness first bounds the absolute inner product by
\(L_g\E\left[\norm{y_{i_r,k}^r-x^{r-1}}\cdot\norm{q^r}\right]\).
Substituting the unrolled update and applying Cauchy--Schwarz to each
term, with \eqref{eq:nc-noise-sum}, gives
\begin{align*}
 \left|\E\ip{\grad f(y_{i_r,k}^r)}{q^r}\right|
 \le L_g\eta_\ell\Bigg[
 \frac{\sigma\sqrt{k}}{\sqrt B}
   \left(\sum_{t=0}^{k-1}\E\norm{\mu_t^r}^2\right)^{1/2}+\frac{k\sigma^2}{B}+\frac{\sqrt{k}\sigma^2}{\sqrt B}\Bigg].
\end{align*}
Here we used
\(\E\norm{\sum_{t<k}\mu_t^r}^2\le k\sum_{t<k}\E\norm{\mu_t^r}^2\);
no independence between \(q^r\) and the local iterates is required.
Multiply by \(\eta_\ell\) and sum over \(k\), using
\(\sum_{k<K}\sqrt{k}\le K^{3/2}\), \(\sum_{k<K}k\le K^2/2\),
and bounding each partial sum of direction norms by the full sum. This gives
\begin{align*}
 \eta_\ell\sum_{k=0}^{K-1}
 \left|\E\ip{\grad f(y_{i_r,k}^r)}{q^r}\right|
 \le L_g\eta_\ell^2\Bigg[
 &\frac{\sigma K^{3/2}}{\sqrt B}
       \left(\sum_{k=0}^{K-1}\E\norm{\mu_k^r}^2\right)^{1/2}
 +\frac{K^2\sigma^2}{2B}
 +\frac{K^{3/2}\sigma^2}{\sqrt B}\Bigg].
\end{align*}
For the first term, Young's inequality and \(\eta_\ell KL_g\le1\) yield
\[
 \frac{L_g\eta_\ell^2\sigma K^{3/2}}{\sqrt B}
       \left(\sum_k\E\norm{\mu_k^r}^2\right)^{1/2}
 \le\frac{\eta_\ell}{16}\sum_k\E\norm{\mu_k^r}^2
       +\frac{4L_g^2\eta_\ell^3K^3\sigma^2}{B}
 \le\frac{\eta_\ell}{16}\sum_k\E\norm{\mu_k^r}^2
       +\frac{4L_g\eta_\ell^2K^2\sigma^2}{B}.
\]
For the last term use
\(K^{3/2}/\sqrt B\le(K^2/B+K)/2\). Consequently,
\begin{equation}
 \eta_\ell\sum_{k=0}^{K-1}
 \left|\E\ip{\grad f(y_{i_r,k}^r)}{q^r}\right|
 \le\frac{\eta_\ell}{16}\sum_{k=0}^{K-1}\E\norm{\mu_k^r}^2
 +\frac{5L_g\eta_\ell^2K^2\sigma^2}{B}
 +\frac{L_g\eta_\ell^2K\sigma^2}{2}.
 \label{eq:nc-control-inner-product}
\end{equation}

\paragraph{Local movement and server displacement.}
The displacement from the stale anchor has four terms:
\[
 y_{i_r,k}^r-x^{r-2}
 =(x^{r-1}-x^{r-2})-\eta_\ell\sum_{t=0}^{k-1}\mu_t^r
   -\eta_\ell kq^r-\eta_\ell\sum_{t=0}^{k-1}\zeta_{i_r,t}^r.
\]
Apply \(\norm{a+b+c+d}^2\le4(\norm a^2+\norm b^2+\norm c^2+\norm d^2)\)
and sum over \(k\). The accumulated directions satisfy
\begin{align*}
 \sum_{k=0}^{K-1}\E\norm{\sum_{t=0}^{k-1}\mu_t^r}^2
 &\le\sum_{k=0}^{K-1}k\sum_{t=0}^{k-1}\E\norm{\mu_t^r}^2\\
 &\le\frac{K^2}{2}\sum_{t=0}^{K-1}\E\norm{\mu_t^r}^2.
\end{align*}
For the control and local noise terms, use
\(\E\norm{q^r}^2\le\sigma^2/B\) and
\eqref{eq:nc-noise-sum}, respectively. Together these yield
\begin{align}
 \sum_{k=0}^{K-1}\E\norm{y_{i_r,k}^r-x^{r-2}}^2\le{}&
 4K\E\norm{x^{r-1}-x^{r-2}}^2
 +2\eta_\ell^2K^2\sum_{k=0}^{K-1}\E\norm{\mu_k^r}^2\notag\\
 &+\frac{4\eta_\ell^2K^3\sigma^2}{3B}
 +2\eta_\ell^2K^2\sigma^2.
 \label{eq:nc-path}
\end{align}
Here we used \(\sum_{k<K}k\le K^2/2\) and
\(\sum_{k<K}k^2\le K^3/3\). At the endpoint, since
\(x^r=y_{i_r,K}^r\),
\[
 x^r-x^{r-1}
 =-\eta_\ell\sum_{t=0}^{K-1}\mu_t^r
   -\eta_\ell Kq^r-\eta_\ell\sum_{t=0}^{K-1}\zeta_{i_r,t}^r.
\]
Applying the three-term squared-norm inequality and the same moment
bounds gives
\begin{equation}
 \E\norm{x^r-x^{r-1}}^2
 \le3\eta_\ell^2K\sum_{k=0}^{K-1}\E\norm{\mu_k^r}^2
       +\frac{3\eta_\ell^2K^2\sigma^2}{B}
       +3\eta_\ell^2K\sigma^2.
 \label{eq:nc-endpoint}
\end{equation}

\paragraph{Summation over local steps and rounds.}
Sum \eqref{eq:nc-one-step} over local steps and apply
\eqref{eq:nc-control-inner-product} and \eqref{eq:nc-path}.
The path bound contributes the previous-round displacement term
\[
 \frac{\eta_\ell\bar\delta_S^2}{2}
       \cdot4K\E\norm{x^{r-1}-x^{r-2}}^2
 =2\eta_\ell\bar\delta_S^2K\E\norm{x^{r-1}-x^{r-2}}^2.
\]
To match this coefficient and obtain cancellation across rounds, add
\(2\eta_\ell\bar\delta_S^2K\E\norm{x^r-x^{r-1}}^2\) to both sides.
Bound its contribution on the right by \eqref{eq:nc-endpoint}.
The coefficient of \(\sum_k\E\norm{\mu_k^r}^2\) receives contributions
from the one-step decrease, the control inner product, the path bound,
and the added endpoint term, respectively:
\begin{align*}
 -\frac{\eta_\ell}{4}+\frac{\eta_\ell}{16}
   +\eta_\ell^3\bar\delta_S^2K^2
   +6\eta_\ell^3\bar\delta_S^2K^2
 &=-\frac{3\eta_\ell}{16}+7\eta_\ell^3\bar\delta_S^2K^2\\
 &\le-\frac{5\eta_\ell}{64}\le0,
\end{align*}
where the last line uses \(\eta_\ell K\bar\delta_S\le1/8\).
Discarding it and using \(K\le K^2\), we obtain
\begin{align}
 \E\left[f(x^r)+2\eta_\ell\bar\delta_S^2K\norm{x^r-x^{r-1}}^2\right]
 \le{}&\E\left[f(x^{r-1})+2\eta_\ell\bar\delta_S^2K\norm{x^{r-1}-x^{r-2}}^2\right]
 -\frac{\eta_\ell}{2}\sum_{k=0}^{K-1}\E\norm{\grad f(y_{i_r,k}^r)}^2\notag\\
 &+L_g\eta_\ell^2K\sigma^2
 +\frac{6L_g\eta_\ell^2K^2\sigma^2}{B}
 +7\eta_\ell^3\bar\delta_S^2K^2\sigma^2
 +\frac{7\eta_\ell^3\bar\delta_S^2K^3\sigma^2}{B}.
 \label{eq:nc-round}
\end{align}
The displacement terms telescope over rounds. Since \(x^{-1}=x^0\),
\(f(x^R)\ge f_{\inf}\), and the output is uniform over the \(RK\) local
iterates, summing and dividing by \(\eta_\ell RK/2\) gives
\begin{equation}
 \E\norm{\grad f(\widehat x)}^2\le\frac{2F_0}{\eta_\ell RK}
 +2L_g\eta_\ell\sigma^2+\frac{12L_g\eta_\ell K\sigma^2}{B}
 +14\bar\delta_S^2\eta_\ell^2K\sigma^2\left(1+\frac KB\right).
 \label{eq:nc-small-local-bound}
\end{equation}

\subsection{The case \texorpdfstring{\(\eta_\ell KL_g\ge1\)}{eta K Lg >= 1}}

For this case, keep the control error inside the conditional mean
direction \(\mu_k^r+q^r\). A direct bound on the control error will
contribute \(2\sigma^2/B\) to the final rate. Since
\(\eta_\ell KL_g\ge1\), this is at most
\(2L_g\eta_\ell K\sigma^2/B\), so the centering argument from the
previous case is not needed. The direction's deviation from the global
gradient obeys
\[
 \norm{\mu_k^r+q^r-\grad f(y_{i_r,k}^r)}^2
 \le2\bar\delta_S^2\norm{y_{i_r,k}^r-x^{r-2}}^2+2\norm{q^r}^2.
\]
The same smoothness calculation, using \(\eta_\ell L_g\le1/4\), now gives
\begin{align*}
 \E f(y_{i_r,k+1}^r)\le{}&\E f(y_{i_r,k}^r)
 -\frac{\eta_\ell}{2}\E\norm{\grad f(y_{i_r,k}^r)}^2
 -\frac{\eta_\ell}{4}\E\norm{\mu_k^r+q^r}^2\\
 &+\eta_\ell\bar\delta_S^2\E\norm{y_{i_r,k}^r-x^{r-2}}^2
 +\frac{\eta_\ell\sigma^2}{B}+\frac{L_g\eta_\ell^2\sigma^2}{2}.
\end{align*}
Only the local noise remains outside this direction in the unrolled
updates:
\[
 y_{i_r,k}^r-x^{r-2}
 =(x^{r-1}-x^{r-2})
   -\eta_\ell\sum_{t=0}^{k-1}(\mu_t^r+q^r)
   -\eta_\ell\sum_{t=0}^{k-1}\zeta_{i_r,t}^r.
\]
At \(k=K\), subtracting \(x^{r-1}\) instead leaves only the two sums.
Applying the three-term and two-term squared-norm inequalities, with the
same sum estimates as in the previous case, gives
\begin{align*}
 \sum_{k=0}^{K-1}\E\norm{y_{i_r,k}^r-x^{r-2}}^2\le{}&
 3K\E\norm{x^{r-1}-x^{r-2}}^2
 +\frac32\eta_\ell^2K^2\sum_{k=0}^{K-1}\E\norm{\mu_k^r+q^r}^2
 +\frac32\eta_\ell^2K^2\sigma^2,\\
 \E\norm{x^r-x^{r-1}}^2\le{}&
 2\eta_\ell^2K\sum_{k=0}^{K-1}\E\norm{\mu_k^r+q^r}^2
 +2\eta_\ell^2K\sigma^2.
\end{align*}
Summing the decrease over local steps and substituting the path bound
produces \(3\eta_\ell\bar\delta_S^2K\E\norm{x^{r-1}-x^{r-2}}^2\).
To match this coefficient and telescope across rounds, add
\(3\eta_\ell\bar\delta_S^2K\E\norm{x^r-x^{r-1}}^2\) to both sides.
Bound the added term on the right using the endpoint estimate.
The coefficient of \(\sum_k\E\norm{\mu_k^r+q^r}^2\) then becomes
\begin{align*}
 -\frac{\eta_\ell}{4}
 +\frac32\eta_\ell^3\bar\delta_S^2K^2
 +6\eta_\ell^3\bar\delta_S^2K^2
 &=-\frac{\eta_\ell}{4}+\frac{15}{2}\eta_\ell^3\bar\delta_S^2K^2\\
 &\le-\frac{17\eta_\ell}{128}\le0,
\end{align*}
where the inequality uses \(\eta_\ell K\bar\delta_S\le1/8\).
After discarding this term, the remaining per-round noise terms are
\[
 \frac{\eta_\ell K\sigma^2}{B}
 +\frac{L_g\eta_\ell^2K\sigma^2}{2}
 +\frac{15}{2}\eta_\ell^3\bar\delta_S^2K^2\sigma^2.
\]
The displacement term has coefficient \(3\eta_\ell\bar\delta_S^2K\) on both
sides, so summation over rounds and the same output identity yield
\begin{equation}
 \E\norm{\grad f(\widehat x)}^2\le\frac{2F_0}{\eta_\ell RK}
 +L_g\eta_\ell\sigma^2+15\bar\delta_S^2\eta_\ell^2K\sigma^2
 +\frac{2\sigma^2}{B}.
 \label{eq:nc-large-local-bound}
\end{equation}
Since \(\eta_\ell KL_g\ge1\), the last term is at most
\(2L_g\eta_\ell K\sigma^2/B\).

\subsection{Conclusion for arbitrary control batch sizes}
Equations~\eqref{eq:nc-small-local-bound} and
\eqref{eq:nc-large-local-bound} both imply
\begin{equation}
 \E\norm{\grad f(\widehat x)}^2
 \le\frac{2F_0}{\eta_\ell RK}
 +16\bigl(L_g\eta_\ell+\bar\delta_S^2K\eta_\ell^2\bigr)
       \sigma^2\left(1+\frac KB\right).
 \label{eq:nc-general-batch-rate}
\end{equation}
The stepsize is constant, so the same case holds in every round.
The displacement coefficient therefore remains constant in each
telescoping argument. Taking \(B=K\) proves \eqref{eq:nc-fixed-rate}.
All expectations are finite: global smoothness and uniform similarity
give each local gradient at most linear growth, and the finite-horizon
updates with bounded noise variance preserve finite second moments.
\hfill\(\square\)

\subsection{Stepsize choices and communication complexity}
\label{app:nonconvex-tuning}
The following choices use \(B=K\), as in the main theorem.

\begin{corollary}[Refined stochastic tuning]
\label{cor:nc-refined-tuned}
Choose
\begin{equation}
 \eta_\ell=\left(4L_g+8K\bar\delta_S
 +4\sigma\sqrt{\frac{L_gRK}{F_0}}
 +\left(\frac{16\bar\delta_S^2K^2R\sigma^2}{F_0}\right)^{1/3}\right)^{-1}.
 \label{eq:nc-refined-tuned-step}
\end{equation}
Then
\begin{equation}
 \E\norm{\grad f(\widehat x)}^2
 \le\frac{8L_gF_0}{RK}+\frac{16\bar\delta_SF_0}{R}
 +16\sigma\sqrt{\frac{L_gF_0}{RK}}
 +4\sqrt[3]{16}\left(\frac{\bar\delta_S^2\sigma^2F_0^2}{KR^2}\right)^{1/3}.
 \label{eq:nc-refined-tuned-rate}
\end{equation}
\end{corollary}
\begin{proof}
The denominator is at least \(4L_g\) and \(8K\bar\delta_S\), so the
stepsize conditions hold. The two noise-dependent terms in the denominator
contribute, respectively,
\[
 8\sigma\sqrt{\frac{L_gF_0}{RK}},\qquad
 2\sqrt[3]{16}\left(\frac{\bar\delta_S^2\sigma^2F_0^2}{KR^2}\right)^{1/3}
\]
to \(2F_0/(\eta_\ell RK)\). When the corresponding noise coefficients
are positive, bounding \(\eta_\ell\) by the reciprocal of each
noise-dependent denominator term gives
\begin{align*}
 32L_g\eta_\ell\sigma^2
 &\le\frac{32L_g\sigma^2}
              {4\sigma\sqrt{L_gRK/F_0}}
 =8\sigma\sqrt{\frac{L_gF_0}{RK}},\\
 32\bar\delta_S^2K\eta_\ell^2\sigma^2
 &\le\frac{32\bar\delta_S^2K\sigma^2}
              {(16\bar\delta_S^2K^2R\sigma^2/F_0)^{2/3}}\\
 &=2\sqrt[3]{16}
      \left(\frac{\bar\delta_S^2\sigma^2F_0^2}{KR^2}\right)^{1/3}.
\end{align*}
Adding these estimates in \eqref{eq:nc-fixed-rate} proves the claim and
establishes Theorem~\ref{thm:scaffold-nonconvex}.
If a noise coefficient vanishes, its corresponding bound is zero.
\end{proof}

For completeness, \(\eta_\ell K\bar\delta_S\le1/8\) also gives
\begin{equation}
 \E\norm{\grad f(\widehat x)}^2
 \le\frac{2F_0}{\eta_\ell RK}
       +32\left(L_g+\frac{\bar\delta_S}{8}\right)\eta_\ell\sigma^2.
 \label{eq:nc-simple-fixed-rate}
\end{equation}
This yields the following alternative tuning.
\begin{corollary}[Simpler stochastic tuning]
\label{cor:nc-tuned}
Choose
\begin{equation}
 \eta_\ell=\left(4L_g+8K\bar\delta_S
 +4\sigma\sqrt{\frac{(L_g+\bar\delta_S/8)RK}{F_0}}\right)^{-1}.
 \label{eq:nc-tuned-step}
\end{equation}
Then
\begin{equation}
 \E\norm{\grad f(\widehat x)}^2
 \le\frac{8L_gF_0}{RK}+\frac{16\bar\delta_SF_0}{R}
 +16\sigma\sqrt{\frac{(L_g+\bar\delta_S/8)F_0}{RK}}.
 \label{eq:nc-simple-tuned-rate}
\end{equation}
\end{corollary}
\begin{proof}
The stepsize satisfies \eqref{eq:nc-stepsize}. The noise-dependent part
of \(2F_0/(\eta_\ell RK)\) equals
\(8\sigma\sqrt{(L_g+\bar\delta_S/8)F_0/(RK)}\), and the last term of
\eqref{eq:nc-simple-fixed-rate} is bounded by the same quantity.
\end{proof}

Neither tuned bound is uniformly smaller. For a prescribed horizon,
one may compare \eqref{eq:nc-refined-tuned-rate} and
\eqref{eq:nc-simple-tuned-rate} and use the corresponding stepsize.
The refined bound retains the quadratic dependence on the stepsize in
the similarity-dependent noise term; the simpler bound uses
\(\eta_\ell K\bar\delta_S\le1/8\) to make that dependence linear.

For a target \(\E\norm{\grad f(\widehat x)}^2\le\varepsilon\),
where \(\varepsilon>0\) denotes squared-gradient accuracy, the refined
tuning suffices with
\[
 R=\left\lceil\max\left\{
 \frac{32L_gF_0}{K\varepsilon},\quad
 \frac{64\bar\delta_SF_0}{\varepsilon},\quad
 \frac{4096L_g\sigma^2F_0}{K\varepsilon^2},\quad
 \frac{256\bar\delta_S\sigma F_0}{\sqrt K\,\varepsilon^{3/2}}
 \right\}\right\rceil.
\]
Each term in \eqref{eq:nc-refined-tuned-rate} is then at most
\(\varepsilon/4\). Apart from the initial control aggregation, the
round complexity, including the effect of rounding up, is therefore
\begin{equation}
 \mathcal O\!\left(1+
 \left(\frac{L_g}{K}+\bar\delta_S\right)\frac{F_0}{\varepsilon}
 +\frac{L_g\sigma^2F_0}{K\varepsilon^2}
 +\frac{\bar\delta_S\sigma F_0}{\sqrt K\,\varepsilon^{3/2}}\right).
 \label{eq:nc-refined-complexity}
\end{equation}
Alternatively, \eqref{eq:nc-tuned-step} suffices with
\[
 R=\left\lceil\max\left\{
 \frac{24L_gF_0}{K\varepsilon},\quad
 \frac{48\bar\delta_SF_0}{\varepsilon},\quad
 \frac{2304(L_g+\bar\delta_S/8)\sigma^2F_0}{K\varepsilon^2}
 \right\}\right\rceil,
\]
giving round complexity
\[
 \mathcal O\!\left(1+
 \left(\frac{L_g}{K}+\bar\delta_S\right)\frac{F_0}{\varepsilon}
 +\frac{(L_g+\bar\delta_S)\sigma^2F_0}{K\varepsilon^2}\right).
\]
One can choose the smaller sufficient round count and use its corresponding
stepsize. Each round uses one communication exchange with all clients,
\(K\) local gradient evaluations at the selected client, and \(B\)
gradient evaluations per client for the controls.

For general \(B\), replace \(\sigma\) in the tunings and their resulting
bounds by \(\sigma\sqrt{(1+K/B)/2}\). Equation~\eqref{eq:nc-general-batch-rate}
then gives the same calculations. In particular, the refined rate is
\[
 \mathcal O\!\left(
 \left(\frac{L_g}{K}+\bar\delta_S\right)\frac{F_0}{R}
 +\sigma\sqrt{\left(1+\frac KB\right)\frac{L_gF_0}{KR}}
 +\left(\frac{\bar\delta_S^2\sigma^2(1+K/B)F_0^2}{KR^2}\right)^{1/3}
 \right).
\]
Thus \(B\ge cK\), for any fixed \(c>0\), preserves the rate in
\eqref{eq:nc-tuned-rate} up to constants depending only on \(c\).

\clearpage
\subsection{Comparison with existing nonconvex guarantees}
\label{app:nonconvex-comparison}

\begin{table}[!htbp]
\centering
\caption{Nonconvex convergence guarantees under full participation.
The sum of the last two columns bounds
\(\E\norm{\grad f(\widehat x)}^2\) at the specified output, up to
universal constants. Here \(L=L_{\mathrm{loc}}\),
\(F_0=f(x^0)-f_{\inf}\), and \(\sigma^2\) bounds within-client gradient
variance. The parameters \(\delta\) and \(\bar\delta\) denote averaged
and uniform Hessian similarity, respectively.
For minibatch SGD, \(K\) is the per-client batch size; otherwise it
describes local work, subject to the notes below.}
\label{tab:nonconvex-comparison}
\smallskip
\small
\setlength{\tabcolsep}{4pt}
\renewcommand{\arraystretch}{1.25}
\begin{tabular*}{\textwidth}{@{\extracolsep{\fill}}
 >{\raggedright\arraybackslash}m{0.32\textwidth}
 >{\centering\arraybackslash}m{0.29\textwidth}
 >{\centering\arraybackslash}m{0.35\textwidth}@{}}
\toprule
Method / assumptions & Optimization term & Stochastic terms \\
\midrule
\textbf{Minibatch SGD}\newline
{\footnotesize\citep[Lemma~1]{luo2025revisiting}\newline
No similarity assumption; stochastic}
 & \(\displaystyle \frac{LF_0}{R}\)
 & \(\displaystyle \sigma\sqrt{\frac{LF_0}{NKR}}\) \\
\midrule
\textbf{SCAFFOLD}\textsuperscript{a}\newline
{\footnotesize\citep{karimireddy2020scaffold}\newline
No similarity assumption; stochastic}
 & \(\displaystyle \frac{LF_0}{R}\)
 & \(\displaystyle \sigma\sqrt{\frac{LF_0}{NKR}}\) \\
\midrule
\textbf{SCAFFOLD}\textsuperscript{b}\newline
{\footnotesize\citep{luo2025revisiting}\newline
Averaged similarity; stochastic}
 & \(\displaystyle \left(\frac LK+\sqrt{L\delta}+\rho\right)\frac{F_0}{R}\)
 & \(\displaystyle \sigma\sqrt{\frac{LF_0}{NKR}}
       +\left(\frac{L^2F_0^2\sigma^2}{KR^2}\right)^{1/3}\) \\
\midrule
\textbf{FedRed-SGD}\textsuperscript{c}\newline
{\footnotesize\citep{jiang2024fedred}\newline
Uniform similarity; exact controls}
 & \(\displaystyle \frac{\bar\delta F_0}{R}\)
 & \(\displaystyle \sigma\sqrt{\frac{\bar\delta F_0}{R}}\) \\
\midrule
\textbf{CE-LSGD}\textsuperscript{d}\newline
{\footnotesize\citep{patel2022communication}\newline
Uniform similarity; two-point oracle}
 & \(\displaystyle \left(\frac{\bar L}{\sqrt K}+\bar\delta\right)\frac{F_0}{R}\)
 & \(\displaystyle \frac{\sigma^2}{NKR}
       +\left(\frac{\sigma\bar L F_0}{NKR}\right)^{2/3}\) \\
\midrule
\textbf{This work (\(B=K\))}\newline
{\footnotesize Uniform similarity; stochastic}
 & \(\displaystyle \left(\frac LK+\bar\delta\right)\frac{F_0}{R}\)
 & \(\displaystyle \sigma\sqrt{\frac{LF_0}{KR}}\) \\
\bottomrule
\end{tabular*}
\par\smallskip
\begin{minipage}{\textwidth}
\footnotesize
Minibatch SGD and the original general nonconvex SCAFFOLD analysis use
a sampled server iterate. FedRed-SGD, CE-LSGD, and our method use sampled
local iterates. All clients participate in our control updates, while
one client performs the local steps.
\par\noindent
\textsuperscript{a}The general nonconvex guarantee uses the prescribed
stochastic initialization and tuned local and server stepsizes
(e.g., server stepsize \(\eta_g=\sqrt N\)); see Appendix~E of
\citet{karimireddy2020scaffold}.
\par\noindent
\textsuperscript{b}Theorem~4 of \citet{luo2025revisiting}, with local
weak-convexity parameter \(\rho\), concerns within-round client averages, which are generally
unavailable at the server. The algorithm aggregates fresh minibatch controls.
\par\noindent
\textsuperscript{c}For \(0<\bar\delta\le L\), Theorem~8 of
\citet{jiang2024fedred} sets communication
probability \(p=\bar\delta/L\), so \(K=L/\bar\delta\) in expectation. The displayed
bound substitutes \(R=pT\) into the rate after \(T\) local steps;
\(R\) is an expected communication count. Exact controls are cached
between communications.
\par\noindent
\textsuperscript{d}Theorem~3.1(a) of \citet{patel2022communication}.
The constant \(\bar L\) bounds both local smoothness
and the additional stochastic-gradient regularity in the cited result;
it need not equal \(L_{\mathrm{loc}}\). The oracle evaluates gradients
at two inputs using the same sample. The displayed theorem assumes
\(\bar L F_0/R\lesssim\sigma^2/\sqrt{NK}\) and uses an initial batch of
\(KR\) samples per client. This batch preserves order \(NKR\) total
queries but exceeds the ordinary per-round batch size. The separate
deterministic theorem has the \(L/K+\bar\delta\) dependence.
\par\noindent
Luo et al.'s SCAFFOLD and CE-LSGD use two communication exchanges per
outer iteration; their rates are expressed in terms of communication
counts with this constant factor absorbed. Our method uses one exchange
per round after initialization.
\end{minipage}
\end{table}

\clearpage
\section{Proof of Theorem~\ref{thm:nc-lower}}
\label{app:nc-lower}

\subsection{Instance class and proof organization}
Fix the parameters and eligible set in Theorem~\ref{thm:nc-lower}.
Let \(\mathcal P_{\mathrm{nc}}\) be the class of finite-dimensional
instances satisfying
\[
 f=\frac1N\sum_i f_i,\qquad
 f(x^0)-\inf_x f(x)\le F_0,
\]
local \(L_{\mathrm{loc}}\)-smoothness, uniform similarity with common
upper bound \(\bar\delta\), and the conditional oracle bounds
\eqref{eq:stochastic-unbiasedness}--\eqref{eq:stochastic-oracle}.
The parameters are upper bounds, not required equalities. We use exactly
the client-selection variant and output \eqref{eq:nc-output}, including
minibatch initialization. The constructions below apply for every finite
\(B\ge1\); the matching statement takes \(B=K\). The same sampled losses
are used for local gradients and control batches.
For each fixed positive stepsize and permitted selection rule, we
construct hard instances in this class. The homogeneous instances adapt
\citet{drori2020}; the communication instance adapts
\citet{patel2022communication} and \citet{carmon2020}.
All estimates are proved below with conservative numerical constants.

\subsection{Homogeneous instances: adaptation of quadratic SGD examples}

For deterministic homogeneous instances, every control batch equals the
same gradient, so the correction vanishes and the local updates are
\(T=KR\) consecutive gradient steps. For noisy homogeneous quadratics,
the control corrections need not vanish. We show instead that the second
moment at each local iterate is at least that of SGD with the same
initialization and local noise. A translation gives all examples below
the same initial point.

\begin{lemma}\label{nclb:lem:sgd}
For every constant stepsize and \(T=KR\ge2\), the supremum over
homogeneous instances is at least
\begin{equation}
 \frac18\left[\frac{L_{\mathrm{loc}}F_0}{T}
 +\min\left\{L_{\mathrm{loc}}F_0,
       \sigma\sqrt{\frac{L_{\mathrm{loc}}F_0}{T}}\right\}\right].
 \label{nclb:eq:sgd}
\end{equation}
\end{lemma}
\begin{proof}
Only in this proof, write \(L=L_{\mathrm{loc}}\),
\(\eta=\eta_\ell\), \(M=LF_0\), and \(q=\eta L\).
These scalar calculations specialize the quadratic arguments of
\citet[Supplement A.4--A.5]{drori2020} to our expected-output criterion.

\paragraph{Small-step example.}
Take \(f(u)=au^2/2\), where
\(a=\min\{L,1/(2\eta T)\}\), initialize \(u_0^2=2F_0/a\), and use
no noise. The iterates and squared gradients are
\[
 u_t=(1-\eta a)^t u_0,\qquad
 |f'(u_t)|^2=a^2u_0^2(1-\eta a)^{2t}
             =2aF_0(1-\eta a)^{2t}.
\]
For \(0\le t<T\), Bernoulli's inequality gives
\((1-\eta a)^t\ge1-t\eta a\ge1/2\). Consequently
\begin{equation}
 \frac1T\sum_{t=0}^{T-1}|f'(u_t)|^2
 \ge\frac{aF_0}{2}
 \ge\frac14\min\{M,M/(qT)\}.
 \label{nclb:eq:slow}
\end{equation}

\paragraph{Noise example.}
Give every client the fixed sampled losses
\[
 f_i(v;\omega)=\frac L2v^2+\sigma\omega v,\qquad
 \Pr(\omega=1)=\Pr(\omega=-1)=\frac12,
 \qquad v_0^2=F_0/L.
\]
All local and control samples are independent. The expected objective is
\(f(v)=Lv^2/2\), and the initial gap is \(F_0/2\). Denote the local
noise by \(\zeta_{i_r,k}^r=\sigma\omega_{i_r,k}^r\).
Write \(v_{(r-1)K+k}=y_{i_r,k}^r\) for \(0\le k<K\), and
\(v_{rK}=x^r\). In this homogeneous instance, write
\(e_r=c^{r-1}-c_{i_r}^{r-1}\) for the scalar control error in
\eqref{eq:nc-control-error}. Conditional on \(\mathcal H_r\), it is centered
and independent of the current local signs, including in the first round.
This is where the restriction on client selection is used.

Unrolling the scalar local updates, for \(0\le k\le K\), gives
\begin{equation}
\begin{aligned}
 v_{(r-1)K+k}={}&(1-q)^k v_{(r-1)K}
 -\eta\left(\sum_{j=0}^{k-1}(1-q)^j\right)e_r
 -\eta\sum_{j=0}^{k-1}(1-q)^{k-1-j}\zeta_{i_r,j}^r.
\end{aligned}
 \label{nclb:eq:batch-quadratic-update}
\end{equation}
The conditional cross terms vanish, so
\begin{equation}
\begin{aligned}
 \E[v_{(r-1)K+k}^2\mid\mathcal H_r]
 ={}&(1-q)^{2k}v_{(r-1)K}^2
 +\eta^2\left(\sum_{j=0}^{k-1}(1-q)^j\right)^2
             \E[(e_r)^2\mid\mathcal H_r]\\
 &+\eta^2\sigma^2\sum_{j=0}^{k-1}(1-q)^{2j}\\
 \ge{}&(1-q)^{2k}v_{(r-1)K}^2
       +\eta^2\sigma^2\sum_{j=0}^{k-1}(1-q)^{2j}.
\end{aligned}
 \label{nclb:eq:batch-second-moment}
\end{equation}
This inequality also holds at \(k=K\), the next server iterate.
Induction over rounds therefore yields, for every \(0\le t\le T\),
\begin{equation}
 \E v_t^2\ge(1-q)^{2t}v_0^2
       +\eta^2\sigma^2\sum_{j=0}^{t-1}(1-q)^{2j}.
 \label{nclb:eq:sgd-second-moment-comparison}
\end{equation}
The right-hand side is the SGD second moment. The comparison holds for
every \(q>0\), including when \(1-q\le0\), since the retained
coefficients are squares.

For \(0<q<1\), evaluating the geometric sum gives
\begin{equation}
 \E|f'(v_t)|^2\ge(1-q)^{2t}M+
  \bigl(1-(1-q)^{2t}\bigr)\frac{q\sigma^2}{2-q}
 \ge\frac12\min\{M,q\sigma^2\}.
 \label{nclb:eq:noise}
\end{equation}
The first expression is a weighted average of \(M\) and
\(q\sigma^2/(2-q)\), with nonnegative weights summing to one.
It is therefore at least their minimum; using \(1/(2-q)\ge1/2\)
gives the last inequality.
Unlike the high-probability argument in the source, this second-moment
calculation needs neither increasing dimension nor Gaussian concentration.
Rademacher noise has the required first two moments. For \(q\ge1\),
\eqref{nclb:eq:sgd-second-moment-comparison} gives
\(\E|f'(v_t)|^2\ge q^2\sigma^2\) for every \(t\ge1\), hence the
average is at least \(\sigma^2/2\).

\paragraph{All stepsizes.}
For \(q<1\), the maximum of the two example bounds is at least
\(M/(4T)\), and also at least
\[
 \frac14\min\{M,\max\{M/(qT),q\sigma^2\}\}
 \ge\frac14\min\{M,\sigma\sqrt{M/T}\}.
\]
Here the last inequality follows from
\[
 \max\{M/(qT),q\sigma^2\}
 \ge\sqrt{\frac{M}{qT}\,q\sigma^2}
 =\sigma\sqrt{M/T}.
\]
The maximum of these last two lower bounds is at least their half-sum,
which proves \eqref{nclb:eq:sgd}.
For \(1\le q<2\), a deterministic \(L\)-quadratic with initial gap
\(F_0\) contributes at least \(2M/T\), just from its initial iterate.
Together with the noisy example, the worst-case error is at least
\(\max\{2M/T,\sigma^2/2\}\ge\sigma\sqrt{M/T}\), and again implies
\eqref{nclb:eq:sgd}. For \(q\ge2\), that deterministic quadratic has
nondecreasing gradient magnitude, giving error at least \(2M\), which
also implies \eqref{nclb:eq:sgd}. These are separate admissible instances;
their maximum is a lower bound on the same supremum.
\end{proof}

\subsection{Communication instance: an alternating nonconvex chain}

\paragraph{Origin and modifications.}
We use the alternating-pair allocation of
\citet[Appendix B, equations (B.1)--(B.4), Theorem B.1]{patel2022communication}, based on
\citet{carmon2020}. We use the sum normalization for the global chain,
rescale client groups to handle every \(N\), and prove a support invariant
for the server and control updates of Algorithm~\ref{alg:scaffold-nonconvex}.
The elementary chain estimates are included to
make the normalization and smoothness constants explicit.

The two client groups act on alternating pairs of coordinates: one on
\((2,3),(4,5),\ldots\), and the other on \((1,2),(3,4),\ldots\).
The first group also has a term depending only on coordinate 1, which
starts the updates from zero. We show that after round \(r\), only the
first \(2r\) coordinates can be nonzero, however many local steps are
taken. Choosing \(d=2R+2\) therefore keeps the final coordinate zero
throughout the run. The chain below has gradient norm at least one
whenever that coordinate is zero; rescaling it gives the required
lower bound under the prescribed smoothness and initial gap.

For even \(d\), define
\[
 \psi(t)=\begin{cases}0,&t\le1/2,\\
 \exp\bigl(1-(2t-1)^{-2}\bigr),&t>1/2,
 \end{cases}
 \qquad
 \phi(t)=\sqrt e\int_{-\infty}^t e^{-s^2/2}\,ds,
\]
and \(b(u,v)=\psi(-u)\phi(-v)-\psi(u)\phi(v)\). Set
\begin{align}
 H_1(z)&=-\phi(z_1)+\sum_{j=1}^{d/2-1}b(z_{2j},z_{2j+1}),\notag\\
 H_2(z)&=\sum_{j=1}^{d/2}b(z_{2j-1},z_{2j}),\qquad H=H_1+H_2.
 \label{nclb:eq:chain}
\end{align}
Let \(E_m=\{z:z_j=0\text{ for }j>m\}\), with \(E_0=\{0\}\).
Thus \(E_m\) consists of points whose only possibly nonzero coordinates
are the first \(m\).

\begin{lemma}\label{nclb:lem:chain}
The functions in \eqref{nclb:eq:chain} are smooth and bounded, and satisfy:
\begin{enumerate}
\item \(H(0)-\inf H\le20d\).
\item Each \(H_j\) has gradient Lipschitz constant at most \(C=3000\).
\item If \(z_d=0\), then \(\norm{\grad H(z)}\ge1\).
\item For \(z\in E_m\), \(\grad H_j(z)\in E_{m+1}\) (with indices
truncated at \(d\)). Moreover, \(E_p\) is invariant under \(\grad H_1\)
when \(p\) is odd and under \(\grad H_2\) when \(p\) is even.
\end{enumerate}
\end{lemma}
\begin{proof}
Both scalar functions are smooth, including the flat junction of \(\psi\).
Elementary bounds are
\[
 0\le\psi\le3,\quad 0\le\psi'\le10,\quad |\psi''|\le200,
 \qquad 0<\phi\le5,\quad 0<\phi'\le2,\quad |\phi''|\le2.
\]
For example, putting \(u=(2t-1)^{-1}>0\) gives
\(\psi'=4e u^3e^{-u^2}\) and
\(\psi''=e(16u^6-24u^4)e^{-u^2}\).
The bound \(\sup_{u\ge0}u^m e^{-u^2}=(m/(2e))^{m/2}\)
verifies the stated constants. Each link is at least \(-15\),
the unary term is at least \(-5\), and \(H(0)=-\phi(0)<0\).
This proves the gap bound and boundedness.

The absolute Hessian entries of one link obey
\[
 |b_{uu}|\le2000,\qquad |b_{uv}|\le40,\qquad |b_{vv}|\le12.
\]
The symmetric Hessian therefore has operator norm at most 2040, by its
maximum absolute row sum. Each \(H_j\) consists of disjoint two-coordinate
blocks and, for \(H_1\), a disjoint unary term. Thus 3000 bounds its
Hessian operator norm independently of dimension.

To prove the gradient claim, put \(z_0=1\), and choose the first
\(j\in\{1,\ldots,d\}\) with \(|z_j|<1\); it exists because \(z_d=0\).
Then \(|z_{j-1}|\ge1\). The contribution to \(\partial_jH\) from the
preceding link is
\[
 -\psi(-z_{j-1})\phi'(-z_j)-\psi(z_{j-1})\phi'(z_j)\le-1,
\]
since \(\psi(t)\ge1\) for \(t\ge1\) and \(\phi'(t)\ge1\) for
\(|t|\le1\). For \(j=1\), the unary derivative \(-\phi'(z_1)\le-1\)
gives the same conclusion. The contribution from the following link,
when present, is
\(-\psi'(-z_j)\phi(-z_{j+1})-\psi'(z_j)\phi(z_{j+1})\le0\).
Hence \(\partial_jH\le-1\).

Finally \(\psi(0)=\psi'(0)=0\), so a pair of zero coordinates has zero
gradient. A prefix can activate at most its adjacent coordinate.
The pairs of \(H_1\) end at odd indices and those of \(H_2\) at even
indices, proving the two stronger invariances.
\end{proof}

\begin{lemma}\label{nclb:lem:communication}
For \(0<\bar\delta\le2L_{\mathrm{loc}}\), every \(N\ge2\), every nonempty
\(S\), every positive stepsize, and every permitted selection rule, there
is a deterministic instance in \(\mathcal P_{\mathrm{nc}}\) for which
\begin{equation}
 \norm{\grad f(y_{i_r,k}^r)}^2
 \ge\frac{\bar\delta F_0}{480CR},\qquad C=3000,
 \quad 1\le r\le R,\quad 0\le k\le K.
 \label{nclb:eq:communication}
\end{equation}
The inequality holds for every realization of the selection randomness.
\end{lemma}
\begin{proof}
Choose \(d=2R+2\). Partition the clients into two nonempty groups of
sizes \(n_1=\lfloor N/2\rfloor\) and \(n_2=N-n_1\). Define
\[
 a=\frac{F_0}{20d},\qquad b_0^2=\frac{6Ca}{\bar\delta},\qquad
 f_i(x)=\frac{N}{n_j}\,aH_j(x/b_0)
 \quad\text{for a client in group }j.
\]
Then \(f=aH(x/b_0)\) and \(N/n_j\le3\). By Lemma~\ref{nclb:lem:chain},
\[
 \operatorname{Lip}(\grad f_i)\le\frac{3Ca}{b_0^2}
 =\frac{\bar\delta}2\le L_{\mathrm{loc}},\qquad
 \operatorname{Lip}(\grad f)\le\frac{2Ca}{b_0^2}=\frac{\bar\delta}3.
\]
The gradient of every \(f_i-f\) is therefore Lipschitz with constant at
most \(5\bar\delta/6\le\bar\delta\). Boundedness and the initial gap condition
follow from the same lemma. Set \(x^0=0\) and take every sampled loss to
equal its expected local objective. Every local gradient and control
batch, including at initialization, is then exact for every \(B\ge1\).

We prove that all local iterates in round \(r\) lie in \(E_{2r}\).
Initially every control is supported in \(E_1\). In general the correction
used in round \(r\) is
\[
 -c_{i_r}^{r-1}+c^{r-1}
 =-\grad f_{i_r}(x^{r-2})+\grad f(x^{r-2}).
\]
For \(r=1\) use \(x^{-1}=0\). For \(r\ge2\), the induction hypothesis
gives \(x^{r-2}\in E_{2r-4}\), so this correction lies in
\(E_{2r-3}\). In both cases it lies in \(E_{2r-1}\), while the starting
point lies in \(E_{2r-2}\).

To see why this restricts every local step, if \(y\), \(\grad f_i(y)\),
and the correction all belong to \(E_p\), then
\[
 y-\eta_\ell\bigl(\grad f_i(y)-c_i^{r-1}+c^{r-1}\bigr)\in E_p,
\]
because \(E_p\) is a linear subspace. For a group-1 client,
Lemma~\ref{nclb:lem:chain} ensures that \(y\in E_{2r-1}\) implies
\(\grad f_i(y)\in E_{2r-1}\). This subspace contains the correction
and starting point, so repeated application of the update keeps all
\(K\) local iterates there. For a group-2 client, the same argument
applies to \(E_{2r}\). Scaling the functions and coordinates preserves
these zero coordinates. Thus every choice yields local iterates in \(E_{2r}\),
including the endpoint. This proves the induction, for arbitrary \(K\)
and arbitrary selection sequences. It also covers singleton eligible sets.

All these points have last coordinate zero because \(2r\le2R<d=2R+2\).
The gradient bound therefore holds at every local iterate from which the
output is sampled, as well as at each server endpoint.
Lemma~\ref{nclb:lem:chain} and \(\grad f(x)=(a/b_0)\grad H(x/b_0)\) give
\[
 \norm{\grad f(y_{i_r,k}^r)}^2
 \ge\frac{a^2}{b_0^2}
 =\frac{\bar\delta F_0}{120Cd}
 \ge\frac{\bar\delta F_0}{480CR}.
\]
The client weighting handles odd and even \(N\) in the same proof.
The instance is deterministic, and so satisfies every prescribed variance
upper bound. No oracle samples or choices enter its definition.
\end{proof}

\subsection{Combination and matching upper bound}

\begin{proof}[Proof of Theorem~\ref{thm:nc-lower}]
Fix a stepsize and selection rule. For \(RK\ge2\),
Lemma~\ref{nclb:lem:sgd} lower bounds the worst-case error by
\[
 \frac18\left[\frac{L_{\mathrm{loc}}F_0}{KR}
 +\min\left\{L_{\mathrm{loc}}F_0,
 \sigma\sqrt{\frac{L_{\mathrm{loc}}F_0}{KR}}\right\}\right].
\]
Lemma~\ref{nclb:lem:communication} lower bounds the same supremum by
\(\bar\delta F_0/(480CR)\); for \(\bar\delta=0\) this term is absent.
The maximum of the two bounds is at least half their sum after using
the smaller coefficient. Since \(a+\min\{m,b\}\ge\min\{m,a+b\}\)
for nonnegative \(a,b,m\), this proves \eqref{eq:nc-lower} with
\(c=1/(960C)\).
All constructions use either a noiseless oracle or a fixed quadratic
with an iid random linear perturbation. Taking the infimum over positive
constant stepsizes preserves this lower bound.
\end{proof}

\subsection{Specialization of the upper bound and communication complexity}
\label{app:nc-local-tuning}
Apply Corollary~\ref{cor:nc-tuned} with \(L_g=L_{\mathrm{loc}}\) and
\(\bar\delta_S=\bar\delta\). Its stepsize is admissible when these are upper
bounds on the actual constants. It gives
\[
 \E\norm{\grad f(\widehat x)}^2\le
 \frac{8L_{\mathrm{loc}}F_0}{RK}+\frac{16\bar\delta F_0}{R}
 +16\sigma\sqrt{\frac{(L_{\mathrm{loc}}+\bar\delta/8)F_0}{RK}}.
\]
Since \(\bar\delta\le2L_{\mathrm{loc}}\), the last coefficient is at most
\(8\sqrt5<18\) times the square-root term in
\eqref{eq:nc-local-rate}. We next obtain the other branch of the minimum.

\begin{lemma}[A small-stepsize upper bound]
\label{lem:nc-initial-cap}
Suppose the local objectives are \(L_{\mathrm{loc}}\)-smooth, the initial
objective gap is at most \(F_0>0\), and the stochastic oracle satisfies
Assumption~\ref{ass:stochastic-oracle}. For Algorithm~\ref{alg:scaffold-nonconvex}
with any permitted selection rule, \(R,K,B\ge1\), the constant stepsize
\begin{equation}
 \eta_\ell=
 \frac{\sqrt{2L_{\mathrm{loc}}F_0}}
 {8L_{\mathrm{loc}}KR(\sqrt{2L_{\mathrm{loc}}F_0}+2\sigma)}
 \label{eq:nc-small-stepsize}
\end{equation}
gives \(\E\norm{\grad f(\widehat x)}^2\le4L_{\mathrm{loc}}F_0\).
\end{lemma}
\begin{proof}
Write \(L=L_{\mathrm{loc}}\), \(T=KR\), and
\(G=\sqrt{2LF_0}\). Smoothness and the initial gap give
\(\norm{\grad f(x^0)}\le G\). For a selected local iterate \(y\)
and its control evaluation point \(w\), the exact corrected direction
satisfies
\[
 \norm{\grad f_i(y)-\grad f_i(w)+\grad f(w)}
 \le G+L\norm{y-x^0}+2L\norm{w-x^0}.
\]
The local noise has second moment at most \(\sigma^2\), and
\eqref{eq:nc-control-moments} bounds the control error's second moment
by \(\sigma^2/B\). Define the deterministic sequence
\[
 m_0=0,\qquad
 m_{t+1}=(1+3\eta_\ell L)m_t+\eta_\ell(G+2\sigma).
\]
By induction and the triangle inequality in mean square, \(m_t\)
bounds the root-mean-square distance from \(x^0\) of every iterate
reached in the first \(t\) local updates. The control evaluation point
is a preceding server iterate, so its distance is covered by the same
bound. This argument does not require independence between the current
iterate and its control error. Solving the recursion yields
\[
 m_T=\frac{G+2\sigma}{3L}
       \bigl((1+3\eta_\ell L)^T-1\bigr).
\]
The choice \eqref{eq:nc-small-stepsize} gives
\(3\eta_\ell LT\le3/8\). Using \(e^v-1\le2v\) for
\(0\le v\le3/8\), we obtain
\[
 m_T\le2\eta_\ell T(G+2\sigma)=\frac{G}{4L}.
\]
Since the output is sampled from these local iterates and \(f\) is
\(L\)-smooth,
\[
 \sqrt{\E\norm{\grad f(\widehat x)}^2}
 \le G+Lm_T\le\frac54G.
\]
Squaring gives \(\E\norm{\grad f(\widehat x)}^2
\le25LF_0/8\le4LF_0\).
\end{proof}

For \(B=K\), use \eqref{eq:nc-tuned-step} with
\(L_g=L_{\mathrm{loc}}\) and \(\bar\delta_S=\bar\delta\) when
\[
 \left(\frac{L_{\mathrm{loc}}}{K}+\bar\delta\right)\frac{F_0}{R}
 +\sigma\sqrt{\frac{L_{\mathrm{loc}}F_0}{KR}}
 \le L_{\mathrm{loc}}F_0,
\]
and use \eqref{eq:nc-small-stepsize} otherwise. The preceding estimates
then prove \eqref{eq:nc-local-rate}, with universal constant 18.

\begin{corollary}[Communication rounds]\label{nclb:cor:rounds}
Let \(0<\varepsilon<L_{\mathrm{loc}}F_0/16\). For the prescribed method
to guarantee \(\E\norm{\grad f(\widehat x)}^2\le\varepsilon\)
uniformly over \(\mathcal P_{\mathrm{nc}}\), it is necessary that
\[
 R\ge\max\left\{
 \frac{L_{\mathrm{loc}}F_0}{8K\varepsilon},
 \frac{\bar\delta F_0}{480C\varepsilon},
 \frac{L_{\mathrm{loc}}F_0\sigma^2}{64K\varepsilon^2}
 \right\},\qquad C=3000.
\]
Conversely, the stepsize \eqref{eq:nc-tuned-step}, specialized to
\(L_g=L_{\mathrm{loc}}\) and \(\bar\delta_S=\bar\delta\), guarantees
this accuracy for
\[
 R=\left\lceil\max\left\{
 \frac{24L_{\mathrm{loc}}F_0}{K\varepsilon},
 \frac{48\bar\delta F_0}{\varepsilon},
 \frac{2916L_{\mathrm{loc}}F_0\sigma^2}{K\varepsilon^2}
 \right\}\right\rceil.
\]
Thus, up to universal factors and integer rounding, the necessary and
sufficient round count has scale
\[
 \left(\frac{L_{\mathrm{loc}}}{K}+\bar\delta\right)\frac{F_0}{\varepsilon}
 +\frac{L_{\mathrm{loc}}F_0\sigma^2}{K\varepsilon^2}.
\]
The initial control aggregation is separate from these optimization rounds.
\end{corollary}
\begin{proof}
If \(RK=1\), the returned point is \(x^0\), and a homogeneous quadratic
can give squared gradient \(2L_{\mathrm{loc}}F_0>\varepsilon\).
Otherwise apply Lemma~\ref{nclb:lem:sgd}.
Its first term gives the first necessary inequality. The minimum in
\eqref{nclb:eq:sgd}, together with \(8\varepsilon<L_{\mathrm{loc}}F_0\), implies
\(\sigma\sqrt{L_{\mathrm{loc}}F_0/(KR)}\le8\varepsilon\), giving
the third. Lemma~\ref{nclb:lem:communication} gives the second.
For sufficiency, the three terms of the tuned upper bound, with stochastic
coefficient relaxed to 18, are each at most \(\varepsilon/3\).
\end{proof}

\section{Partial participation}
\label{app:partial-participation}

We extend the deterministic analysis to uniform client sampling.  The
upper bound recovers \eqref{eq:deterministic-rate} under full participation.
For strict partial participation, a lower-bound construction shows that
a term of order \(L_{\mathrm{loc}}D_0^2/R\), with a participation-dependent
coefficient, is necessary even for identical client objectives.
Throughout this section, gradients are exact and expectations are over
client sampling only.

\subsection{Algorithm and convergence guarantee}

Let \(N\ge2\), \(K\ge1\), and \(1\le m\le N\).  We modify
Algorithm~\ref{alg:scaffold-option-i} by sampling, independently in each
round \(r\), a uniform subset \(\mathcal S_r\subseteq[N]\) of size \(m\).
In this deterministic setting, initialize \(c_i^0=\grad f_i(x^0)\) and
\(c^0=N^{-1}\sum_i c_i^0\).  Selected clients start at
\(y_{i,0}^r=x^{r-1}\) and perform the exact-gradient updates
\begin{equation}
 y_{i,k+1}^r=y_{i,k}^r-\eta_\ell
 \bigl(\grad f_i(y_{i,k}^r)-c_i^{r-1}+c^{r-1}\bigr),
 \qquad k=0,\ldots,K-1.
 \label{eq:pp-local-update}
\end{equation}
The server and controls are updated by
\begin{align}
 x^r&=x^{r-1}+\frac{\gamma}{m}
       \sum_{i\in\mathcal S_r}(y_{i,K}^r-x^{r-1}),
 \label{eq:pp-server-update}\\
 c_i^r&=\begin{cases}
       \grad f_i(x^{r-1}),&i\in\mathcal S_r,\\
       c_i^{r-1},&i\notin\mathcal S_r,
       \end{cases}
 \qquad c^r=\frac1N\sum_{i=1}^N c_i^r.
 \label{eq:pp-global-control-update}
\end{align}
In particular, the global control averages all stored controls, including
those of unselected clients.  We choose the server stepsize as
\begin{equation}
 p\coloneqq\frac mN,\qquad
 \rho\coloneqq\frac{N-m}{m(N-1)},\qquad
 \eta_g=\gamma\coloneqq\frac1{1+\rho}\in[1/2,1].
 \label{eq:pp-sampling-parameters}
\end{equation}
The coefficient \(\rho\) is the variance factor for uniform sampling
without replacement.  Retain \(\bar x^R\), \(\Delta_R\), \(D_0\), and
\(\alpha=K\eta_\ell\) from the main text, and set
\begin{equation}
 L_{\mathrm{inc}}^2\coloneqq
 \min\{L_{\mathrm{loc}}^2,L_g^2+\delta^2\},\qquad
 \Lambda_p\coloneqq\frac{2(1-p)(2-p)}{p^2}.
 \label{eq:pp-Linc}
\end{equation}
Here \(L_g\le L_{\mathrm{loc}}\) is a smoothness constant of \(f\);
one may always take \(L_g=L_{\mathrm{loc}}\).

\begin{theorem}[Partial participation]
\label{thm:pp-main}
Suppose Assumptions~\ref{ass:convex-smooth} and~\ref{ass:similarity}
hold.  Run \eqref{eq:pp-local-update}--\eqref{eq:pp-global-control-update}
with the initialization and sampling above.  Choose
\begin{equation}
 \alpha^{-1}=\max\left\{
 \frac{L_{\mathrm{loc}}}{K},\;8\sqrt2\,\delta,\;
 8L_{\mathrm{inc}}\frac{\sqrt{(1-p)(2-p)}}p\right\},
 \qquad \eta_\ell=\frac\alpha K.
 \label{eq:pp-chosen-stepsize}
\end{equation}
Then, for every \(R\ge1\),
\begin{equation}
 \Delta_R=\mathcal{O}\!\left(
 \max\left\{\frac{L_{\mathrm{loc}}}{K},\delta,
 L_{\mathrm{inc}}\frac{\sqrt{(1-p)(2-p)}}p\right\}
 \frac{D_0^2}{R}\right).
 \label{eq:pp-main-rate}
\end{equation}
\end{theorem}

For \(p=1\), the additional term vanishes and \(\gamma=1\), recovering
the full-participation deterministic rate.  For \(p<1\), even identical
clients can hold controls evaluated at different server iterates, so
\(\delta=0\) alone does not eliminate this error.  The lower bound below
shows that a \(K\)-independent \(1/R\) term is necessary for this method;
we do not claim optimal dependence on \(p\).

\subsection{Proof of the upper bound}

The local endpoint estimate is unchanged.  We need only account for
the unrefreshed controls, sampling of client endpoints, and the server
stepsize \(\gamma\).  Let \(\mathcal F_{r-1}\) denote the history before
sampling round \(r\), and write
\(\E_{r-1}[\cdot]=\E[\cdot\mid\mathcal F_{r-1}]\).
Conditional on this history, define \eqref{eq:pp-local-update} for every
client as an analysis device.  Only selected clients compute these
updates.  With \(x^{-1}=x^0\), set
\[
 G_r\coloneqq f(x^r)-f^\star,\qquad
 D_r\coloneqq\norm{x^r-x^\star},\qquad
 b_r\coloneqq\norm{x^r-x^{r-1}},\qquad
 s_r\coloneqq\norm{x^r-x^{r-2}}.
\]
Thus \(b_0=0\), \(s_1=b_1\), and
\(s_r^2\le2b_r^2+2b_{r-1}^2\) for \(r\ge2\).

\paragraph{Control error.}
The centered gradient increments have zero population mean, so
\[
\begin{aligned}
 \frac1N\sum_i\|\grad f_i(x)-\grad f_i(y)\|^2
 &=\|\grad f(x)-\grad f(y)\|^2\\
 &\quad+\frac1N\sum_i
 \|[\grad f_i(x)-\grad f(x)]-[\grad f_i(y)-\grad f(y)]\|^2.
\end{aligned}
\]
Combining this identity with similarity and local smoothness gives
\begin{equation}
 \frac1N\sum_i\|\grad f_i(x)-\grad f_i(y)\|^2
 \le L_{\mathrm{inc}}^2\|x-y\|^2,
 \qquad L_g\le L_{\mathrm{inc}}.
 \label{eq:pp-increment-bound}
\end{equation}

At the beginning of round \(r\), compare every stored control with the exact
gradient at the preceding server model \(x^{r-2}\):
\begin{equation}
    \ell_i^{r-1}
    \coloneqq
    c_i^{r-1}-\grad f_i(x^{r-2}),
    \qquad
    H_{r-1}
    \coloneqq
    \frac1N\sum_{i=1}^N\norm{\ell_i^{r-1}}^2.
    \label{eq:pp-lag-definition}
\end{equation}

Exact initialization gives
\(H_0=0\).  Round~1 refreshes a selected control to the same value it already
has, while every unselected control remains equal to its initialized value.
Therefore
\begin{equation}
    H_0=H_1=0.
    \label{eq:pp-initial-lag}
\end{equation}

Define
\begin{equation}
    \chi_p\coloneqq\frac{(1-p)(2-p)}p.
    \label{eq:pp-lag-constants}
\end{equation}
Both \(\chi_p\) and \(\Lambda_p\) are zero when \(p=1\).

\begin{lemma}[One-round and cumulative lag bounds]
\label{lem:pp-lag}
For every \(r\ge2\),
\begin{equation}
    \E_{r-1}H_r
    \le
    \left(1-\frac p2\right)H_{r-1}
    +\chi_p L_{\mathrm{inc}}^2 b_{r-1}^2.
    \label{eq:pp-lag-recursion}
\end{equation}
Consequently, for every \(R\ge1\),
\begin{equation}
    \sum_{r=2}^R\E H_{r-1}
    \le
    \Lambda_p L_{\mathrm{inc}}^2
    \sum_{j=1}^{R-1}\E b_j^2.
    \label{eq:pp-cumulative-lag}
\end{equation}
Empty sums are understood as zero.
\end{lemma}

\begin{proof}
Let
\[
    g_i^{r-1}
    =\grad f_i(x^{r-1})-\grad f_i(x^{r-2}).
\]
After round \(r\), the new lag relative to \(x^{r-1}\) is zero if client
\(i\) was selected.  If it was not selected, then
\[
    \ell_i^r=\ell_i^{r-1}-g_i^{r-1}.
\]
Every client is excluded with probability \(1-p\), so
\begin{equation}
    \E_{r-1}H_r
    =(1-p)\frac1N\sum_{i=1}^N
      \norm{\ell_i^{r-1}-g_i^{r-1}}^2.
    \label{eq:pp-lag-exact}
\end{equation}
When \(p<1\), Young's inequality with
\(t=p/[2(1-p)]\) gives
\[
    \norm{a-b}^2
    \le(1+t)\norm a^2+(1+t^{-1})\norm b^2.
\]
Multiplication by \(1-p\) produces the coefficients
\[
    (1-p)(1+t)=1-\frac p2,
    \qquad
    (1-p)(1+t^{-1})=\chi_p.
\]
Equation~\eqref{eq:pp-increment-bound} gives
\[
    \frac1N\sum_{i=1}^N\norm{g_i^{r-1}}^2
    \le L_{\mathrm{inc}}^2 b_{r-1}^2,
\]
which proves \eqref{eq:pp-lag-recursion}.  When \(p=1\), every control is
refreshed and both sides of \eqref{eq:pp-lag-recursion} are zero.

Taking total expectations in \eqref{eq:pp-lag-recursion} and iterating
from \(H_1=0\) gives, for \(r\ge2\),
\[
 \E H_{r-1}
 \le\chi_p L_{\mathrm{inc}}^2
       \sum_{j=1}^{r-2}(1-p/2)^{r-2-j}\E b_j^2.
\]
When summing over \(r\), the total weight of each \(\E b_j^2\) is
at most \(\sum_{s=0}^\infty(1-p/2)^s=2/p\). Extending the resulting
sum of nonnegative movement terms to \(j=R-1\) therefore gives
\[
    \sum_{r=2}^R\E H_{r-1}
    \le\frac{2\chi_p}{p}L_{\mathrm{inc}}^2
       \sum_{j=1}^{R-1}\E b_j^2,
\]
and \(2\chi_p/p=\Lambda_p\).
\end{proof}

\paragraph{Endpoint comparison.}

Fix a round and condition on \(\mathcal F_{r-1}\).  Write
\(x=x^{r-1}\), \(w=x^{r-2}\), \(c_i=c_i^{r-1}\), and \(c=c^{r-1}\).
The local steps are gradient steps on the same type of linearly corrected objective as
under full participation:
\[
    \phi_i(y)=f_i(y)+\ip{c-c_i}{y-w}.
\]
Lemma~\ref{lem:app-refined-endpoint}, with \(\nu=0\), applies
conditionally to any fixed linear correction.
Thus, when \(\eta_\ell L_{\mathrm{loc}}\le1\), it gives directly
\begin{equation}
 2\alpha[\phi_i(y_{i,K}^r)-\phi_i(u)]
 \le\|x-u\|^2-\|y_{i,K}^r-u\|^2.
 \label{eq:pp-local-endpoint}
\end{equation}

To express the server relaxation, move each endpoint toward the round's
starting point by defining
\begin{equation}
    v_i^r
    \coloneqq
    x+\gamma(y_{i,K}^r-x),
    \qquad
    z^r\coloneqq\frac1N\sum_{i=1}^N v_i^r,
    \qquad
    e_r^2\coloneqq\frac1N\sum_{i=1}^N\norm{v_i^r-z^r}^2.
    \label{eq:pp-contracted-endpoints}
\end{equation}
These points rewrite the existing server update; they require no
additional local steps. Equation~\eqref{eq:pp-server-update} becomes
\begin{equation}
    x^r=\frac1m\sum_{i\in\mathcal S_r}v_i^r.
    \label{eq:pp-server-sampled-v}
\end{equation}
Uniform sampling without replacement gives
\begin{equation}
 \E_{r-1}x^r=z^r,\qquad
 \E_{r-1}\|x^r-z^r\|^2=\rho e_r^2.
 \label{eq:pp-sampling-variance}
\end{equation}
Indeed, the inclusion probabilities are \(m/N\) and
\(m(m-1)/(N(N-1))\) for one and two distinct clients, respectively.
Since \(\sum_i(v_i^r-z^r)=0\), the cross terms satisfy
\[
 \sum_{i\ne j}\ip{v_i^r-z^r}{v_j^r-z^r}
 =-\sum_i\norm{v_i^r-z^r}^2.
\]
Expanding the squared norm using the inclusion probabilities therefore gives
\begin{align*}
 \E_{r-1}\norm{x^r-z^r}^2
 &=\frac1{m^2}\left[\frac mN-\frac{m(m-1)}{N(N-1)}\right]
       \sum_i\norm{v_i^r-z^r}^2\\
 &=\frac{N-m}{m(N-1)}e_r^2=\rho e_r^2.
\end{align*}
Consequently, for every comparison point
\(u\) fixed given \(\mathcal F_{r-1}\),
\begin{align}
 \E_{r-1}\|x^r-u\|^2&=\|z^r-u\|^2+\rho e_r^2,
 \label{eq:pp-sampling-distance}\\
 \E_{r-1}f(x^r)&\le f(z^r)+\frac{L_g\rho}{2}e_r^2.
 \label{eq:pp-sampling-objective}
\end{align}
The second inequality follows from global smoothness.

The full-participation corrected-objective comparison
(Lemma~\ref{lem:app-corrected-comparison}) changes only through the stored controls.
With \(\ell_i=c_i-\grad f_i(w)\) and
\(\bar\ell=N^{-1}\sum_i\ell_i\), the centered correction at \(z=z^r\) is
\[
 [\grad f_i(z)-\grad f(z)]-[\grad f_i(w)-\grad f(w)]
 +\bar\ell-\ell_i.
\]
Its root-mean-square norm is at most
\(\delta\|z-w\|+\sqrt{H_{r-1}}\), since
\(N^{-1}\sum_i\|\ell_i-\bar\ell\|^2\le H_{r-1}\).
The same convexity and Cauchy--Schwarz argument therefore gives
\begin{equation}
 \frac1N\sum_i[\phi_i(v_i^r)-\phi_i(u)]
 \ge f(z^r)-f(u)
       -(\delta\|z^r-w\|+\sqrt{H_{r-1}})e_r.
 \label{eq:pp-comparison}
\end{equation}
The additional \(\sqrt{H_{r-1}}e_r\) term is the only change to this
comparison.  The cancellation still relies on
\(c=N^{-1}\sum_i c_i\), including unselected clients' stored controls.

\paragraph{One-round progress.}

We now combine local descent, contraction, corrected-objective comparison, and
sampling.  This is the central one-round estimate.

\begin{lemma}[Joint endpoint inequality]
\label{lem:pp-joint-endpoint}
Suppose \(\eta_\ell L_{\mathrm{loc}}\le1\), and define
\begin{equation}
    \tau
    \coloneqq
    \frac1\gamma-\rho(1+\alpha L_g).
    \label{eq:pp-tau}
\end{equation}
Then, conditionally on \(\mathcal F_{r-1}\), for every comparison point \(u\) fixed given this history,
\begin{align}
    2\alpha\E_{r-1}[f(x^r)-f(u)]
    \le{}&
    2\alpha(1-\gamma)\bigl(f(x^{r-1})-f(u)\bigr)
    +\norm{x^{r-1}-u}^2
    -\E_{r-1}\norm{x^r-u}^2
    \nonumber\\
    &-\tau e_r^2
    +2\alpha
      \bigl(\delta\norm{z^r-x^{r-2}}+\sqrt{H_{r-1}}\bigr)e_r.
    \label{eq:pp-joint-before-square}
\end{align}
If, in addition, \(\tau\ge7/8\), then
\begin{align}
    2\alpha\E_{r-1}[f(x^r)-f(u)]
    \le{}&
    2\alpha(1-\gamma)\bigl(f(x^{r-1})-f(u)\bigr)
    +\norm{x^{r-1}-u}^2
    -\E_{r-1}\norm{x^r-u}^2
    \nonumber\\
    &+\frac{16}{7}\alpha^2\delta^2
       \E_{r-1}\norm{x^r-x^{r-2}}^2
    +\frac{16}{7}\alpha^2H_{r-1}.
    \label{eq:pp-joint-endpoint}
\end{align}
\end{lemma}

\begin{proof}
For brevity, write \(y_i=y_{i,K}^r\), \(v_i=v_i^r\), \(z=z^r\), and
\(e=e_r\).  Since \(v_i=(1-\gamma)x+\gamma y_i\), convexity of
\(\phi_i\) and \eqref{eq:pp-local-endpoint} imply
\begin{align}
    2\alpha\bigl(\phi_i(v_i)-\phi_i(u)\bigr)
    \le{}&
    2\alpha(1-\gamma)\bigl(\phi_i(x)-\phi_i(u)\bigr)
    \nonumber\\
    &+\gamma\bigl(\norm{x-u}^2-\norm{y_i-u}^2\bigr).
    \label{eq:pp-contract-local}
\end{align}
The last term has the exact representation
\begin{equation}
    \gamma\bigl(\norm{x-u}^2-\norm{y_i-u}^2\bigr)
    =\norm{x-u}^2-\norm{v_i-u}^2
     -\frac{1-\gamma}{\gamma}\norm{v_i-x}^2.
    \label{eq:pp-contraction-identity}
\end{equation}
Average \eqref{eq:pp-contract-local} over clients.  The control corrections have
zero population mean, so
\[
    \frac1N\sum_{i=1}^N
    \bigl(\phi_i(x)-\phi_i(u)\bigr)=f(x)-f(u).
\]
Apply \eqref{eq:pp-comparison}, and use
\[
    \frac1N\sum_i\norm{v_i-u}^2=\norm{z-u}^2+e^2,
    \qquad
    \frac1N\sum_i\norm{v_i-x}^2=\norm{z-x}^2+e^2.
\]
We obtain
\begin{align}
    2\alpha(f(z)-f(u))
    \le{}&
    2\alpha(1-\gamma)(f(x)-f(u))
    +\norm{x-u}^2-\norm{z-u}^2
    -\frac1\gamma e^2
    \nonumber\\
    &-\frac{1-\gamma}{\gamma}\norm{z-x}^2
    +2\alpha
      \bigl(\delta\norm{z-w}+\sqrt{H_{r-1}}\bigr)e.
    \label{eq:pp-population-endpoint}
\end{align}
The term \(\norm{z-u}^2\) is converted to the sampled server
iterate using \eqref{eq:pp-sampling-distance}.  The objective is converted using
\eqref{eq:pp-sampling-objective}.  Their two variance contributions are
\(\rho e^2\) and \(\alpha L_g\rho e^2\), respectively.  Dropping the
nonpositive term involving \(\norm{z-x}^2\) gives
\eqref{eq:pp-joint-before-square}.

If \(\tau\ge7/8\), completing the square yields
\begin{align*}
    -\tau e^2
    +2\alpha
      \bigl(\delta\norm{z-w}+\sqrt{H_{r-1}}\bigr)e
    &\le
    \frac{\alpha^2}{\tau}
    \bigl(\delta\norm{z-w}+\sqrt{H_{r-1}}\bigr)^2\\
    &\le
    \frac{16}{7}\alpha^2
    \bigl(\delta^2\norm{z-w}^2+H_{r-1}\bigr).
\end{align*}
Finally, \(z=\E_{r-1}x^r\), and the bias--variance identity gives
\[
    \norm{z-w}^2
    \le\E_{r-1}\norm{x^r-w}^2.
\]
Substitution proves \eqref{eq:pp-joint-endpoint}.
\end{proof}

The choice \(\gamma=(1+\rho)^{-1}\) changes
\eqref{eq:pp-tau} into
\begin{equation}
    \tau=1-\rho\alpha L_g.
    \label{eq:pp-tau-cancel}
\end{equation}
Thus the finite-population variance \(\rho e_r^2\) is canceled by the
dispersion supplied by endpoint contraction; only the smoothness remainder
\(\rho\alpha L_g e_r^2\) remains to be absorbed.

\begin{proof}[Proof of Theorem~\ref{thm:pp-main}]
We adapt the telescoping distance and descent inequalities used in
Lemma~\ref{lem:app-server-reduction}.  Its statement does not apply
directly because the new error \(H_{r-1}\) varies with the history.
We prove the stronger fixed-stepsize estimate
\begin{equation}
 \sum_{r=1}^R\E G_r\le\frac{227}{224\alpha\gamma}D_0^2
 \label{eq:pp-sum-gap}
\end{equation}
whenever
\begin{equation}
 \eta_\ell L_{\mathrm{loc}}\le1,\qquad
 \alpha^2\delta^2\le\frac1{128},\qquad
 \alpha^2\Lambda_p L_{\mathrm{inc}}^2\le\frac1{32}.
 \label{eq:pp-stepsize-conditions}
\end{equation}
The third stepsize condition and \(L_g\le L_{\mathrm{inc}}\) imply, for
\(p<1\),
\[
 \alpha L_g\le\frac{p}{8\sqrt{(1-p)(2-p)}},\qquad
 \rho\alpha L_g\le\frac{\sqrt{1-p}}{8(N-1)\sqrt{2-p}}\le\frac18.
\]
For \(p=1\), \(\rho=0\).  Hence \(\tau\ge7/8\) in both cases.

Set
\[
 a_1=\frac{16}{7}\alpha^2\delta^2,\qquad
 a_2=\frac{16}{7}\alpha^2,\qquad
 a_3=a_2\Lambda_p L_{\mathrm{inc}}^2.
\]
The stepsize conditions give \(a_1\le1/56\) and \(a_3\le1/14\).
As in the full-participation proof, take \(u=x^\star\) and \(u=x^{r-1}\)
in the joint endpoint inequality.  After total expectation, the two
recursions are
\begin{align}
 2\alpha\E G_r
 &\le2\alpha(1-\gamma)\E G_{r-1}
   +\E D_{r-1}^2-\E D_r^2
   +a_1\E s_r^2+a_2\E H_{r-1},
 \label{eq:pp-gap-recursion}\\
 2\alpha(\E G_{r-1}-\E G_r)
 &\ge\E b_r^2-a_1\E s_r^2-a_2\E H_{r-1}.
 \label{eq:pp-descent-recursion}
\end{align}
Relative to full participation, these introduce a previous-gap term and
the control error; the displacement terms have the same form.

\paragraph{Initialization.}
Relaxation leaves an initial-gap contribution, but it does not introduce
an extra factor \(\alpha L_g\) into the final bound:
\begin{equation}
 2\alpha(1-\gamma)G_0
 \le\alpha(1-\gamma)L_gD_0^2\le\frac18D_0^2,
 \label{eq:pp-initial-gap-residual}
\end{equation}
because \(1-\gamma\le\rho\).  For \(p=1\) the contribution is zero.
Using \(H_0=0\), \(s_1=b_1\), and
\(b_1^2\le2D_0^2+2D_1^2\) in \eqref{eq:pp-gap-recursion} gives
\[
 2\alpha\E G_1\le\frac{65}{56}D_0^2-\frac{54}{56}\E D_1^2.
\]
In particular,
\begin{equation}
 \E D_1^2\le\frac{65}{54}D_0^2,\qquad
 \E b_1^2\le\frac{119}{27}D_0^2,\qquad
 2\alpha\E G_1\le\frac{65}{56}D_0^2.
 \label{eq:pp-first-round-bounds}
\end{equation}
This is the same separate-first-round step used under full participation.

\paragraph{Movement bound.}
Let \(U_2=\sum_{r=2}^R\E b_r^2\) and \(U=\E b_1^2+U_2\).
Summing \eqref{eq:pp-descent-recursion} over \(r=2,\ldots,R\) and
dropping the nonpositive terminal term \(-2\alpha\E G_R\) gives
\[
 U_2\le2\alpha\E G_1
       +a_1\sum_{r=2}^R\E s_r^2
       +a_2\sum_{r=2}^R\E H_{r-1}.
\]
Use \(\sum_{r=2}^R\E s_r^2\le4U_2+2\E b_1^2\) and the cumulative
lag bound \eqref{eq:pp-cumulative-lag}. Since
\(\sum_{j=1}^{R-1}\E b_j^2\le U_2+\E b_1^2\), this yields
\[
 U_2\le2\alpha\E G_1
       +a_1(4U_2+2\E b_1^2)
       +a_3(U_2+\E b_1^2).
\]
Rearranging gives
\[
 (1-4a_1-a_3)U_2
 \le2\alpha\E G_1+(2a_1+a_3)\E b_1^2.
\]
Thus the control error changes the full-participation absorption coefficient by
\(a_3\).  Since \(1-4a_1-a_3\ge6/7\) and \(2a_1+a_3\le3/28\),
\eqref{eq:pp-first-round-bounds} yields
\begin{equation}
 U_2\le\frac{823}{432}D_0^2,\qquad
 U\le\frac{101}{16}D_0^2.
 \label{eq:pp-movement-bound}
\end{equation}
Empty sums cover \(R=1\) as well.

\paragraph{Summing the objective-gap inequalities.}
Let \(T_R=\sum_{r=1}^R\E G_r\). The previous-round gaps satisfy
\[
 \sum_{r=1}^R\E G_{r-1}
 =G_0+T_R-\E G_R\le G_0+T_R.
\]
Summing \eqref{eq:pp-gap-recursion}, using this identity,
\(\sum_r\E s_r^2\le4U\), and the cumulative lag bound, then absorbing
\(2\alpha(1-\gamma)T_R\) on the left gives
\[
\begin{aligned}
 2\alpha\gamma T_R
 &\le 2\alpha(1-\gamma)G_0+D_0^2+(4a_1+a_3)U\\
 &\le\left(\frac18+1+\frac17\frac{101}{16}\right)D_0^2
 =\frac{227}{112}D_0^2.
\end{aligned}
\]
This proves \eqref{eq:pp-sum-gap}.  The choice
\eqref{eq:pp-chosen-stepsize} satisfies \eqref{eq:pp-stepsize-conditions}, since
\(\sqrt{32\Lambda_p}=8\sqrt{(1-p)(2-p)}/p\).
Convexity and \(\gamma\ge1/2\) then give
\(\Delta_R\le227D_0^2/(224\alpha\gamma R)\), proving
\eqref{eq:pp-main-rate}.
\end{proof}

\subsection{Limits of the \texorpdfstring{\(1/K\)}{1/K} improvement}

\begin{theorem}[A \(K\)-independent lower bound]
\label{thm:pp-lower}
Fix \(N\ge2\), \(1\le m<N\), \(K\ge1\), \(R\ge4\), \(L_{\mathrm{loc}}>0\), and \(D_0>0\).
Define
\begin{equation}
    \beta\coloneqq\gamma^2p(1-p)>0.
    \label{eq:pp-beta}
\end{equation}
Run \eqref{eq:pp-local-update}--\eqref{eq:pp-global-control-update} with the exact initialization and uniform sampling above.  For every constant local stepsize \(0<\eta_\ell\le1/L_{\mathrm{loc}}\), there exist identical,
differentiable, convex, \(L_{\mathrm{loc}}\)-smooth client objectives and an initialization
with \(\norm{x^0-x^\star}=D_0\), such that
\begin{equation}
    \Delta_R
    \ge
    \frac{\gamma^2p(1-p)}{1024}
    \frac{L_{\mathrm{loc}}D_0^2}{R}.
    \label{eq:pp-lower}
\end{equation}
The objectives satisfy Assumption~\ref{ass:similarity} with
\(\delta=0\).
\end{theorem}

For fixed \(N,m\) with \(m<N\), the coefficient in
\eqref{eq:pp-lower} is positive and independent of \(K,R\).  Consequently,
no choice of constant \(0<\eta_\ell\le1/L_{\mathrm{loc}}\) guarantees
\(\mathcal{O}(L_{\mathrm{loc}}D_0^2/(KR))\) uniformly over this class,
even with an additional \(\mathcal{O}(L_{\mathrm{loc}}D_0^2/R^2)\) term.
The instance may depend on the prescribed stepsize and horizon.
This establishes the need for a \(K\)-independent \(1/R\) term for the
specified algorithm and output, but does not match the participation
dependence in Theorem~\ref{thm:pp-main}.

The proof uses two explicit objectives.  If \(K\eta_\ell\) is small, a shallow quadratic has insufficient progress.
If it is large, an asymmetric convex objective makes the stale corrections
push the expected server iterate far into a region of small curvature.
That negative mean persists for the remainder of the horizon.

\subsection{Proof of the lower bound}

The asymmetric-curvature mechanism is also used in
Section~\ref{sec:nonlinear-lower}.  Here the perturbations come from
unequal stored controls, even though every local gradient is exact.
We first show that the mean displacement persists under later sampling.

Fix \(0<\mu\le L_{\mathrm{loc}}\) and let every client have the scalar objective
\begin{equation}
    \varphi(x)=
    \begin{cases}
        \dfrac{\mu}{2}x^2,&x<0,\\[2pt]
        \dfrac{L_{\mathrm{loc}}}{2}x^2,&x\ge0.
    \end{cases}
    \label{eq:pp-asymmetric}
\end{equation}
This function is differentiable at zero, convex, and has an
\(L_{\mathrm{loc}}\)-Lipschitz gradient.  Its unique minimizer is zero.

For a scalar correction \(h\), let \(T_h(x)\) be the endpoint of \(K\) steps
\begin{equation}
    y_0=x,\qquad
    y_{k+1}=y_k-\eta_\ell\varphi'(y_k)+\eta_\ell h.
    \label{eq:pp-local-map}
\end{equation}
Define the zero-correction server map
\begin{equation}
    G(x)=(1-\gamma)x+\gamma T_0(x).
    \label{eq:pp-G}
\end{equation}

\begin{lemma}[Concavity and propagation of the mean]
\label{lem:pp-concavity}
If \(0<\eta_\ell\le1/L_{\mathrm{loc}}\), then \(T_h(x)\) is jointly concave in \((x,h)\),
and \(G\) is concave and nondecreasing.  Under
\eqref{eq:pp-local-update}--\eqref{eq:pp-global-control-update},
\begin{equation}
    \E x^r\le G(\E x^{r-1}).
    \label{eq:pp-mean-recursion}
\end{equation}
In particular, with
\[
    n=(1-\eta_\ell\mu)^K,\qquad
    v=1-\gamma+\gamma n,
\]
one has \(G(x)=vx\) for every \(x<0\).
\end{lemma}

\begin{proof}
The one-step map without correction is
\[
    Q(y)=y-\eta_\ell\varphi'(y)=
    \begin{cases}
        (1-\eta_\ell\mu)y,&y<0,\\
        (1-\eta_\ell L_{\mathrm{loc}})y,&y\ge0.
    \end{cases}
\]
Its left slope is at least its right slope, and both slopes are nonnegative.
Thus \(Q\) is concave and nondecreasing.
Starting from the affine function \(y_0(x,h)=x\), induction in
\(y_{k+1}(x,h)=Q(y_k(x,h))+\eta_\ell h\) proves joint concavity:
composition with the nondecreasing concave map \(Q\) preserves concavity,
and addition of the linear term preserves it as well.
The same induction proves monotonicity in \(x\).
Consequently \(G\) is concave and nondecreasing.

Condition on \(\mathcal F_{r-1}\) and put
\(h_i=c_i^{r-1}-c^{r-1}\).  The invariant in \eqref{eq:pp-global-control-update} gives
\(N^{-1}\sum_i h_i=0\).  Uniform sampling and concavity in the correction imply
\begin{align*}
    \E[x^r\mid\mathcal F_{r-1}]
    &=(1-\gamma)x^{r-1}
      +\frac{\gamma}{N}\sum_{i=1}^N T_{h_i}(x^{r-1})\\
    &\le (1-\gamma)x^{r-1}+\gamma T_0(x^{r-1})
     =G(x^{r-1}).
\end{align*}
Taking expectation and applying Jensen's inequality to the concave function
\(G\) proves \eqref{eq:pp-mean-recursion}.
For a negative starting point and zero correction, all local iterates stay
nonpositive and are multiplied by \(1-\eta_\ell\mu\) at each step.
This proves \(G(x)=vx\) on the negative half-line.
\end{proof}

\begin{proof}[Proof of Theorem~\ref{thm:pp-lower}]
Let
\[
    M=\alpha L_{\mathrm{loc}}=K\eta_\ell L_{\mathrm{loc}}.
\]
We distinguish two ranges of \(M\).  All functions in either
construction are the same across clients; hence their gradient differences
relative to the global objective are identically zero, giving \(\delta=0\).

\paragraph{Case \(M\le64/\beta\).}

Take identical objectives on \(\R^2\),
\begin{equation}
    f_i(x,z)=\frac{\lambda}{2}x^2+\frac{L_{\mathrm{loc}}}{2}z^2,
    \qquad
    \lambda=\frac{\beta L_{\mathrm{loc}}}{128R},
    \qquad
    (x^0,z^0)=(D_0,0).
    \label{eq:pp-slow-quadratic}
\end{equation}
The second coordinate stays zero and makes \(L_{\mathrm{loc}}\) the exact smoothness
constant.  We analyze only the first coordinate.
For this quadratic, unrolling the local updates with a constant
correction \(h\) gives the endpoint
\[
 (1-\eta_\ell\lambda)^Kx
 +\eta_\ell h\sum_{j=0}^{K-1}(1-\eta_\ell\lambda)^j
 =(1-\eta_\ell\lambda)^Kx
   +\frac{1-(1-\eta_\ell\lambda)^K}{\lambda}h.
\]
The coefficient of \(h\) is common to all clients. Since
\(h_i=c_i^{r-1}-c^{r-1}\) satisfies \(N^{-1}\sum_i h_i=0\),
the correction terms cancel in the population average, even when the
stored controls differ. Uniform sampling and the server update then give
\[
    \E[x^r\mid\mathcal F_{r-1}]=q_\lambda x^{r-1},
    \qquad
    q_\lambda
    =1-\gamma\bigl[1-(1-\eta_\ell\lambda)^K\bigr].
\]
Bernoulli's inequality yields
\[
    q_\lambda\ge1-\gamma\alpha\lambda
    \ge1-\frac1{2R},
\]
where we used \(M\le64/\beta\) and \(\gamma\le1\).
Therefore, for \(1\le r\le R\),
\[
    \E x^r=q_\lambda^rD_0
    \ge\left(1-\frac r{2R}\right)D_0
    \ge\frac{D_0}{2}.
\]
Convexity then gives
\begin{equation}
    \E[f(\bar x^R)]-f^\star
    \ge\frac{\lambda}{2}(\E\bar x^R)^2
    \ge\frac{\beta L_{\mathrm{loc}}D_0^2}{1024R}.
    \label{eq:pp-small-time-lower}
\end{equation}

\paragraph{Case \(M>64/\beta\).}

Use \(\varphi\) from \eqref{eq:pp-asymmetric}, with
\begin{equation}
    \mu=\frac1{\alpha R}=\frac{L_{\mathrm{loc}}}{MR},
    \qquad x^0=D_0.
    \label{eq:pp-mu}
\end{equation}
One may append the inactive term \(L_{\mathrm{loc}}z^2/2\) to keep both constructions in
\(\R^2\).  Since \(\beta\le1/4\), this choice satisfies \(0<\mu<L_{\mathrm{loc}}\).
Set
\begin{equation}
    q_L=(1-\eta_\ell L_{\mathrm{loc}})^K,\quad
    n=(1-\eta_\ell\mu)^K,\quad
    q=1-\gamma+\gamma q_L,\quad
    v=1-\gamma+\gamma n.
    \label{eq:pp-mnqv}
\end{equation}
We will use
\begin{equation}
    1-q_L\ge\frac12,\qquad
    \frac{1-n}{\mu}\ge\frac\alpha2,\qquad
    v\ge1-\frac1R.
    \label{eq:pp-basic-scalar-bounds}
\end{equation}
Indeed, \(q_L\le e^{-M}\le1/2\).
Also \(\eta_\ell\mu=1/(KR)\).  The elementary bounds
\[
    (1-u)^K\ge1-Ku,\qquad
    (1-u)^K\le\frac1{1+Ku}\qquad(0\le u\le1)
\]
give
\[
    1-\frac1R\le n\le\frac{R}{R+1}.
\]
Hence \(1-n\ge1/(2R)\), which proves the middle inequality in
\eqref{eq:pp-basic-scalar-bounds}, and the last follows from \(\gamma\le1\).

\paragraph{The first two rounds.}
Initially every control equals \(L_{\mathrm{loc}}D_0\).
The first round refreshes selected controls to exactly the same value.
Thus all local corrections are zero in the first two rounds, and
\begin{equation}
    x^1=qD_0,\qquad x^2=q^2D_0.
    \label{eq:pp-first-two}
\end{equation}
After round two, the fraction \(p\) of clients selected in that round has
control \(L_{\mathrm{loc}}qD_0\); the remaining fraction \(1-p\) has control \(L_{\mathrm{loc}}D_0\).
Since \(1-q=\gamma(1-q_L)\), the corrections at the start of round three are
\begin{align}
    h_-&=-\gamma(1-q_L)(1-p)L_{\mathrm{loc}}D_0
       &&\text{on a fraction }p\text{ of clients},\nonumber\\
    h_+&=\gamma(1-q_L)pL_{\mathrm{loc}}D_0
       &&\text{on a fraction }1-p\text{ of clients}.
       \label{eq:pp-split-corrections}
\end{align}
The identity of these clients is random, but the two values and their
population fractions are deterministic.

\paragraph{The third round.}
For a positive correction \(h_+\), the local path stays nonnegative, so
\begin{equation}
    T_{h_+}(x^2)=q_L q^2D_0+(1-q_L)\frac{h_+}{L_{\mathrm{loc}}}.
    \label{eq:pp-positive-endpoint}
\end{equation}
For the negative correction we use the pointwise inequality
\[
    y-\eta_\ell\varphi'(y)\le(1-\eta_\ell\mu)y
    \qquad\text{for every }y\in\R.
 \]
The coefficient \(1-\eta_\ell\mu\) is nonnegative.
Induction in the local recursion therefore gives
\begin{equation}
    T_{h_-}(x^2)\le n q^2D_0+(1-n)\frac{h_-}{\mu}.
    \label{eq:pp-negative-endpoint}
\end{equation}
Taking the population average of these two endpoints is exactly the
conditional expectation of the sampled average. Since \(x^2\), the two
correction values, and their fractions are deterministic,
\[
 \E x^3=(1-\gamma)x^2
       +\gamma\bigl[pT_{h_-}(x^2)+(1-p)T_{h_+}(x^2)\bigr].
\]
Although the corrections have zero population mean, their effects need
not cancel: positive corrections keep the local path in the region of
curvature \(L_{\mathrm{loc}}\), while negative corrections can move it
into the region of smaller curvature \(\mu\). Substituting the endpoint
bounds and \eqref{eq:pp-split-corrections} gives
\begin{align}
    \E x^3
    \le{}&
    [1-\gamma+\gamma(pn+(1-p)q_L)]q^2D_0
    +\beta(1-q_L)^2D_0
    \nonumber\\
    &-\beta(1-q_L)(1-n)\frac{L_{\mathrm{loc}}}{\mu} D_0
    \nonumber\\
    \le{}&\left(1+\beta-\frac{\beta M}{4}\right)D_0
    \le-\frac{\beta M}{8}D_0.
    \label{eq:pp-negative-mean}
\end{align}
For the second inequality, the coefficient of \(q^2D_0\) is at most
one, \(q^2\le1\), and \((1-q_L)^2\le1\). The magnitude of the
negative term is bounded below using \eqref{eq:pp-basic-scalar-bounds}:
\[
 (1-q_L)\frac{1-n}{\mu}L_{\mathrm{loc}}
 \ge\frac12\cdot\frac\alpha2 L_{\mathrm{loc}}=\frac M4.
\]
For the third, note that \(\beta\le1/4\) and \(\beta M>64\).

\paragraph{Persistence of the mean.}
Since \(\E x^3<0\), Lemma~\ref{lem:pp-concavity} and
\(G(x)=vx\) for \(x<0\) first give
\[
 \E x^4\le G(\E x^3)=v\E x^3<0.
\]
Repeating this step, using that \(G\) is nondecreasing, yields
\begin{equation}
    \E x^r\le v^{r-3}\E x^3
    \qquad (3\le r\le R).
    \label{eq:pp-persistence}
\end{equation}
This controls the expected iterate; individual sampled trajectories need
not remain negative.
For \(3\le r\le R\) and \(R\ge4\),
\[
    v^{r-3}\ge
    \left(1-\frac1R\right)^{R-1}
    \ge e^{-1}>\frac14.
\]
The middle inequality follows from
\(\log(1-u)\ge-u/(1-u)\), applied with \(u=1/R\).
Because \(\E x^3<0\), \eqref{eq:pp-persistence} implies
\begin{equation}
    \E x^r\le-\frac{\beta M}{32}D_0
    \qquad (3\le r\le R).
    \label{eq:pp-mean-through-horizon}
\end{equation}
Combining this with \(x^1,x^2\le D_0\),
\begin{equation}
    \E\bar x^R
    \le\frac{2D_0}{R}
      -\frac{R-2}{R}\frac{\beta M}{32}D_0
    \le-\frac{\beta M}{128}D_0.
    \label{eq:pp-mean-output}
\end{equation}
For the last inequality it is enough that
\(\beta M(3R-8)\ge256\), which follows from \(R\ge4\) and \(\beta M>64\).

\paragraph{Objective gap.}
Jensen's inequality and \eqref{eq:pp-mu} give
\begin{align}
    \E[\varphi(\bar x^R)]
    &\ge \varphi(\E\bar x^R)
     =\frac{\mu}{2}(\E\bar x^R)^2\nonumber\\
    &\ge \frac{\beta^2 M}{32768}\frac{L_{\mathrm{loc}}D_0^2}{R}
     >\frac{\beta}{512}\frac{L_{\mathrm{loc}}D_0^2}{R}.
    \label{eq:pp-large-time-lower}
\end{align}
This is stronger than \eqref{eq:pp-lower}.  Together with
\eqref{eq:pp-small-time-lower}, it completes the proof.
\end{proof}

\section{Experimental details}
\label{app:experiments}

\subsection{Exact constants and initialization}
For \(x=(w,z)\in\R^m\times\R^2\), define
\begin{equation}
    f_i(w,z)=\frac12\sum_{j=1}^m\lambda_j w_j^2
        +\delta(a_i^\top z)^2,
    \label{eq:experiment-objective}
\end{equation}
where the \(\lambda_j\) are logarithmically spaced between
\(\lambda_{\min}\) and \(L\), and the \(a_i\) are evenly spaced unit
directions. The shared coordinates determine local smoothness, while the
two-dimensional client-dependent component determines Hessian similarity.
For \(0<\delta\le L/2\), the corresponding constants are exactly
\(L_{\mathrm{loc}}=L\) and \(\delta\), with \(f^\star=f(0)=0\).

We run Algorithm~\ref{alg:scaffold-option-i} with exact gradients,
\(N=10\), \(m=160\), \(L=1\), \(\lambda_{\min}=10^{-8}\), and
\(D_0=1\). We vary \(\delta\in\{0.01,0.03,0.1\}\) and
\(K\in\{1,2,4,\ldots,256\}\), measuring the objective at the averaged
server iterate, as in \eqref{eq:averaged-output-gap}.

In \eqref{eq:experiment-objective}, we use
\[
    \lambda_j=\lambda_{\min}(L/\lambda_{\min})^{(j-1)/(m-1)},
    \qquad
    a_i=(\cos\varphi_i,\sin\varphi_i)^\top,
    \qquad \varphi_i=\pi(i-1)/N.
\]
For \eqref{eq:experiment-objective}, the local and global Hessians are
\[
    A_i=\operatorname{diag}(\Lambda,2\delta P_i),
    \qquad
    \bar A=\operatorname{diag}(\Lambda,\delta I_2),
\]
where \(\Lambda=\operatorname{diag}(\lambda_1,\ldots,\lambda_m)\) and
\(P_i=a_i a_i^\top\). The tight-frame identity
\(N^{-1}\sum_i P_i=I_2/2\) is the same as in
Lemma~\ref{lem:stale-quadratic-recurrence}. Since
\(\lambda_m=L\) and \(2\delta\le L\), every client has exact smoothness
\(L\). Moreover, \(2P_i-I_2\) is an orthogonal reflection, so for
\(v=(v_w,v_z)\),
\[
    \frac1N\sum_i\norm{(A_i-\bar A)v}^2
    =\delta^2\norm{v_z}^2.
\]
Thus the smallest constant in Assumption~\ref{ass:similarity} is exactly
\(\delta\). The global objective is
\[
    f(w,z)=\frac12\sum_j\lambda_j w_j^2
        +\frac\delta2\norm{z}^2,
\]
with unique minimizer zero. We initialize
\[
    w_j^0=\sqrt{0.9/m},\qquad z^0=(\sqrt{0.1},0)^\top,
\]
so \(D_0=1\), with both components nonzero. All gradients and controls
are exact.

\subsection{Slow progress and the measured output}
The common coordinates evolve independently as
\[
    w_j^r=(1-\eta_\ell\lambda_j)^{Kr}w_j^0.
\]
Consequently, their contribution at the averaged server output is exactly
\[
\begin{split}
    f_c(\bar w^R)&=\frac12\sum_j\lambda_j(w_j^0)^2 m_j(R)^2,\\
    m_j(R)&=\frac1R\sum_{r=1}^R q_j^r,
    \qquad q_j=(1-\eta_\ell\lambda_j)^K.
\end{split}
\]
For a dense logarithmic spectrum, write
\(\rho=L/\lambda_{\min}\) and \(D_c^2=0.9\). Approximating the sum
by an integral and the relevant small-curvature multipliers by exponentials
gives
\[
    f_c(\bar w^R)\approx
    \frac{D_c^2}{2\log\rho}
    \int_{\lambda_{\min}}^L
    \left(\frac{1-e^{-\alpha R\lambda}}{\alpha R\lambda}\right)^2
    d\lambda
    \approx\frac{D_c^2\log2}{\log\rho}\frac1{\alpha R},
\]
when \(\lambda_{\min}\alpha R\ll1\ll L\alpha R\).

\subsection{Exact stability boundary}
The common block contracts when \(0<\eta_\ell L<2\), or
\(0<\alpha<2K/L\). For the heterogeneous block, set
\[
    v=\alpha\delta,\qquad
    \tau_K(v)=(1-2v/K)^K.
\]
Lemma~\ref{lem:stale-quadratic-recurrence} gives
\[
    z^r=B_K(v)z^{r-1}-C_K(v)z^{r-2},
    \qquad z^{-1}=z^0,
\]
where \(B_K(v)=(1+\tau_K(v))/2\) and
\(C_K(v)=(\tau_K(v)+2v-1)/4\). Both roots of
\(q^2-B_K(v)q+C_K(v)\) lie strictly inside the unit disk precisely when
\[
\begin{split}
    1-B_K+C_K&=\frac{1-\tau_K+2v}{4}>0,\\
    1+B_K+C_K&=\frac{5+3\tau_K+2v}{4}>0,\\
    1-C_K&=\frac{5-\tau_K-2v}{4}>0.
\end{split}
\]
Within \(\alpha<2K/L\) and \(\delta\le L/2\), we have
\(\lvert1-2v/K\rvert<1\), so the first two inequalities hold.
Therefore the full strict stability region is
\begin{equation}
    0<\alpha<\frac{2K}{L},\qquad
    2\alpha\delta+(1-2\alpha\delta/K)^K<5.
    \label{eq:experiment-stability}
\end{equation}
We define \(\alpha_{\max}\) as the supremum of this region. The left side
of the second inequality increases with \(\alpha\) in the local stability
interval. If it reaches \(5\) before \(2K/L\), we locate that root by
bisection; otherwise \(\alpha_{\max}=2K/L\). At fixed
\(v=\alpha\delta\) and large \(K\), the heterogeneous boundary tends
to the solution of \(2v+e^{-2v}=5\), approximately \(2.497\).

\subsection{Stepsize selection and reproducibility}
\begin{figure}[t]
    \centering
    \includegraphics[width=0.55\textwidth]{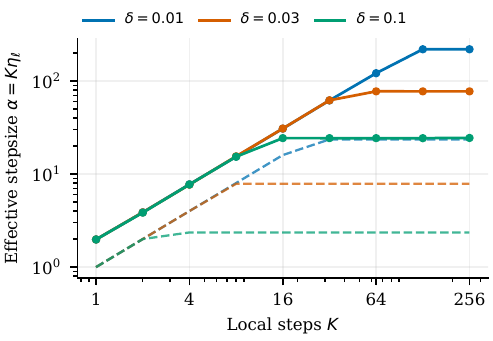}
    \caption{Effective stepsizes \(\alpha=K\eta_\ell\) for the quadratic
    experiment in Figure~\ref{fig:convex-communication}. Dashed curves
    use the theoretical rule; solid curves show the tuned choices.}
    \label{fig:convex-stepsizes}
\end{figure}
Figure~\ref{fig:convex-stepsizes} shows how tuning permits larger effective
stepsizes while retaining the similarity-dependent limit on local work.

The theoretical deterministic rule is
\[
    \alpha_{\mathrm{ref}}=\frac1{\max\{L/K,3\sqrt2\delta\}},
    \qquad \eta_{\ell,\mathrm{ref}}=\alpha_{\mathrm{ref}}/K.
\]
This is exactly the choice in the deterministic case following
Theorem~\ref{thm:scaffold-convex}, yielding \eqref{eq:deterministic-rate}.

For each \((\delta,K)\), we evaluate \(40\) logarithmically spaced
values of \(\alpha\) from \(0.1\alpha_{\mathrm{ref}}\) to
\(0.99\alpha_{\max}\), and include \(\alpha_{\mathrm{ref}}\)
explicitly. We then evaluate \(20\) additional interior logarithmic
points between the neighboring grid values of the candidate with the
earliest first target crossing, using a one-sided interval at an endpoint.
This gives \(61\) candidates per setting. Each candidate uses a constant
local stepsize \(\eta_\ell=\alpha/K\) and is evaluated for
\(H=20{,}000\) rounds.

For the target \(\varepsilon=10^{-5}\), we score each candidate by
\[
    R_\varepsilon^{(H)}(\alpha)
    =\min\bigl\{r\in\{1,\ldots,H\}:
       f(\bar x^s)-f^\star\le\varepsilon
       \text{ for all }s=r,\ldots,H\bigr\}.
\]
This criterion accounts for trajectories that cross the target and later
exceed it. The plotted round counts exclude control initialization.
We select the candidate with the smallest score, breaking ties
toward smaller \(\alpha\). Candidates that do not reach the target within
the horizon are excluded from selection.

We evaluate the exact common-coordinate formula and delayed recurrence,
then compute the objective at the average of the server iterates
\(x^1,\ldots,x^R\). We verify the selected averaged trajectories against
explicit client updates for up to \(200\) rounds per setting.

\subsection{Neural-network experiments}
\label{app:nn-experiments}

\subsubsection{Experimental setup}

\paragraph{Data and model.}
We use CIFAR-10 with \(45{,}000\) training images, \(5{,}000\) validation
images (\(500\) per class), and the official \(10{,}000\)-image test set.
The training set is split among \(N=10\) clients. The IID partition is a
random equal-size split, so each client receives a sample from the same
pooled empirical distribution. For the heterogeneous partition, we draw
\(\pi^{(c)}\sim\operatorname{Dirichlet}(0.1,\ldots,0.1)\) independently
for each class \(c\), where \(\pi_i^{(c)}\) is the fraction of that
class assigned to client \(i\). Small concentration favors allocations
in which a few clients receive most examples of a class. We randomly
shuffle each class, round its allocation counts down, and assign the
remaining examples to a randomly chosen client. This yields client sizes
\(8786,10467,3841,3762,809,8996,614,2454,1384,3887\).
The split and partition seeds are both \(2024\); each partition is held
fixed across methods, stepsizes, and training seeds.
The two partitions preserve the pooled training set and the global
objective while changing the local datasets. They test whether similar
client data make local computation more effective. Data similarity
motivates, but does not itself establish, the gradient-difference bounds
in our assumptions; these runs do not estimate \(\delta\).

The CIFAR-style ResNet-18 has a \(3\times3\), stride-one input convolution,
no initial max pooling, and four residual stages with two blocks each.
We use GroupNorm with \(32\) groups in place of BatchNorm, random
initialization, and no dropout or pretrained weights. GroupNorm computes
normalization statistics separately for each example, so no running
batch statistics need to be maintained or aggregated. BN-SCAFFOLD
\citep{quintana2024bnscaffold} identifies drift in BatchNorm statistics
as a separate difficulty for SCAFFOLD; our choice avoids introducing
that issue into the comparison of local updates and aggregation.
The architecture, preprocessing, and partitioning code are adapted from
the S-DANE implementation\footnote{\url{https://github.com/mlolab/S-DANE}}.
Training uses random crops with padding four and random horizontal flips.
Images are normalized channelwise with mean
\((0.4914,0.4822,0.4465)\) and standard deviation
\((0.2023,0.1994,0.2010)\).

\paragraph{Objectives and updates.}
Let \(n_i\) be the size of client \(i\)'s dataset and
\(n=\sum_i n_i\). Each client gradient is multiplied by \(Nn_i/n\),
and endpoints and controls are averaged uniformly. Equivalently,
\(f_i(x)=(Nn_i/n)\E[\ell(x;\xi_i)]\), where \(\xi_i\) samples a local
image and its augmentation. Then \(N^{-1}\sum_i f_i\) is the
example-weighted global cross-entropy objective.

All methods use \(400\) rounds, minibatches of \(b=64\) examples sampled
with replacement, and a constant stepsize. We use no momentum, weight
decay, clipping, or learning-rate schedule. Both SCAFFOLD variants
initialize controls at the initial model using \(K\) minibatches per
client. In each round, local steps use the stored controls, while new
controls estimated at the round's starting model become active in the
next round. Averaging uses all client endpoints. Selection chooses one
client uniformly before the local steps and uses its endpoint; all
clients still estimate controls. The server stepsize is one.
Minibatch SGD averages \(2NK\) minibatch gradients at the same model
before taking one update. It samples from the pooled training set, so
its paired IID and Dirichlet runs give identical results and serve as
the same baseline in both partitions.

\paragraph{Stepsize selection and evaluation.}
For each method, partition, and \(K\), we select the stepsize with the
smallest validation loss averaged over rounds
\(380,385,390,395,400\) and two tuning seeds. Exact ties are resolved in
favor of the smaller stepsize. The initial grid is
\(\{0.0003,0.001,0.003,0.01,0.03\}\). For IID averaging, the tested
grid also includes \(\{0.06,0.1,0.2,0.3\}\). For both SCAFFOLD variants
on the Dirichlet partition, it also includes
\(\{0.00003,0.0001\}\) at \(K=10\), and
\(\{0.000003,0.00001,0.00003,0.0001\}\) at \(K=40\).
The grids were expanded using validation results; search budgets
therefore differ across settings. All \(80\) candidate configurations
were evaluated with both tuning seeds, and no selected stepsize lies
at the boundary of its tested grid.

The tuning initialization/sampling seed pairs are
\((901,1001)\) and \((902,1002)\). After selecting stepsizes, we evaluate
them with seed pairs \((1901,2001)\), \((1902,2002)\), and
\((1903,2003)\), paired across methods. Test metrics do not enter
stepsize selection. Table~\ref{tab:nn-results} reports the round-400
server model, and all curves evaluate server models every five rounds.
Training loss is clean-image cross-entropy on a fixed stratified
\(5{,}000\)-image training subset; validation and test metrics use their
full splits, without random augmentation. Shaded bands are sample
standard deviations across the three runs. Since the
partitions are fixed, these bands describe training randomness rather
than variation across client partitions. The evaluated server model
differs from the random local output in Theorem~\ref{thm:scaffold-nonconvex}.

\paragraph{Implementation.}
Training uses PyTorch~2.5.1, torchvision~0.20.1, CUDA~12.4,
and NVIDIA H200 GPUs, with float32 computation and deterministic
operations enabled.

\subsubsection{Additional training results}
Table~\ref{tab:nn-results} summarizes the selected stepsizes and final
accuracies for both local-step settings. Figure~\ref{fig:nn-k10} gives
the $K=10$ learning curves, complementing the $K=40$ results in
Figure~\ref{fig:nn-comparison}. On IID data, increasing $K$ improves
both SCAFFOLD variants at the same round budget. This benefit is absent
on the heterogeneous partition. The modest final difference between
the two SCAFFOLD variants on Dirichlet data with $K=40$ does not
establish an advantage for selection.

\begin{table}[H]
\centering
\caption{Selected constant stepsizes and test accuracy at round \(400\).
Accuracies are percentages, reported as mean \(\pm\) sample standard
deviation over three runs.}
\label{tab:nn-results}
\begin{tabular}{llrrr}
\toprule
Partition & Method & \(K\) & Stepsize & Test accuracy (\%)\\
\midrule
IID & SCAFFOLD averaging & 10 & \(0.1\) & \(71.68\pm1.65\)\\
& SCAFFOLD selection & 10 & \(0.01\) & \(53.14\pm3.18\)\\
& Minibatch SGD & 10 & \(0.01\) & \(37.25\pm3.74\)\\
& SCAFFOLD averaging & 40 & \(0.2\) & \(88.54\pm0.62\)\\
& SCAFFOLD selection & 40 & \(0.01\) & \(78.73\pm0.51\)\\
& Minibatch SGD & 40 & \(0.01\) & \(35.88\pm2.25\)\\
\midrule
Dirichlet \(0.1\) & SCAFFOLD averaging & 10 & \(0.0003\) & \(35.99\pm1.63\)\\
& SCAFFOLD selection & 10 & \(0.0003\) & \(30.16\pm4.36\)\\
& Minibatch SGD & 10 & \(0.01\) & \(37.25\pm3.74\)\\
& SCAFFOLD averaging & 40 & \(0.00003\) & \(30.34\pm1.18\)\\
& SCAFFOLD selection & 40 & \(0.00003\) & \(30.94\pm0.89\)\\
& Minibatch SGD & 40 & \(0.01\) & \(35.88\pm2.25\)\\
\bottomrule
\end{tabular}
\end{table}

\begin{figure}[H]
    \centering
    \includegraphics[width=\textwidth]{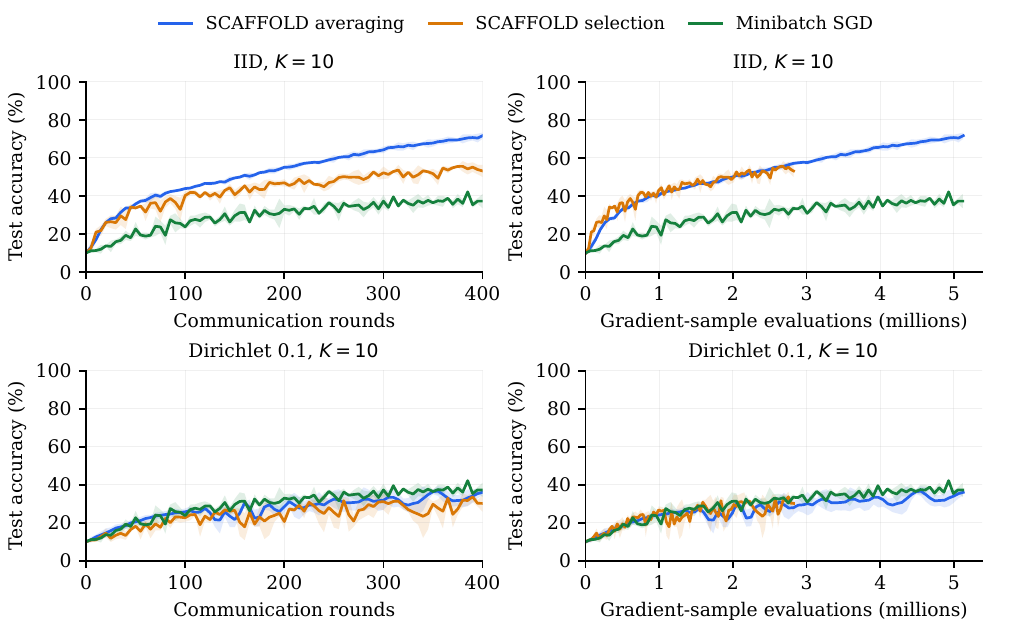}
    \caption{Additional CIFAR-10 results for $K=10$: server test accuracy
    versus communication rounds (left) and gradient-sample evaluations,
    including controls (right). Curves show means and bands show sample
    standard deviations over three runs. Figure~\ref{fig:nn-comparison}
    and Figure~\ref{fig:nn-computation} give the corresponding $K=40$ results.}
    \label{fig:nn-k10}
\end{figure}

\paragraph{Training loss.}
Figure~\ref{fig:nn-training-loss} shows that the IID advantage also
appears in training loss: averaging reaches lower loss than selection
and minibatch SGD for both values of $K$. On the Dirichlet partition,
increasing $K$ does not produce the same improvement, and the SCAFFOLD
curves exhibit occasional loss spikes. These results concern one
architecture, one dataset, two fixed partitions, and a $400$-round budget.

\begin{figure}[H]
    \centering
    \includegraphics[width=\textwidth]{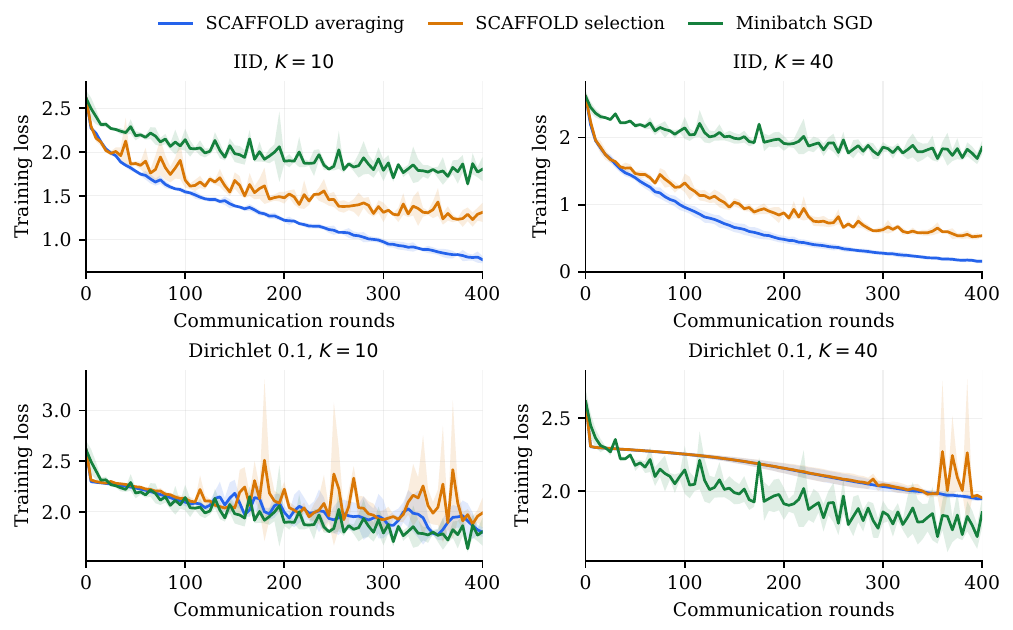}
    \caption{Server training loss for $K=10$ (left) and $K=40$ (right),
    evaluated without augmentation on the fixed $5{,}000$-image training
    subset. Curves show means and bands show sample standard deviations
    over three runs.}
    \label{fig:nn-training-loss}
\end{figure}

\clearpage
\subsubsection{Computation and communication costs}
We count each example used in a gradient calculation, including control
estimation and repeated draws. An ordinary round uses \(2NKb\) such
evaluations for averaging, \((N+1)Kb\) for selection, and \(2NKb\) for
minibatch SGD. Both SCAFFOLD variants also use \(NKb\) evaluations at
initialization. Thus minibatch SGD matches averaging's ordinary-round
sample budget; selection uses less computation. Evaluation passes are
excluded from these counts.

Figure~\ref{fig:nn-computation} compares accuracy at the same
computation budget for $K=40$, complementing the communication-round
comparison in Figure~\ref{fig:nn-comparison}.

\begin{figure}[H]
    \centering
    \includegraphics[width=\textwidth]{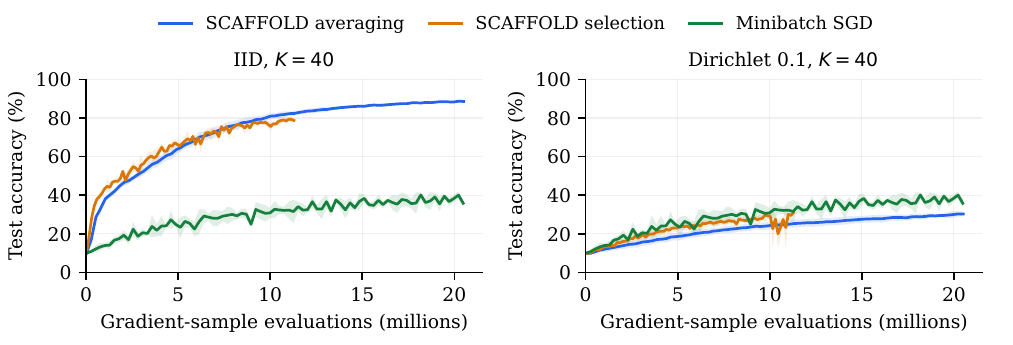}
    \caption{CIFAR-10 with $K=40$: server test accuracy versus
    gradient-sample evaluations, including controls, for IID (left)
    and Dirichlet (right) partitions. Curves show means and bands
    show sample standard deviations over three runs.}
    \label{fig:nn-computation}
\end{figure}

For communication volume, let \(P\) be the bytes in one model-sized
vector. We count uncompressed uploads and downloads: averaging uses
\(4NPR+2NP\) bytes over \(R\) rounds, selection uses
\(2(N+1)PR+2NP\), and minibatch SGD uses \(2NPR\).
These counts refer to the server-output implementation used in the
experiments, which does not return the additional sampled local iterate
in Algorithm~\ref{alg:scaffold-nonconvex}. The SCAFFOLD counts include
initial control estimation, model and control exchanges, and the returned
endpoints. Here
\(P=44{,}695{,}848\) bytes for \(11{,}173{,}962\) float32 parameters.
This measures modeled tensor traffic. Clients are simulated sequentially,
so runtime does not measure distributed communication latency.

\subsection{Trajectory diagnostics}
\label{app:nn-diagnostics}

We measure differences between clients' gradient changes along the
\(K=40\) training trajectories from Appendix~\ref{app:nn-experiments}.
These measurements use the same selected stepsizes and three random
seeds, with server models at rounds \(0,20,\ldots,400\).
They do not enter stepsize selection.

\paragraph{Measurements.}
For each partition, we sample \(256\) training examples per client
without replacement and take a nested subset of \(128\) examples.
The subsets and each image's crop and flip are fixed across models,
methods, and seeds. Let \(\widehat\nabla f_i(x)\) denote the gradient
averaged over one such subset, including the objective weight
\(Nn_i/n\) used in training. For a displacement \(u\), define
\[
 a_i(x,u)=\frac{\widehat\nabla f_i(x+u)-\widehat\nabla f_i(x)}{\|u\|},
 \qquad
 \bar a(x,u)=\frac1N\sum_{i=1}^N a_i(x,u).
\]
We compute four quantities:
\begin{align*}
 \widehat\delta_{\mathrm{avg}}(x,u)
 &=\left(\frac1N\sum_{i=1}^N\|a_i(x,u)-\bar a(x,u)\|^2\right)^{1/2},
 &\widehat\delta_{\max}(x,u)
 &=\max_i\|a_i(x,u)-\bar a(x,u)\|,\\
 \widehat L_{\mathrm{loc}}(x,u)&=\max_i\|a_i(x,u)\|,
 &\widehat L_g(x,u)&=\|\bar a(x,u)\|.
\end{align*}
The first two measure averaged and uniform disagreement between local
gradient changes; the last two measure local and global gradient
sensitivity on the sampled data.

At each model, we use three normalized Gaussian directions and
displacements of length \(10^{-3}\|x\|\) and \(10^{-2}\|x\|\).
The directions are shared across all measurements, and the denominator
uses the displacement represented in floating-point arithmetic.
For each seed, we take the maximum over the three directions separately
for each quantity, then report means and sample standard deviations
over seeds. The maximizing direction can differ between quantities.
The same examples and augmentations are used at \(x\) and \(x+u\),
so their gradient difference contains no fresh sampling between the
two evaluations. These are finite gradient-difference measurements
on sampled objectives, rather than estimates with guarantees for the
global constants in our assumptions.

\paragraph{Partition and trajectory comparisons.}
Figure~\ref{fig:nn-diagnostic-disagreement} shows the two disagreement
measurements using \(256\) examples per client and relative radius
\(10^{-2}\). Minibatch SGD follows
the same model trajectory in both partitions, allowing a comparison
at the same model parameters. On this trajectory, both disagreement
measurements are larger for the Dirichlet partition at every measured
round, seed, radius, and subset size. This supports the distinction
between the two partitions in terms of gradient changes. The comparison
includes both their different class distributions and the objective
weights induced by unequal client sizes.

\begin{figure}[H]
    \centering
    \includegraphics[width=0.95\textwidth]{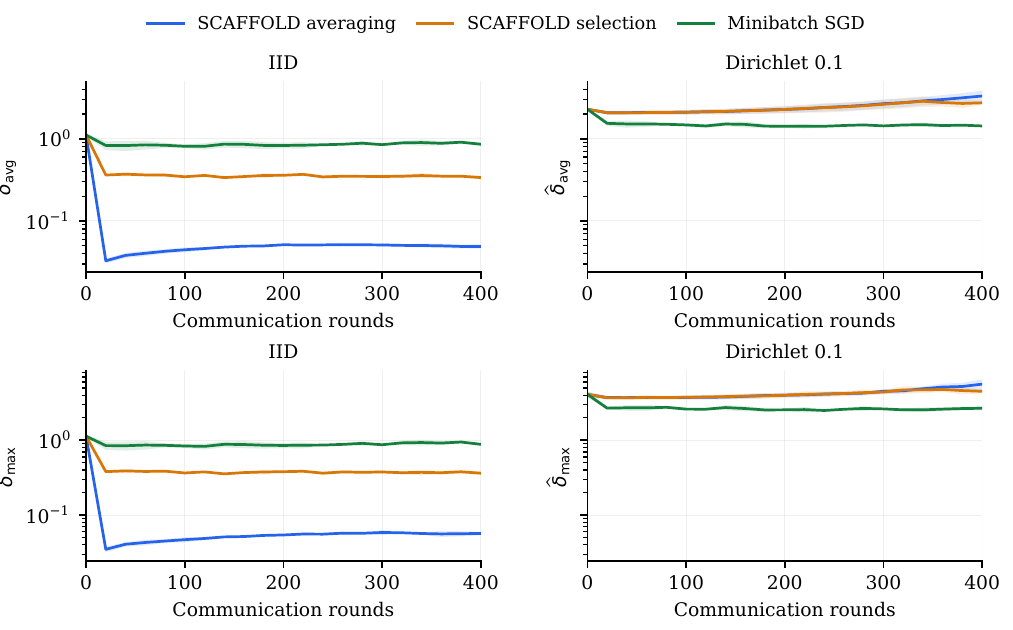}
    \caption{Averaged and uniform gradient disagreement along the
    \(K=40\) trajectories, using \(256\) examples per client and relative
    radius \(10^{-2}\). Curves show means of the three-direction maxima;
    bands show sample standard deviations over three runs.
    Vertical axes are logarithmic.}
    \label{fig:nn-diagnostic-disagreement}
\end{figure}

\paragraph{Comparison of similarity and smoothness measurements.}
Table~\ref{tab:nn-diagnostic-endpoint} compares all four quantities at
round $400$, with relative radius $10^{-2}$ and $256$ examples per
client. Figure~\ref{fig:nn-diagnostic-ratios} shows the corresponding
ratios to local sensitivity, computed separately for each seed and
then averaged.

\begin{table}[H]
\centering
\caption{Gradient-difference measurements at round \(400\), with
relative radius \(10^{-2}\) and \(256\) examples per client.
Values are means \(\pm\) sample standard deviations over three runs.}
\label{tab:nn-diagnostic-endpoint}
\small
\setlength{\tabcolsep}{5pt}
\begin{tabular}{llrrrr}
\toprule
Partition & Method & \(\widehat\delta_{\mathrm{avg}}\) &
\(\widehat\delta_{\max}\) & \(\widehat L_{\mathrm{loc}}\) & \(\widehat L_g\)\\
\midrule
IID & Averaging & \(0.049\pm0.003\) & \(0.057\pm0.003\) & \(0.061\pm0.003\) & \(0.017\pm0.001\)\\
& Selection & \(0.338\pm0.017\) & \(0.363\pm0.019\) & \(0.406\pm0.029\) & \(0.164\pm0.026\)\\
& Minibatch SGD & \(0.860\pm0.036\) & \(0.884\pm0.036\) & \(1.013\pm0.018\) & \(0.461\pm0.045\)\\
\midrule
Dirichlet \(0.1\) & Averaging & \(3.322\pm0.551\) & \(5.642\pm0.893\) & \(7.715\pm2.328\) & \(2.583\pm1.314\)\\
& Selection & \(2.745\pm0.156\) & \(4.519\pm0.274\) & \(5.567\pm0.687\) & \(1.688\pm0.610\)\\
& Minibatch SGD & \(1.434\pm0.021\) & \(2.680\pm0.168\) & \(2.947\pm0.261\) & \(0.513\pm0.053\)\\
\bottomrule
\end{tabular}
\end{table}

\begin{figure}[H]
    \centering
    \includegraphics[width=\textwidth]{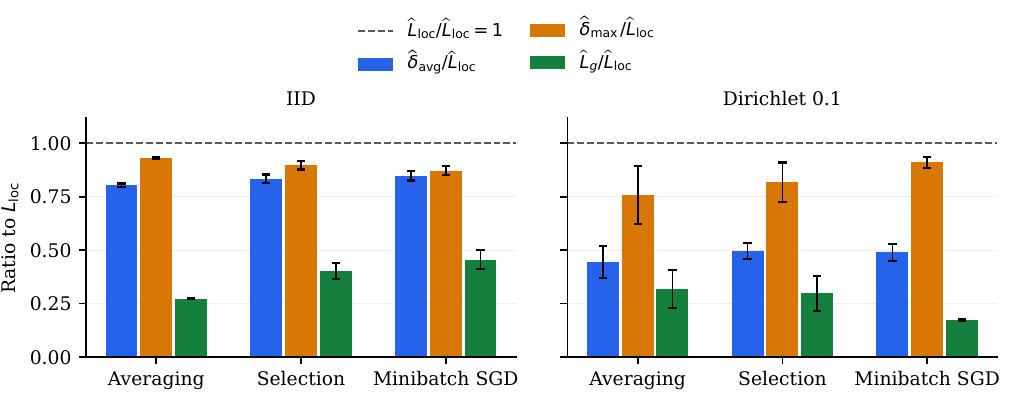}
    \caption{Similarity and global sensitivity relative to local sensitivity
    at round $400$, using relative radius $10^{-2}$ and $256$ examples per
    client. Bars show mean per-seed ratios and error bars show sample
    standard deviations over three runs. The dashed line represents
    $\widehat L_{\mathrm{loc}}/\widehat L_{\mathrm{loc}}=1$.}
    \label{fig:nn-diagnostic-ratios}
\end{figure}

By definition, $\widehat\delta_{\mathrm{avg}}\le\widehat\delta_{\max}$.
The ratio $\widehat\delta_{\max}/\widehat\delta_{\mathrm{avg}}$ ranges
from $1.03$ to $1.16$ across methods on IID data, compared with $1.65$
to $1.87$ on Dirichlet data. Thus averaged and uniform similarity give
closer measurements on the IID partition.

Global gradient sensitivity is also smaller than local sensitivity:
$\widehat L_g/\widehat L_{\mathrm{loc}}$ ranges from $0.17$ to $0.46$
across the six settings. Averaging gradient changes across clients can
therefore substantially reduce their magnitude relative to the largest
local change. This illustrates the distinction between the two smoothness
quantities retained in our nonconvex upper bound.

On IID data, SCAFFOLD with averaging reaches models with smaller absolute
disagreement, but their local gradient sensitivity is also smaller.
The ratio $\widehat\delta_{\mathrm{avg}}/\widehat L_{\mathrm{loc}}$ is
$0.80$--$0.85$ on IID data and $0.45$--$0.50$ on Dirichlet data;
$\widehat\delta_{\max}/\widehat L_{\mathrm{loc}}$ remains
$0.76$--$0.93$ across both partitions. These measurements do not exhibit
a regime of uniform disagreement much smaller than local sensitivity,
and the absolute and relative comparisons need not give the same ordering.
Since the methods also visit different models, these ratios alone do not
explain averaging's empirical advantage.

\paragraph{Sensitivity to radius and sample size.}
We repeat the measurements at both radii and nested subset sizes. With \(256\) examples per client, reducing the
relative radius from \(10^{-2}\) to \(10^{-3}\) increases
\(\widehat\delta_{\mathrm{avg}}\) by a median factor of \(2.86\)
over all measured models and seeds. Reducing the subset size from
\(256\) to \(128\) increases it by a median factor of \(1.31\)
at radius \(10^{-2}\) and \(1.38\) at radius \(10^{-3}\).
The ordering of the partitions on the minibatch-SGD trajectory
persists in all four settings.
Finite client samples, changes in ReLU activation patterns, and
numerical precision can affect these measurements. In particular,
they neither certify the similarity assumptions for the neural network
nor identify the constants governing the theoretical rates.

\end{document}